\PassOptionsToPackage{noend}{algorithmic}
\documentclass{article}

    \PassOptionsToPackage{numbers, compress}{natbib}
\usepackage[preprint]{neurips_2026}

\usepackage[utf8]{inputenc} % allow utf-8 input
\usepackage[T1]{fontenc}    % use 8-bit T1 fonts
\usepackage{hyperref}       % hyperlinks
\usepackage{url}            % simple URL typesetting
\usepackage{booktabs}       % professional-quality tables
\usepackage{amsfonts}       % blackboard math symbols
\usepackage{nicefrac}       % compact symbols for 1/2, etc.
\usepackage{microtype}      % microtypography
\usepackage[table]{xcolor}         % colors

\title{DARLING: Detection Augmented Reinforcement Learning with Non-Stationary Guarantees}

\usepackage{amsmath}
\usepackage{amssymb}
\usepackage{mathtools}
\usepackage{amsthm}
\usepackage{tikz}
\usetikzlibrary{decorations.pathreplacing}
\usepackage{algorithm}
\usepackage{algorithmic}

\usepackage{xcolor}
\usepackage{dsfont}
\usepackage{soul} % This allows use of \st for strike-through
\usepackage{multirow}
\usepackage{subcaption}

\usepackage{pifont}
\newcommand{\cmark}{\ding{51}}%
\newcommand{\xmark}{\ding{55}}%

\theoremstyle{plain}
\newtheorem{theorem}{Theorem}[section]
\newtheorem{proposition}[theorem]{Proposition}
\newtheorem{lemma}[theorem]{Lemma}

\theoremstyle{definition}
\newtheorem{definition}[theorem]{Definition}
\newtheorem{assumption}[theorem]{Assumption}
\theoremstyle{plain}
\newtheorem{property}[theorem]{Property}

\author{%
  Argyrios Gerogiannis\:\quad Yu-Han Huang\:\quad Venugopal V. Veeravalli\\
  ECE and CSL,
  The Grainger College of Engineering\\ University of Illinois at Urbana-Champaign \\
  \texttt{\{ag91,yuhanhh2,vvv\}@illinois.edu} \\
}

\begin{document}

\maketitle

\begin{abstract}

We study model-free reinforcement learning (RL) in non-stationary finite-horizon episodic Markov decision processes (MDPs) without prior knowledge of the non-stationarity. We focus on the piecewise stationary (PS) setting, where both rewards and transition dynamics can change at unknown times. We first revisit existing state-of-the-art approaches and identify theoretical and practical limitations that change the current landscape of performance guarantees. To characterize the difficulty of the problem, we establish the first minimax lower bounds for PS-RL in tabular and linear MDPs. We then introduce \emph{Detection Augmented Reinforcement Learning} (DARLING), a modular wrapper for PS-RL that applies to both tabular and linear MDPs, without knowledge of the changes. In tabular MDPs, under change-point separability and reachability conditions, DARLING improves the best known dynamic regret bounds and matches our minimax lower bound. In linear MDPs, DARLING matches the minimax lower bound when the relevant reachability parameters are known, and our analysis clarifies the structural obstacles that distinguish this setting from the tabular case. Finally, through extensive experimentation across diverse non-stationary benchmarks, we show that DARLING consistently surpasses the state-of-the-art methods.

\end{abstract}

\section{Introduction}

Reinforcement Learning (RL) studies sequential decision making in unknown environments, typically modeled as Markov decision processes (MDPs), with the goal of maximizing cumulative reward \citep{sutton2018reinforcement}. 
Most RL algorithms assume a \emph{stationary} environment, where rewards and transition dynamics are fixed but unknown. 
In many real-world applications, however, this assumption is violated: environments evolve due to changing conditions. 
Such non-stationarity is central to applications including clinical treatment planning \cite{shortreed2011treatment}, real-time bidding \cite{cai2017realtimebidding}, inventory management \cite{agrawal2019inventory}, and traffic control \cite{chen2020trafficmanage}. 
In such settings, stationary guarantees no longer apply and performance can degrade significantly \citep{ortner2020whystationarybadinns}, motivating the development of algorithms for \emph{non-stationary} (NS) RL.

Non-stationarity is often divided into two regimes: \emph{drifting} changes, where the MDP evolves gradually, and \emph{abrupt} changes, where the environment shifts at discrete times. The latter is captured by the \emph{piecewise stationary} (PS) model, where the MDP remains stationary on the segments separated by change-points. While drifting models have received substantial attention \cite{ortner2020variational_budget_restart_tabular_modelbased,cheung2020reinforce_sliding_window_tabular_modelbased,touati2021optwlsvi,zhou2022restartlsviucb,mao2022restartqucb,domingues2020kernel_sliding_window_discounted_modelbased}, the PS setting remains under-explored in RL \cite{gajane2018sliding_window_tabular_modelbased}. Recent results in the NS bandit literature indicate that algorithms designed for PS can be empirically robust even under drift and on experiments that deviate from strict piecewise stationarity, outperforming approaches explicitly tuned for drifting settings \cite{gerogiannis2025dal}. This suggests that the PS model can yield effective methods beyond the nominal regime of validity.

Existing NS-RL algorithms differ along three dimensions: (i) \emph{prior knowledge} of the non-stationarity (e.g., change frequency or variation budgets), (ii) the \emph{adaptation mechanism} used to respond to changes, and (iii) whether the method is \emph{model-based} or \emph{model-free}. Model-based approaches attempt to track the underlying dynamics as the environment evolves; while theoretically appealing, they can incur substantial computational and memory overhead and may degrade under drift due to model mis-specification and estimation error \cite{cheung2020reinforce_sliding_window_tabular_modelbased,mao2022restartqucb}. Thus, we focus on \emph{model-free} methods.

Within NS-RL, the dominant design axis is the adaptation mechanism, which yields three widely used paradigms: (i) \emph{discounted/sliding window} methods \cite{gajane2018sliding_window_tabular_modelbased,cheung2020reinforce_sliding_window_tabular_modelbased,touati2021optwlsvi,domingues2020kernel_sliding_window_discounted_modelbased}, (ii) \emph{budget-restart} methods \cite{jaksch2010bugdget_restart_tabular_modelbased,ortner2020variational_budget_restart_tabular_modelbased,mao2022restartqucb,zhou2022restartlsviucb}, and (iii) \emph{detection-restart} methods \cite{wei2021master}. Discounted and sliding-window approaches are \emph{adaptive}, continuously down-weighting or discarding older data, whereas budget-restart and detection-restart approaches are \emph{restarting} strategies that periodically or conditionally reset the learning process. These paradigms are prevalent in the NS multi-armed bandit literature, which has served as a canonical testbed for studying non-stationarity in online learning \cite{garivier2011sw,besbes2014nsmabs,besson2022efficient}. An additional discussion on the NS bandit literature is given in the Appendix. Among these paradigms, detection-restart methods are distinctive in enabling \emph{prior-free} design with optimal guarantees: they do not require knowledge of the timing, frequency, or magnitude of changes, in contrast to discounted, sliding-window, and budget-restart methods whose performance depends on tuned parameters that encode such prior information. Recent results further suggest that restarting strategies can enjoy more favorable worst-case complexity guarantees than fully adaptive schemes \citep{peng_papadim_2024}.

Despite the appeal of \emph{prior-free}, \emph{model-free} detection--restart, a theory--practice gap remains. To our knowledge, MASTER \citep{wei2021master} is the only algorithm in this class with performance guarantees, yet recent empirical work shows that its internal detection can be practically unreliable, leading to performance far worse than competing alternatives \citep{gerogiannis2024blackboxfeas,gerogiannis2025dal}. On the theory side, to our knowledge, minimax lower bounds for PS episodic MDPs have been unavailable, obscuring the difficulty of PS-RL and the optimality landscape. Importantly, after revisiting MASTER we discovered some errors in its analysis, which could possibly explain its poor performance. These gaps motivate new PS-RL methods that are simultaneously prior-free, theoretically grounded, and empirically robust.

\begin{table*}[t]
\small
\caption{Dynamic regret comparison of algorithms in PS episodic, finite-horizon tabular and linear MDPs, under Assumptions \ref{assum:cp_assum},  \ref{assum:reachability} (tabular) and \ref{assum:reachability_lin_mdps}, \ref{assum:cp_assum_lin_mdps} (linear). $S$ is the number of states, $A$ is the number of actions, $d$ is the dimension of the feature space for the linear case, $T$ is the number of episodes, $H$ is the number of steps per episode and $N_T$ is the number of changes. Prior-free means no knowledge about the number of changes $N_T$. Gray cells denote results from this work.}
\label{table:reg_bounds}
\vskip 0.11in
\centering
\begin{tabular}{@{}l|c|c|c@{}}
\toprule
    \shortstack{\textbf{Setting}} & \shortstack{\textbf{Algorithm}} & \shortstack{\textbf{Regret} }& \shortstack{\textbf{Prior-Free}} \\
\midrule
\multirow{7}{*}{\shortstack{Tabular\\MDPs}}
& RestartQ-UCB \citep{mao2022restartqucb}& $\tilde{\mathcal{O}}(S^{3/4} A^{3/4}H^{5/3} N_T^{1/3} T^{2/3})$ &  \xmark \\ 
& Double-Restart Q-UCB \citep{mao2022restartqucb} & $\tilde{\mathcal{O}}(S^{1/3} A^{1/3} H^{5/3} N_T^{1/3}T^{2/3}
+ H^{6/4} T^{3/4})$ & \cmark \\ 
& \cellcolor{gray!40} DARLING + UCMQ \cite{menard2021ucbmq} & $\tilde{\mathcal{O}}(\sqrt{S AH^{3} N_T T})$ & \cmark \\ 
& \cellcolor{gray!40} Lower Bound & $\Omega(\sqrt{SAH^3N_TT})$
 &  \\ 
\midrule
\multirow{8}{*}{\shortstack{Linear\\MDPs}}
& OPT-WLSVI \citep{touati2021optwlsvi} & $\tilde{\mathcal{O}}(d^{5/4} H^{2} N_T^{1/4} T^{3/4})$ &  \xmark \\ 
& LSVI-UCB-Restart \citep{zhou2022restartlsviucb} & $\tilde{\mathcal{O}}(d^{4/3} H^{2} N_T^{1/3} T^{2/3})$ &  \xmark \\ 
& ADA-LSVI-UCB-Restart \citep{zhou2022restartlsviucb} & $\tilde{\mathcal{O}}(d^{5/4} H^{2} N_T^{1/4} T^{3/4})$ &  \cmark \\ 
& \cellcolor{gray!40} DARLING + LSVI-UCB$++$ \citep{he2023lsviucbplusplus} & $\tilde{\mathcal{O}}(d\sqrt{ H^3N_T T})$ & \cmark \\
& \cellcolor{gray!40} Lower Bound & $\Omega(d\sqrt{H^3N_TT})$
 &  \\ 
\bottomrule
\end{tabular}
\end{table*}

\textbf{Contributions} We revisit the analysis of MASTER and identify flaws in its proof, along with other issues in the NS-RL literature. We then establish the \emph{first}, to our knowledge, minimax lower bounds for PS episodic MDPs in both tabular and linear settings. We propose DARLING, a modular and \emph{prior-free} detection-restart framework for PS episodic MDPs, which can augment \emph{any} base RL algorithm with order-optimal stationary regret, lifting them to the PS setting. We instantiate DARLING for tabular and linear MDPs and show that DARLING is the first, to our knowledge, \emph{nearly-optimal} algorithm, improving upon the best known prior-free guarantees. Finally, we evaluate extensively on PS and drifting benchmarks against state-of-the-art prior-free and prior-based baselines, where DARLING consistently outperforms all alternatives and remains robust beyond its nominal regime.

\section{Problem Formulation}
\label{sec:problem_form}
Let $[n]\coloneqq \{1, \dots, n\}$.
We study episodic RL over $T$ episodes with horizon $H$. We index time by $(t,h)$ for episode $t\in[T]$ and step $h\in[H]$. The environment is an episodic MDP with state space $\mathcal{S}$ ($|\mathcal{S}|=S$), action space $\mathcal{A}$ ($|\mathcal{A}|=A$), and step-dependent reward and transition functions $\{r_h^t, P_h^t\}$. At $(t,h)$, after taking action $a_h^t$ in state $s_h^t$, the agent observes $R_h^t(s_h^t,a_h^t)\in[0,1]$ with mean $r_h^t(s_h^t,a_h^t)$ and transitions to $s_{h+1}^t \sim P_h^t(\cdot\mid s_h^t,a_h^t)$. The episode ends at $s_{H+1}^t$.

\textbf{Value Functions and Bellman Equations} A deterministic policy $\pi: [T] \times [H] \times \mathcal{S} \rightarrow \mathcal{A}$ maps the time index and the current state to the selected action; we let $\pi_h^t(s)$ denote the chosen action at time $(t, h)$ when the current state is $s$. Under policy $\pi$, the value function $V_h^{t,\pi}: \mathcal{S} \rightarrow \mathbb{R}$ and the corresponding state-action value function $Q_h^{t,\pi}: \mathcal{S} \times \mathcal{A} \rightarrow \mathbb{R}$ at time $(t, h)$ are:
\begin{align*}
&V_{h}^{t,\pi}(s):=\mathbb{E}\left[\sum_{h^{\prime}=h}^{H} r^t_{h^{\prime}}\left(s^{t}_{h^{\prime}}, \pi_{h^{\prime}}^t\left(s^{t}_{h^{\prime}}\right)\right) \Bigg| s^{t}_{h}=s\right],\\
&Q_{h}^{t,\pi}(s, a):= r^t_{h}(s, a) +
\mathbb{E}\left[\sum_{h^{\prime}=h+1}^{H} r_{h^{\prime}}^t\left(s^{t}_{h^{\prime}}, \pi_{h^{\prime}}^t\left(s^{t}_{h^{\prime}}\right)\right) \mid s^{t}_{h}=s, a^{t}_{h}=a\right]
\end{align*}
where $s_{h^{\prime}+1}^{t} \sim P_{h^{\prime}}^{t}\left(\cdot \mid s_{h^{\prime}}^{t}, a_{h^{\prime}}^{t}\right)$. For brevity, let $P^t_h V_{h+1}^{t,\pi}(s, a) := \mathbb{E}_{s' \sim P^t_h(\cdot \mid s, a)}[V^{t,\pi}_{h+1}(s')]$. The Bellman equations give $V^{t,\pi}_h(s) = Q_h^{t,\pi}(s, \pi_h^t(s))$ and $Q_h^{t,\pi}(s, a) = (r^t_h + P_h^t V^{t,\pi}_{h+1})(s, a)$, with $V_{H+1}^{t,\pi}(s) = 0$ for all $s \in \mathcal{S}$.  There exists an optimal policy $\pi^\star$ that leads to the optimal value function $V_h^{t,\star}(s) := \sup_\pi V_h^{t,\pi}(s)$ for all $(s, t, h)$. From the Bellman optimality equation, $V^{t,\star}_h(s) = \max_{a \in \mathcal{A}} Q_h^{t,\star}(s, a)$; $Q_h^{t,\star}(s, a) := (r^t_h + P_h^t V^{t,\star}_{h+1})(s, a)$.

\textbf{Linear MDP.} We also consider a class of MDPs called \textit{linear MDPs}~\citep{jin2020lsviucb}. Linear MDPs assume both $P^t_h$ and $r^t_h$ are linear in a known feature map $\phi:\mathcal{S}\times\mathcal{A}\to\mathbb{R}^d$, such that for any $(t,h)\in[T]\times[H]$, there exist $d$ unknown measures $\mu_{h,t}=(\mu_{h,t}^1,\ldots,\mu_{h,t}^d)^\top$ on $\mathcal{S}$ and $\theta_{h,t}\in\mathbb{R}^d$ 
$
    P_h^t(s' \mid s, a) = \phi(s, a)^\top \mu_{h, t}(s'), \:\: r_h^t(s, a) = \phi(s, a)^\top \theta_{h, t}.
$
Without loss of generality, we assume $\|\phi(s, a)\|_2 \leq 1$ for all $(s, a)$, and $\max\{\|\mu_{h, t}(\mathcal{S})\|_2, \|\theta_{h, t}\|_2\} \leq \sqrt{d}$ for all $(h, t)$. For linear MDPs, the state space is possibly countably infinite ($S = \infty$), while the action space is finite.

\textbf{Dynamic Regret}  We evaluate performance using \emph{dynamic regret} \citep{cheung2020reinforce_sliding_window_tabular_modelbased, mao2022restartqucb}, which compares the agent's policy $\pi$ against the optimal policy for each episode in hindsight:
\begin{equation*}
\mathcal{R}(\pi, T) := \sum_{t=1}^{T} \left( V_{1}^{t,\star}(s_{1}^{t}) - V_{1}^{t,\pi}(s_{1}^{t}) \right),
\end{equation*}
where the initial state $s_1^t$ for each episode is selected by an oblivious adversary \citep{cheung2020reinforce_sliding_window_tabular_modelbased,mao2022restartqucb}. Thereupon, the goal of the agent is to minimize the dynamic regret with respect to the time-dependent policy $\pi$.

\textbf{Non-Stationarity Measure} In stationary MDPs, $r^{t}_{h}$ and $P^{t}_{h}$ remain the same with respect to episodes, i.e., with respect to $t$. In the PS setting, the MDP undergoes abrupt changes at $N_{T}$ unknown episodes, termed as change-points. Specifically, let
\begin{align*}
1 =: \nu_0 < \nu_1 < \dots < \nu_{N_T} < \nu_{N_T + 1} := T+1,
\end{align*}
denote the change-points. Then, 
$r_h^t$ and $P_h^t$ for each step remain the same across all $t\in \{\nu_k, \dots, \nu_{k+1}-1\}$ and at least either $r_h^t$ or $P_h^t$ at some step $h$ changes at $\nu_{k+1}$, i.e., $r_{h}^{t} = r_{h}^{t'}$ and $P_{h}^{t} = P_{h}^{t'}$ for all $h \in [H]$ and $t \in \{\nu_k, \dots, \nu_{k+1}-1\}$, and there exists a step $h \in [H]$ such that $r_{h}^{\nu_{k+1}} \neq r_{h}^{\nu_{k}}$ or $P_{h}^{\nu_{k+1}} \neq P_{h}^{\nu_{k}}$.
In the PS setting, any prior-free method aims aims to solve the problem without knowledge of $N_T$, which is a central goal of our work.

\subsection{Issues with Prior-free State-of-the-Art Approaches}
\label{sec:issues}

In tabular and linear NS-MDPs, MASTER \cite{wei2021master} is the state-of-the-art method that achieves the best regret performance theoretically without using knowledge about the non-stationarity. It converts an algorithm satisfying certain properties into a non-stationary procedure by restarting its learning process whenever its non-stationary tests raise an alarm. By applying the proof of Theorem 1 in \cite{gerogiannis2024blackboxfeas}, we find that MASTER requires at least $1.442 \times 10^{19}$ episodes in tabular MDPs and at least $7.7317 \times 10^{20}$ episodes in our simplest experimental settings in order for its detection mechanism to possibly trigger. 
%In addition, the empirical performance of MASTER is eclipsed by other methods by a huge margin as shown later in the experiment section. 
Looking deeper into MASTER's analysis, we discover an error in the regret analysis that, to our knowledge, can not be fixed, which we delineate in Appendix \ref{sec:error_master}. To briefly explain the error, the authors construct an i.i.d. sequence of Bernoulli random variables for scheduling multiple algorithm instances (Algorithm 2 in \cite{wei2021master}). When they compute the probability of these Bernoulli random variables in Lemma 17 in \cite{wei2021master}, they condition on an event that changes the distribution. However, they still treat the distribution as the same as the one without conditioning. We demonstrate that even if they remove the conditioning, the probability is different from what they computed. Since Lemma 17 is the foundation for the regret analysis, this error renders the regret upper bound invalid. 

A related line of work studies non-stationary low-rank MDPs~\citep{cheng2023lowrankreach} and proposes a prior-free, order-optimal algorithm. Although this setting is outside the scope of our main results, it faces the same core difficulty as our linear-MDP extension: the transition model has linear structure while the state space may be infinite.  A key condition in their analysis is the following reachability assumption.
\begin{assumption}\label{assum:strongest_assum}
For each round $t$ and step $h$, the transition kernel $P_h^t$ satisfies that for any $(s,a,s')\in \mathcal{S}\times \mathcal{A}\times\mathcal{S}$, $P_h^t(s'|s,a)\geq p_{\mathrm{m}}>0$. 
\end{assumption}
The issue with Assumption~\ref{assum:strongest_assum} is that it becomes vacuous in infinite state spaces, as if $\mathcal S$ is countably infinite this would make the total transition mass diverge. 
Thus, the condition can only hold with $p_{\mathrm m}=0$, while the bounds in~\citep{cheng2023lowrankreach} depend on $1/p_{\mathrm m}$. At the same time, the intuition behind the assumption is unavoidable: to detect changes in either the tabular or linear setting, the states or directions where the environment may change must be visited with non-zero probability. 
The difficulty is that, when the state space is infinite, visiting every state is impossible, so uniform state-wise reachability is too strong.  This is the main design challenge for our linear-MDP extension.  We address it by replacing uniform reachability over states with a structure-aware reachability condition, which captures only the directions needed to identify changes in the linear model.

\section{Performance Bounds of PS-RL}
\label{sec:lower_bounds}

The current literature lacks a minimax regret lower bound for the PS setting, which is essential to establish optimality. While a minimax regret lower bound exists for drifting non-stationarity in \cite{mao2022restartqucb,zhou2022restartlsviucb}, we cannot extend these results to the PS setting, as these two settings do not subsume each other. To this end, we provide information-theoretic lower bounds of the dynamic regret to characterize the fundamental limits in PS-RL in finite-horizon tabular and linear MDPs.

\begin{theorem}
    \label{thrm:ps-mdp-lower}
    For any algorithm, there exists an episodic, finite-horizon, tabular PS-MDP such that the dynamic regret of the algorithm is at least $\Omega(\sqrt{SAH^3 N_T T})$.
\end{theorem}
\textit{Proof Sketch.} We adapt the tree-based "hard-to-learn" instance construction of \cite{domingues2021lowerbound} to the PS setting. We divide the $T$ episodes into $N_T+1$ stationary segments and construct a family of $2^{N_T+1}$ MDPs indexed by binary vectors $\mathbf{i}\in\{0,1\}^{N_T+1}$. Each MDP contains a waiting state, an $A$-ary tree with $(S-3)(1 - 1/A) + 1/A$ leaves, and good/bad absorbing states (Figure \ref{fig:tree}). The first key novelty lies in the construction of the probability transition kernels of a set of ``hard-to-learn'' MDPs: for any pair of tabular MDPs whose indices differ at the $k^{\mathrm{th}}$ bit, we adversarially construct the probability transition kernels over the $k^{\mathrm{th}}$ stationary segment so that the optimal state-action-step triple in one MDP is expected to be visited the least in the other one. The second novelty is using change of measure so that the regret over the $k^{\mathrm{th}}$ stationary segment can then be lower bounded with Bretagnole-Huber inequality, which relies on the KL divergence between the probability measure induced by the policy operating on two MDPs. Then, we can show that the average expected regret over all MDPs is lower bounded with our minimax lower bound. 
The full proof of is given in Appendix \ref{app:ps-mdp-lower-proof}.

\begin{theorem}
    \label{thrm:ps-lin-mdp-lower}
    For any algorithm, there exists an episodic, finite-horizon, linear PS-MDP such that the dynamic regret of the algorithm is at least $\Omega(d\sqrt{H^3 N_T T})$.
\end{theorem}
\textit{Proof Sketch.} We generalize the hard linear MDP construction of \cite{zhou2021lowerboundlin} to the PS setting. We again divide the $T$ episodes into $N_T+1$ stationary segments and construct $2^{(d-2)(N_{T}+1)H}$ linear MDPs (with $H+2$ states and action set $\{\pm 1\}^{d-2}$) parameterized by $\varphi = \{\varphi_{h,k}: h \in [H], k \in [N_{T} + 1]\}$ where $\varphi_{h,k}\in \{ \pm \Delta\}^{d-2}$ (Figure \ref{fig:linearMDP}). $\varphi_{h,k}$ determines the small transition bias toward the good absorbing state with unit reward at step $h$ over the $k^{\mathrm{th}}$ stationary segment. We first apply Lemma 24 in \cite{zhou2021lowerboundlin} to convert the regret in each episode into the summation over all step $h$ and coordinate $j$ of the probability of the event where the sign of the $j^{\mathrm{th}}$ coordinate of $a$ differs from that of $\varphi_{h,k}$ at which the sign of $a$ differs from $\varphi_{h,k}$. By changing the order of summation, the sum of the aforementioned probability over a pair of linear MDPs is upper bounded by a constant with Pinsker's inequality and an upper bound on the KL divergence. Summing the constant over all steps, episodes, and coordinates leads to the final minimax upper bound. The full proof is given in Appendix \ref{app:ps-lin-mdp-lower-proof}.

Given the above regret bounds, it is evident that according to Table \ref{table:reg_bounds} none of the existing approaches is nearly optimal for neither the tabular nor the linear MDP setting. We stress that  MASTER's regret bounds are $\tilde{\mathcal{O}}(\sqrt{S AH^{5} N_T T})$ for tabular MDPs and $\tilde{\mathcal{O}}(d^{3/2}\sqrt{ H^4N_T T})$ for linear MDPs even if we ignore the error mentioned in Appendix \ref{sec:error_master}, which do not match the above regret bounds. Hence, to the best of our knowledge, there does not exist a nearly optimal approach for PS-RL. In what follows we take the first steps towards such an approach, establishing its algorithmic structure and then proving its theoretical properties.

\section{Our Algorithm}
\label{sec:algorithm}
DARLING is a modular detection-restart wrapper: given a stationary RL algorithm $\mathcal L$, it runs $\mathcal L$ between restarts and restarts it upon detecting non-stationarity. 
Its key design choice is to separate detection from learning by periodically inserting \emph{probing episodes}, whose samples are used only for change detection and never to update $\mathcal L$. 
This preserves the modularity of the base learner and ensures detection of changes in the MDP. We present the detailed design through four aspects: \emph{what} to detect, \emph{where} to detect, \emph{when} to probe, and \emph{how} to test for changes.

\textbf{What to Detect.} The first design choice is the detection signal. MASTER detects non-stationarity indirectly, by testing whether the stationary regret guarantees are violated due to changes. As discussed in Section~\ref{sec:issues}, this test is highly conservative in realistic sample regimes and may be insufficient for MDPs. Instead, DARLING detects changes directly by monitoring the reward samples and state transitions to observe changes in the mean reward $r_h^t$ and transition kernel $P_h^t$ through two dedicated tests. This ``detect-the-model'' viewpoint gives interpretable detection statistics without relying on regret-violation tests.

\textbf{Where to Detect.}
Having fixed the detection signal, DARLING must decide where such signals should be monitored, i.e., the triples $(s,a,h)$ at which we sample for detection. We call a set of triples $\mathcal P\subseteq \mathcal S\times\mathcal A\times[H]$ a \emph{probe set} if at least one triple in $\mathcal P$ undergoes a shift in $r^{t}_{h}(s,a)$ or $P^{t}_{h}(s'|s,a)$ whenever a change occurs. In the fully prior-free tabular setting, the probe set is $\mathcal{P}=\mathcal{S}\times\mathcal{A}\times[H]$, since in the worst case, only the mean reward or transition kernel at one arbitrary $(s,a,h)$ changes at a change-point.
The challenge is to ensure sufficient samples for each triple in the probe set. Since a good learning algorithm selects suboptimal actions less frequently, DARLING enforces coverage through scheduled probing episodes. In linear MDPs, where the state space may be infinite, the probe set can instead be restricted using the linear structure; we defer this extension to Appendix~\ref{app:linear_mdps}.

\textbf{When to Detect.} DARLING designates one probing episode every $\lceil 1/\alpha_k\rceil$ episodes, where $\alpha_k$ is the probing frequency after the $(k-1)^{\mathrm{th}}$ restart. During a probing episode, DARLING overrides the base learner and samples actions uniformly at random from $\mathcal{A}$. The resulting reward and transition samples are used only for change detection and are not passed to $\mathcal{L}$. Here, we highlight that $\alpha_k$ balances sample accumulation and detection delay simultaneously, and we optimize its value to ensure reliable detection, which ultimately leads to optimal performance.

\textbf{How to Detect.}
DARLING reduces change detection to a collection of scalar mean-shift tests. It updates two types of histories elaborated as follows:

\textit{Reward histories.} For each $(s,a,h)\in\mathcal P$, DARLING maintains a history $\mathcal H^{(r)}_{(s,a,h)}.$ Whenever $(s,a,h)$ is probed, the observed stochastic reward $R_h^t(s,a)$ with mean $r_h^t(s,a)$ is appended to this history.

\textit{Transition histories.} For transitions, DARLING encodes the next state into one-hot vectors. For each $(s,a,h,\tilde s)\in\mathcal P\times\mathcal S$, it maintains a history $\mathcal H^{(P)}_{(s,a,h,\tilde s)}$ of Bernoulli random variables. When the next state is $s'$, DARLING appends $1$ to $\mathcal H^{(P)}_{(s,a,h,s')}$
and $0$ to $\mathcal H^{(P)}_{(s,a,h,\tilde s)}$ for every $\tilde s\neq s'$. This stream is an independent Bernoulli process with mean $P_h^t(\tilde s\mid s,a)$, which possibly changes at a change-point.

Thus, in tabular MDPs, DARLING reduces PS-RL change detection to detecting mean shifts over the finite collection of reward streams indexed by $(s,a,h)$ and transition streams indexed by $(s,a,h,\tilde s)$. Every time a history is updated, DARLING applies a detector $\mathcal D$ to the updated stream.
\texttt{Test 1} applies $\mathcal D$ to reward histories, while \texttt{Test 2} applies $\mathcal D$ to transition histories. If any monitored stream triggers the tests, DARLING sets a restart flag and, at the end of the episode, resets both the base learner $\mathcal L$ and all detection histories.

\textbf{The DARLING Algorithm.}
Algorithm~\ref{alg:DARLING} gives the full DARLING wrapper. The algorithm alternates between two modes. In ordinary episodes, DARLING simply runs and updates the stationary learner $\mathcal L$. In probing episodes scheduled every $\lceil 1/\alpha_k\rceil$ episodes after the $(k-1)^{\mathrm{th}}$ restart, DARLING overrides $\mathcal L$, samples actions uniformly at random, and only updates the detection histories with the observed reward and transition samples. DARLING then applies $\mathcal{D}$ described above to these histories, and restarts $\mathcal{L}$ and clears all histories at the end of the episode when $\mathcal{D}$ signals a change. 
\begin{algorithm*}
\small
\caption{\small \textbf{D}etection \textbf{A}ugmented \textbf{R}einforcement \textbf{L}earn\textbf{ING} (\textbf{DARLING})}
\label{alg:DARLING}
\textbf{Input:} stationary algorithm $\mathcal{L}$, detector $\mathcal{D}$, state space $\mathcal{S}$, action space $\mathcal{A}$,
probing frequencies $\{\alpha_k\}_{k\ge 1}$. \\
\textbf{Initialization:} detection $\tau\leftarrow 0$, counter $k \leftarrow 1$, reward history and transition history  $\mathcal{H}^{(r)}_{(s,a,h)}, \mathcal{H}^{(P)}_{(s,a,h,s')}\leftarrow\emptyset$ 
for all $s,s'\in \mathcal{S}$, $a \in \mathcal{A}$, and $h \in [H]$.

\begin{algorithmic}[1]
\FOR{$t=1,2,\dots,T$}
        %\STATE Set starting state $s_1^t\leftarrow s_1$
        \FOR{$h=1,2,\dots,H$}
            \IF{ $(t-\tau-1) \bmod \lceil 1/\alpha_{k} \rceil = 0$ }
                \STATE Set $s\leftarrow s^t_h$ and select an action $a$ from $\mathcal{A}$ uniformly at random \hfill{$\triangleright$ forced probing}
                \STATE Transition to $s'\leftarrow s_{h+1}^t$ and append received reward $R_h^t(s,a)$ into history $\mathcal{H}_{(s,a,h)}^{(r)}$
                \STATE Add ``1" to history $\mathcal{H}_{(s,a,h,s')}^{(P)}$ and ``0" to $\mathcal{H}_{(s,a,h,\tilde{s})}^{(P)}$ for all $\tilde{s} \in \mathcal{S}/s'$
                \STATE \texttt{Test 1} $\leftarrow \mathcal{D}(\mathcal{H}^{(r)}_{(s,a,h)})$,\:\: \texttt{Test 2} $\leftarrow \mathcal{D}(\mathcal{H}^{(P)}_{(s,a,h,s')})$  for all $s'\in \mathcal{S}$ \hfill{$\triangleright$ non-stationarity detection}
            \ENDIF
            \STATE\textbf{else} Run and update $\mathcal{L}$ \hfill{$\triangleright$ stationary learning}  
        \ENDFOR
        \IF{\texttt{Test 1} \textbf{or} \texttt{Test 2} signal {\texttt{Restart}}}
            \STATE Reset the RL algorithm $\mathcal{L}$; empty all histories $\mathcal{H}$ used for detection \hfill{$\triangleright$ restart learning process}
            \STATE $\tau\leftarrow t, \quad k \leftarrow k + 1$ %\texttt{Restart} $\leftarrow$ \texttt{False}
        \ENDIF  
\ENDFOR
\end{algorithmic}
\end{algorithm*}

\section{Theoretical Analysis}
\label{sec:theory}

\subsection{On Effective and Feasible Detection}\label{sec:probe_set}

Effective detection in DARLING relies on checking the possible mean-shift in the histories of the samples from the finite probe set $\mathcal{P}$. To ensure reliable detection, DARLING must also collect enough samples from every triple in $\mathcal{P}$. A triple $(s,a,h)$ may be highly informative for detecting a change, but it may be rarely visited by any arbitrary policy. Thereupon, the remaining requirement is that the states themselves are visited often enough under the probing policy. Let $\pi_{\mathrm{U}}$ denote the \emph{uniform} probing policy used by DARLING in probing episodes. For an episode indexed by $t$ with initial state $s^t_1$, let $\mathbb{P}^{\pi_{\mathrm{U}}}(\cdot)$ denote the probability measure under policy $\pi_{\mathrm{U}}$ at episode $t$. Define the step-$h$ occupancy of state $s$ under the uniform probing policy by $p_{h,t,s_1^{t}}^{\pi_{\mathrm{U}}}(s)
\coloneqq
\mathbb{P}^{\pi_{\mathrm{U}}}(s_h^{t}=s \mid s_1^{t})$.

\begin{assumption}\label{assum:reachability}
There exists $p_{\mathrm{m}}>0$ such that for $t\in[N_T+1]$, initial state $s_1^t\in\mathcal{S}$, and $h\in[H]$,
$%\begin{equation*}
\min_{s}\ p_{h,t,s_1^{t}}^{\pi_{\mathrm{U}}}(s)
\;\ge\; p_\mathrm{m}.
$%\end{equation*}
\end{assumption}

Assumption~\ref{assum:reachability} ensures that all states in $\mathcal{P}$ are reachable with probability at least $p_{\mathrm{m}}$ under the uniform probing policy. Notice that this condition only needs to hold for the uniform policy $\pi_{\mathrm{U}}$ DARLING employs. Its main implication is that after $n$ episodes within a segment, each monitored triple $(s,a,h)\in\mathcal{P}$ accrues
$
\tilde{\Omega}\!\left(\frac{\alpha_k}{A}\,p_{\mathrm{m}}\,n\right)
$
samples in expectation.

\subsection{On Detector Selection}\label{sec:detector}
The stopping time $\tau$ of a change detector $\mathcal{D}$ denotes the time (episode) at which a change is identified. Let $\mathbb{P}_{\nu}$ and $\mathbb{E}_{\nu}$ be the probability and expectation with change-point at $\nu$, and $\mathbb{P}_{\infty}$ and $\mathbb{E}_{\infty}$ be the ones with no change-point. 
The \emph{latency} $\ell_{\mathcal{D}}$ is the length of time post-change within which a change is declared with probability $1-\delta_\mathrm{D}$, i.e.,
\begin{equation*}
\ell_{\mathcal{D}} \coloneqq \inf\{t\in[T]:\mathbb{P}_{\nu}(\tau\geq \nu+t)\leq \delta_{\mathrm{D}},~\forall\,\nu\in[m_{\mathcal{D}}+1,T-t]\}
\end{equation*}
where $m_{\mathcal{D}}$ is the length of the pre-change window at which no changes occur, and $\delta_{\mathrm{D}}$ parametrizes the late detection probability.
A detector seeks to minimize $\ell_{\mathcal{D}}$ while ensuring low false-alarm probability over horizon $T$, namely $\mathbb{P}_{\infty}(\tau\le T)\le \delta_{\mathrm{F}}$ with $\delta_{\mathrm{F}}\in(0,1)$. To ensure order-optimal regret for DARLING, the detector $\mathcal{D}$ must satisfy the following property.
\begin{property}\label{proper:good_CD}
$\ell_{\mathcal{D}}, m_{\mathcal{D}} = \mathcal{O}( \log(T/(\delta_{\mathrm{D}}\delta_F)))$.
\end{property}
This property has been widely used in the NS bandit (with detection-restart approach) literature \cite{besson2022efficient,gerogiannis2025dal,huang2025dabprocedures}, due to its good regret properties. Specifically, with $\delta_{\mathrm{F}}=\delta_{\mathrm{D}}=T^{-\gamma}$ for any $\gamma>1$, Property~\ref{proper:good_CD} implies
$m_{\mathcal{D}}+\ell_{\mathcal{D}}=\tilde{\mathcal{O}}(1)$, so detection overhead per stationary segment is polylogarithmic and does not affect the leading-order regret rates. Regarding the existence of detectors satisfying Property~\ref{proper:good_CD}, prior work shows that the Generalized Likelihood Ratio (GLR) and Generalized Shiryaev--Roberts (GSR) tests
\citep{huang2025sequentialchangedetectionlearning, huang2026finite} satisfy
Property~\ref{proper:good_CD}. Due to space constraints, we provide the details of the GLR and GSR in Appendix \ref{app:glr-gsr}. We emphasize that DARLING is \emph{detector-agnostic}:
our regret analysis depends only on Property~\ref{proper:good_CD}, not on any specific implementation of $\mathcal{D}$. 

\textbf{From sample complexity to episode separation.} Property~\ref{proper:good_CD} states the number of samples required for reliable detection in a \emph{single monitored stream}.
In DARLING, samples for a fixed probed triple $(s,a,h)$ arrive only during probing episodes, and only when (i) the  trajectory of $\pi_{\mathrm{U}}$ visits $s$ at step $h$, and (ii) the probing policy samples action $a\in\mathcal{A}$ (uniformly).
In each probing episode, each monitored stream $(s,a,h)$ is sampled with probability at least $p_{\mathrm{m}}/A$.
Therefore, to obtain $n$ samples after the $k^{\mathrm{th}}$ restart with high probability, we require roughly $\tilde{\Omega}(An/(p_{m}\alpha_{k}))$ episodes.
Consequently, we define the following quantities by taking this sampling complexity into consideration.

\begin{definition}\label{def:sep_scaling}
Define $m_{k}\coloneqq\lceil 1/\alpha_{k} \rceil \lceil m_{\mathcal{D}}A/p_{\mathrm{m}} + (A^{2}\log T)/(4p_{\mathrm{m}}^{2}) + \sqrt{(m_{\mathcal{D}}\log (T)A^{3})/(2p_{\mathrm{m}}^{3}) + ((\log T)^{2}A^{4})/(16p_{\mathrm{m}}^{4})} \rceil$ and $\ell_{k}\coloneqq\lceil 1/\alpha_{k} \rceil\lceil \ell_{\mathcal{D}}A/p_{\mathrm{m}} + (A^{2}\log T)/(4p_{\mathrm{m}}^{2}) + \sqrt{(\ell_{\mathcal{D}}\log (T)A^{3})/(2p_{\mathrm{m}}^{3}) + ((\log T)^{2}A^{4})/(16p_{\mathrm{m}}^{4})} \rceil$ for $k \in [N_{T}]$. 
\end{definition}

Hence, to ensure that there are enough samples between change-points, we make the assumption.

\begin{assumption}
\label{assum:cp_assum}
Assume $\nu_{1}$ $\geq$ $ m_{1}$ and $\nu_{k}$ $-$ $\nu_{k-1}$ $\geq$ $\ell_{k-1}$ $+$ $m_{k}$ for $k\in \{2, \dots, N_{T}\}$.
\end{assumption}

\subsection{DARLING's Regret} 
Given the probe construction and feasibility condition in Section~\ref{sec:probe_set}, the detector requirements in Section~\ref{sec:detector}, we can now characterize DARLING's regret.

\begin{theorem} \label{thrm:darling_reg}
Consider the tabular setting, a detector $\mathcal{D}$ that satisfies Property \ref{proper:good_CD}, a stationary input algorithm $\mathcal{L}$ with regret upper bound $\mathcal{R}_{\mathcal{L}}$, a probe set $\mathcal{P}$ and forced probing frequencies $(\alpha_{k})_{k = 1}^{T}$.
If Assumptions \ref{assum:reachability} and \ref{assum:cp_assum} hold, $\alpha_{k} =\sqrt{kSAH}/(2\sqrt{T}\log^2 T)$, $\delta_{\mathrm{F}}=\delta_{\mathrm{D}}=T^{-\gamma}$, with $\gamma>1$, and $\mathcal{L}$ is order-optimal with $\mathcal{R}_{\mathcal{L}}(T)=\tilde{\mathcal{O}}(\sqrt{SAH^3T})$, then DARLING is order-optimal.
\end{theorem}

The proof of Theorem \ref{thrm:darling_reg} is given in Appendix \ref{app:darling_reg_proof}.  
By instantiating $\mathcal{L}$ with state-of-the-art stationary algorithms, e.g., UCB-MQ~\citep{menard2021ucbmq}, 
we recover the upper bounds in Table~\ref{table:reg_bounds}.

\section{Extending to Linear MDPs}
\label{sec:lin_mdp_exten}
While we extend DARLING to linear MDPs, due to space constraints and to enhance readability, we defer its full implementation and construction specifics to Appendix \ref{app:linear_mdps}. In this section, we elaborate how this extension is done, highlighting the important components. The core design choices that differ with the tabular setting are the \emph{probe set construction} and \emph{identification}, and the \emph{transition detection}.

In the linear case the state space $\mathcal{S}$ can be infinite, therefore a condition similar to Assumption \ref{assum:reachability} is not feasible. To circumvent this, DARLING exploits the linear structure of the MDP. Specifically, to identify changes, one does not need to visit every possible $(s,a,h)$, but only a set of triples whose $(s,a)$ correspond to feature vectors $\phi(s,a)$ that span $\mathbb{R}^d$. This is because both the reward function and the transition kernel depend linearly on an underlying parameter. By sampling linearly independent directions all changes in the underlying parameters can be identified. However, since the underlying parameters can change with $h$, each step $h$ requires its own set of linearly independent feature vectors. Notice that if such features vectors do not exist, then all changes are going to be invisible not to just DARLING, but to  \emph{every} algorithm. Hence, in this case the probe set will be given as $\mathcal{P}=\{\mathcal{P}_h\}_{h=1}^H$, where each $\mathcal{P}_h$ corresponds to a set of $(s,a,h)$ with the same $h$ whose $\phi(s,a)$'s are maximal linearly independent sets. As is evident, unlike the tabular case, $\mathcal{P}$ cannot be known beforehand and may not exist. To ensure both reachability and existence, we consider the following assumption.

\begin{assumption}\label{assum:reachability_lin_mdps}
Let $q_{h,t}(s,a)=p_{h,t,s_1^{t}}^{\pi_{\mathrm{U}}}(s)/A$. We assume that for every $h \in [H]$, there exists a $\mathcal{P}_h$ such that for $t\in[N_T+1]$, initial state $s_1^t\in\mathcal{S}$,
$\min_{(s,a)\in \mathcal{P}_h}\ q_{h,t}(s,a)
\;\ge\; 1/(2d)$. 
\end{assumption}

The probability lower bound in this case ensures optimal regret. Unlike the tabular case, while $\mathcal{P}_h$ for each $h$ and for each $t$ exists, these sets are unknown a-priori. To this end, DARLING dedicates a specific number of episodes in order to \emph{identify the probe set}. DARLING employs the uniform policy for $n_{0}$ episodes and records the number of times it has observed the various triples $(s,a,h)$. After the identification episodes are over, it selects the $d$ most visited state-action pairs to append into $\mathcal{P}_h$. This procedure produces a valid probe set $\mathcal{P}$ with high probability. To ensure enough episodes for probe set identification, we need to modify Assumption \ref{assum:cp_assum} by setting $p_{\mathrm{m}} = 1/(2d)$, $A = 1$, and
\begin{assumption}
\label{assum:cp_assum_lin_mdps}
Assume $\nu_{1}$ $\geq$ $ n_{0} + m_{1}$ and $\nu_{k}$ $-$ $\nu_{k-1}$ $\geq$ $n_{0} + \ell_{k-1}$ $+$ $m_{k}$ for $k\in \{2, \dots, N_{T}\}$.
\end{assumption}

The final important distinction is in the detection mechanism of the transition probabilities. Unlike the tabular setting, the transition cannot be mapped to an one-hot vector. However, since the vectors $\phi$ are known in advance, the detection of transitions is done on the \emph{expected feature vector of the next state}. That is for a given triple $(s,a,h)$, DARLING maintains $[d]\times \mathcal{A}$ histories and employs detection on all $d$ elements of $\phi(s',a)$ for every action. If the transition probability has changed, then $\mathbb{E}[\phi(s',a)]$ should be different for some $a\in\mathcal{A}$. To this end, DARLING's regret is given as follows.

\begin{theorem} \label{thrm:darling_reg_lin_mdps}
Consider linear MDPs, a detector $\mathcal{D}$ that satisfies Property \ref{proper:good_CD}, an order-optimal stationary input algorithm $\mathcal{L}$ with regret upper bound $\mathcal{R}_{\mathcal{L}}(T)=\tilde{\mathcal{O}}(d\sqrt{H^3T})$, a probe set $\mathcal{P}$ and forced probing frequencies $(\alpha_{k})_{k = 1}^{T}$.
If Assumptions \ref{assum:reachability_lin_mdps} and \ref{assum:cp_assum_lin_mdps} hold, $n_{0} = 32Ad\log({128AHdT}/{p_{\mathrm{m}}})$, $\alpha_{k} =\sqrt{kd}/(2\sqrt{T}\log^2 T)$, and $\delta_{\mathrm{F}}=\delta_{\mathrm{D}}=T^{-\gamma}$ with $\gamma>1$, then DARLING is order-optimal.
\end{theorem}

\section{Experimental Study}
\label{sec:experiments}

\textbf{Baselines and tuning.} We compare DARLING against the state-of-the-art PS-RL methods summarized in Table~\ref{table:reg_bounds}, including both prior-free and prior-based approaches. Even though MASTER's analysis indicates a flaw, we still compare with it as a prior-free baseline. All baselines are tuned following their respective original papers. DARLING employs the sub-Bernoulli GLR \citep{besson2022efficient} as the detector $\mathcal{D}$ and uses a threshold, $\beta_{\mathrm{GLR}}(n,\delta_F)=\log(n^{3/2}/{\delta_\mathrm{F}})$ with $\delta_F=1/\sqrt{T}$. Finally, we set $\alpha_k$ according to Theorems \ref{thrm:darling_reg}, \ref{thrm:darling_reg_lin_mdps}. We instantiate DARLING with an order-optimal stationary base learner in each regime.
For tabular MDPs we use UCB-MQ~\citep{menard2021ucbmq}, and for linear MDPs we use LSVI-UCB++~\citep{he2023lsviucbplusplus}. Rewards are already bounded in $[0,1]$. For transition detection, we feed successor-feature coordinates into the detector after mapping each feature value to $[0,1]$.

\textbf{Environments.}
We evaluate on $10$ different benchmarks, $5$ tabular MDPs and $5$ linear MDPs. For tabular MDPs, we evaluate on the NS variant of Bidirectional Diabolical Combination Lock from \cite{mao2022restartqucb}, and our NS versions of DeepSea \cite{osband2020deepsea}, FourRoom \cite{sutton1999fourroom}, NRoom \cite{rlberry2021nroom} and Forked RiverSwim \cite{russo2023forkedriverswim}. For Linear MDPs, we evaluate on the NS Chain Lock of \cite{zhou2022restartlsviucb}, and on our NS versions of a Simplex-based linear MDP \cite{jin2020lsviucb,zhou2022restartlsviucb}, GARNET \cite{archibald1995garnet,bhatnagar2009garnet}, Anchor-feature MDP \cite{yang2019anchor} and a Block-structured low-rank linear MDP \cite{agarwal2020lowrank}. The full environment details are provided in the Appendix \ref{app:exp_envs}.

\textbf{Non-stationarity protocols and horizon.}
We test under both PS and drifting non-stationarity for a total of $T=50000$ episodes.
In the PS setting, we adopt a geometric change-point model~\citep{gerogiannis2024blackboxfeas} to stress-test prior-free adaptation:
segment lengths are i.i.d.\ geometric with parameter $T^{-\xi}$ for $\xi\in\{0.4,0.6,0.8\}$, yielding an average of up to $659$ changes over the horizon.
This is substantially more challenging than the settings used in~\cite{mao2022restartqucb} (5 changes) and~\cite{zhou2021lowerboundlin} (20 changes).
For drifting experiments, we use a linear/smooth drifting schedule for all cases. The analytical non-stationarity protocols are given in Appendix \ref{app:exp_envs}.
Performance is reported in terms of cumulative reward.

\textbf{Probe-set construction.} In tabular MDPs, we set $\mathcal{P}=\mathcal{S}\times \mathcal{A}\times [H]$. On the other hand for linear MDPs, we found out that the selection of the probe set had very little effect to the performance of the algorithm. To this end, we just greedily select as many $(s,a)$ pairs for each $h$ such that their $\phi(s,a)$'s are linearly independent. Theoretically, any maximal independent probe set yields identical asymptotic guarantees. When feature vectors have larger norms, the detector triggers faster due to larger shift magnitude in the mean reward or transition kernel. However, we cannot optimize the probe set in prior-free settings since the reward/transition structure is unknown. Still, DARLING's performance is not affected by probe set choice in practice: varying probe sets across random seeds did not meaningfully affect performance.

Figure~\ref{fig:experiments} reports cumulative reward across all benchmarks. Due to space constraints, higher resolution images of the plots are provided in Appendix \ref{app:enhanced_plots}.
DARLING achieves the highest cumulative reward in both tabular and linear MDPs across all PS configurations, and remains strong even under drifting non-stationarity.
Among prior-free methods, DARLING consistently outperforms MASTER highlighting the advantage of directly detecting changes in the MDP rather than relying on regret-violation tests.
Notably, in the drifting tabular regimes, DARLING also surpasses the best prior-based baseline, Restart-Q-UCB. Despite DARLING's multiple detection tests, it is computationally efficient: it runs in $0.83$ ms/episode in tabular MDPs and $1.67$ ms/episode in linear MDPs, faster than MASTER and comparable with other methods. Finally, it is important to emphasize that Assumptions \ref{assum:cp_assum}, \ref{assum:cp_assum_lin_mdps}, \ref{assum:reachability} and \ref{assum:reachability_lin_mdps} are only necessary for theoretical analyses. None of our experiments enforce these constraints, and, in fact, violate them in almost all cases considered.

\begin{figure*}[t]
\centering

% -------- content row (plot + table) --------
\begin{minipage}[t]{0.81\linewidth}
  \vspace{0pt}
  \centering
  \includegraphics[width=\linewidth]{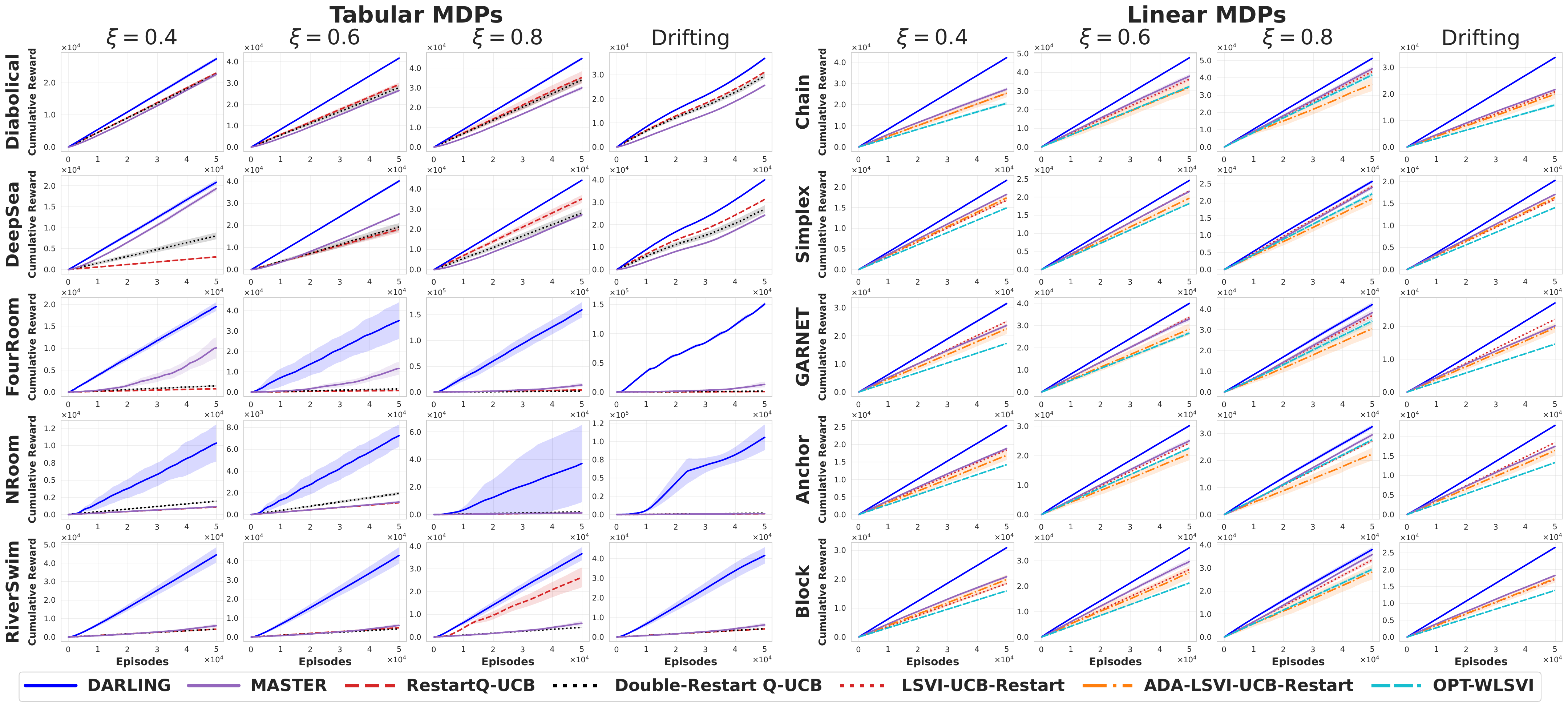}
\end{minipage}\hfill
\begin{minipage}[t]{0.18\linewidth}
  \vspace{10pt}
  \centering

  {\scriptsize\centering\emph{Runtime averaged per episode (mean over seeds/settings).}\par}

  \vspace{3pt}

  {\tiny
  \renewcommand{\arraystretch}{1.1}
  \setlength{\tabcolsep}{2pt}

  \begin{tabular}{@{}l r@{}}
    \hline
    \textbf{Method} & \textbf{ms/ep.} \\
    \hline
    \multicolumn{2}{@{}c@{}}{\textbf{Tabular MDPs}} \\
    \hline
    DARLING & 0.825 \\
    Double-RestartQ & 0.224\\
    MASTER & 1.644 \\
    \shortstack[l]{RestartQ-UCB} & 0.403 \\
    \hline
    \multicolumn{2}{@{}c@{}}{\textbf{Linear MDPs}} \\
    \hline
    DARLING & 1.672 \\
    MASTER & 3.464 \\
    LSVI-UCB-Restart & 1.064 \\
    \shortstack[l]{ADA-LSVI-UCB} & 1.271 \\
    OPT-WLSVI & 1.136 \\
    \hline
  \end{tabular}
  }
\end{minipage}
% -------- full-width caption --------
\caption{Cumulative reward results for the experiments (higher=better). Left: Tabular MDPs, Right: Linear MDPs. DARLING outperforms all state-of-the-art baselines in every scenario.}
\label{fig:experiments}
\end{figure*}
\section{Summary and Outlook}
\label{sec:summary}
In this work, we studied PS-RL in the episodic, finite-horizon setting under both tabular and linear structures, without knowledge of the changes. We identified issues with current state-of-the-art methods, and provided the first, to our knowledge, performance bounds for both linear and tabular settings, to characterize the difficulty of the problem and the state of the literature. To this end, we introduced DARLING, a modular, \emph{prior-free} detection-restart framework for PS-RL. DARLING \emph{detects the model} by monitoring mean shifts in probed reward streams and transition streams, and can wrap \emph{any} stationary RL algorithm with optimal regret. Under certain conditions, DARLING is the first algorithm, to our knowledge, that attains near-optimal dynamic regret in both PS tabular and linear MDPs. Importantly, DARLING consistently outperforms all alternative state-of-the-art baselines in PS benchmarks and remains robust under drifting non-stationarity, while retaining practical runtime.

While DARLING improves the current state-of-the-art, it also has limitations. Irrespective of its good performance, its theoretical analysis relies on change-point separation and reachability assumptions which nonetheless limit the extent to which DARLING achieves fully prior-free theoretical optimality. At the same time, in its current structure, DARLING cannot be applied to an infinite action setting due to its need for finite memory. Given MASTER's shortcomings, an interesting direction would be to investigate whether DARLING's assumptions can be circumvented. On the other hand, future work also includes the extension of DARLING to infinite-horizon MDPs.

\bibliographystyle{plain}
\bibliography{main}

@article{huang2026finite,
  title={Finite-horizon quickest change detection balancing latency with false alarm probability},
  author={Huang, Yu-Han and Veeravalli, Venugopal V},
  journal={Sequential Analysis},
  pages={1--29},
  year={2026},
  publisher={Taylor \& Francis}
}

@inproceedings{cai2017realtimebidding,
	title        = {Real-Time Bidding by Reinforcement Learning in Display Advertising},
	author       = {Cai, Han and Ren, Kan and Zhang, Weinan and Malialis, Kleanthis and Wang, Jun and Yu, Yong and Guo, Defeng},
	year         = 2017,
	booktitle    = {Proceedings of the Tenth ACM International Conference on Web Search and Data Mining},
	location     = {Cambridge, United Kingdom},
	publisher    = {Association for Computing Machinery},
	address      = {New York, NY, USA},
	series       = {WSDM '17},
	pages        = {661–670},
	doi          = {10.1145/3018661.3018702},
	isbn         = 9781450346757,
	url          = {https://doi.org/10.1145/3018661.3018702},
	numpages     = 10
}

@article{shortreed2011treatment,
	title        = {Informing sequential clinical decision-making through reinforcement learning: an empirical study},
	author       = {Shortreed, Susan M. and Laber, Eric and Lizotte, Daniel J. and Stroup, T. Scott and Pineau, Joelle and Murphy, Susan A.},
	year         = 2011,
	month        = jul,
	journal      = {Mach. Learn.},
	publisher    = {Kluwer Academic Publishers},
	address      = {USA},
	volume       = 84,
	number       = {1–2},
	pages        = {109–136},
	doi          = {10.1007/s10994-010-5229-0},
	issn         = {0885-6125},
	url          = {https://doi.org/10.1007/s10994-010-5229-0},
	issue_date   = {July      2011},
	numpages     = 28
}

@article{chen2020trafficmanage,
	title        = {Toward A Thousand Lights: Decentralized Deep Reinforcement Learning for Large-Scale Traffic Signal Control},
	author       = {Chen, Chacha and Wei, Hua and Xu, Nan and Zheng, Guanjie and Yang, Ming and Xiong, Yuanhao and Xu, Kai and Li, Zhenhui},
	year         = 2020,
	month        = 4,
	journal      = {Proceedings of the AAAI Conference on Artificial Intelligence},
	volume       = 34,
	number       = {04},
	pages        = {3414--3421},
	DOI          = {10.1609/aaai.v34i04.5744},
	url          = {https://ojs.aaai.org/index.php/AAAI/article/view/5744},
	abstractNote = {&lt;p&gt;Traffic congestion plagues cities around the world. Recent years have witnessed an unprecedented trend in applying reinforcement learning for traffic signal control. However, the primary challenge is to control and coordinate traffic lights in large-scale urban networks. No one has ever tested RL models on a network of more than a thousand traffic lights. In this paper, we tackle the problem of multi-intersection traffic signal control, especially for large-scale networks, based on RL techniques and transportation theories. This problem is quite difficult because there are challenges such as scalability, signal coordination, data feasibility, etc. To address these challenges, we (1) design our RL agents utilizing ‘pressure’ concept to achieve signal coordination in region-level; (2) show that implicit coordination could be achieved by individual control agents with well-crafted reward design thus reducing the dimensionality; and (3) conduct extensive experiments on multiple scenarios, including a real-world scenario with 2510 traffic lights in Manhattan, New York City &lt;sup&gt;1&lt;/sup&gt; &lt;sup&gt;2&lt;/sup&gt;.&lt;/p&gt;}
}

@inproceedings{agrawal2019inventory,
	title        = {Learning in Structured MDPs with Convex Cost Functions: Improved Regret Bounds for Inventory Management},
	author       = {Agrawal, Shipra and Jia, Randy},
	year         = 2019,
	booktitle    = {Proceedings of the 2019 ACM Conference on Economics and Computation},
	location     = {Phoenix, AZ, USA},
	publisher    = {Association for Computing Machinery},
	address      = {New York, NY, USA},
	series       = {EC '19},
	pages        = {743–744},
	doi          = {10.1145/3328526.3329565},
	isbn         = 9781450367929,
	url          = {https://doi.org/10.1145/3328526.3329565},
	numpages     = 2
}

@book{sutton2018reinforcement,
  title={Reinforcement learning: An introduction},
  author={Sutton, Richard S and Barto, Andrew G},
  year={2018},
  publisher={MIT press}
}

@misc{gerogiannis2025dal,
      title={DAL: A Practical Prior-Free Black-Box Framework for Non-Stationary Bandits}, 
      author={Argyrios Gerogiannis and Yu-Han Huang and Subhonmesh Bose and Venugopal V. Veeravalli},
      year={2025},
      eprint={2501.19401},
      archivePrefix={arXiv},
      primaryClass={cs.LG},
      url={https://arxiv.org/abs/2501.19401}, 
}

@InProceedings{domingues2021lowerbound,
  title = 	 {Episodic Reinforcement Learning in Finite MDPs: Minimax Lower Bounds Revisited},
  author =       {Domingues, Omar Darwiche and M{\'e}nard, Pierre and Kaufmann, Emilie and Valko, Michal},
  booktitle = 	 {Proceedings of the 32nd International Conference on Algorithmic Learning Theory},
  pages = 	 {578--598},
  year = 	 {2021},
  editor = 	 {Feldman, Vitaly and Ligett, Katrina and Sabato, Sivan},
  volume = 	 {132},
  series = 	 {Proceedings of Machine Learning Research},
  month = 	 {16--19 Mar},
  publisher =    {PMLR},
  url = 	 {https://proceedings.mlr.press/v132/domingues21a.html}
}

@misc{mao2022restartqucb,
      title={Model-Free Non-Stationary RL: Near-Optimal Regret and Applications in Multi-Agent RL and Inventory Control}, 
      author={Weichao Mao and Kaiqing Zhang and Ruihao Zhu and David Simchi-Levi and Tamer Başar},
      year={2022},
      eprint={2010.03161},
      archivePrefix={arXiv},
      primaryClass={cs.LG},
      url={https://arxiv.org/abs/2010.03161}, 
}

@InProceedings{menard2021ucbmq,
  title = 	 {UCB Momentum Q-learning: Correcting the bias without forgetting},
  author =       {Menard, Pierre and Domingues, Omar Darwiche and Shang, Xuedong and Valko, Michal},
  booktitle = 	 {Proceedings of the 38th International Conference on Machine Learning},
  pages = 	 {7609--7618},
  year = 	 {2021},
  editor = 	 {Meila, Marina and Zhang, Tong},
  volume = 	 {139},
  series = 	 {Proceedings of Machine Learning Research},
  month = 	 {18--24 Jul},
  publisher =    {PMLR},
  url = 	 {https://proceedings.mlr.press/v139/menard21b.html}
}

@misc{touati2021optwlsvi,
      title={Efficient Learning in Non-Stationary Linear Markov Decision Processes}, 
      author={Ahmed Touati and Pascal Vincent},
      year={2021},
      eprint={2010.12870},
      archivePrefix={arXiv},
      primaryClass={cs.LG},
      url={https://arxiv.org/abs/2010.12870}, 
}

@article{
zhou2022restartlsviucb,
title={Nonstationary Reinforcement Learning with Linear Function Approximation},
author={Huozhi Zhou and Jinglin Chen and Lav R. Varshney and Ashish Jagmohan},
journal={Transactions on Machine Learning Research},
issn={2835-8856},
year={2022},
url={https://openreview.net/forum?id=nS8A9nOrqp},
note={}
}

@InProceedings{he2023lsviucbplusplus,
  title = 	 {Nearly Minimax Optimal Reinforcement Learning for Linear {M}arkov Decision Processes},
  author =       {He, Jiafan and Zhao, Heyang and Zhou, Dongruo and Gu, Quanquan},
  booktitle = 	 {Proceedings of the 40th International Conference on Machine Learning},
  pages = 	 {12790--12822},
  year = 	 {2023},
  editor = 	 {Krause, Andreas and Brunskill, Emma and Cho, Kyunghyun and Engelhardt, Barbara and Sabato, Sivan and Scarlett, Jonathan},
  volume = 	 {202},
  series = 	 {Proceedings of Machine Learning Research},
  month = 	 {23--29 Jul},
  publisher =    {PMLR},
  url = 	 {https://proceedings.mlr.press/v202/he23d.html}
}

@InProceedings{jin2020lsviucb,
  title = 	 {Provably efficient reinforcement learning with linear function approximation},
  author =       {Jin, Chi and Yang, Zhuoran and Wang, Zhaoran and Jordan, Michael I},
  booktitle = 	 {Proceedings of Thirty Third Conference on Learning Theory},
  pages = 	 {2137--2143},
  year = 	 {2020},
  editor = 	 {Abernethy, Jacob and Agarwal, Shivani},
  volume = 	 {125},
  series = 	 {Proceedings of Machine Learning Research},
  month = 	 {09--12 Jul},
  publisher =    {PMLR},
  url = 	 {https://proceedings.mlr.press/v125/jin20a.html}
}

@InProceedings{ortner2020whystationarybadinns,
  title = 	 {Variational Regret Bounds for Reinforcement Learning},
  author =       {Ortner, Ronald and Gajane, Pratik and Auer, Peter},
  booktitle = 	 {Proceedings of The 35th Uncertainty in Artificial Intelligence Conference},
  pages = 	 {81--90},
  year = 	 {2020},
  editor = 	 {Adams, Ryan P. and Gogate, Vibhav},
  volume = 	 {115},
  series = 	 {Proceedings of Machine Learning Research},
  month = 	 {22--25 Jul},
  publisher =    {PMLR},
  url = 	 {https://proceedings.mlr.press/v115/ortner20a.html}
}

@inproceedings{gerogiannis2024blackboxfeas,
	title        = {{Is Prior-Free Black-Box Non-Stationary Reinforcement Learning Feasible?}},
	author       = {Gerogiannis, Argyrios and Huang, Yu-Han and Veeravalli, Venugopal},
	year         = 2025,
	month        = {03--05 May},
	booktitle    = {Proceedings of The 28th International Conference on Artificial Intelligence and Statistics},
	publisher    = {PMLR},
	series       = {Proceedings of Machine Learning Research},
	volume       = 258,
	pages        = {2692--2700},
	url          = {https://proceedings.mlr.press/v258/gerogiannis25a.html},
	editor       = {Li, Yingzhen and Mandt, Stephan and Agrawal, Shipra and Khan, Emtiyaz}
}

@misc{huang2025dabprocedures,
      title={Detection Augmented Bandit Procedures for Piecewise Stationary MABs: A Modular Approach}, 
      author={Yu-Han Huang and Argyrios Gerogiannis and Subhonmesh Bose and Venugopal V. Veeravalli},
      year={2025},
      eprint={2501.01291},
      archivePrefix={arXiv},
      primaryClass={cs.AI},
      url={https://arxiv.org/abs/2501.01291}, 
}

@InProceedings{zhou2021lowerboundlin,
  title = 	 {Nearly Minimax Optimal Reinforcement Learning for Linear Mixture Markov Decision Processes},
  author =       {Zhou, Dongruo and Gu, Quanquan and Szepesvari, Csaba},
  booktitle = 	 {Proceedings of Thirty Fourth Conference on Learning Theory},
  pages = 	 {4532--4576},
  year = 	 {2021},
  editor = 	 {Belkin, Mikhail and Kpotufe, Samory},
  volume = 	 {134},
  series = 	 {Proceedings of Machine Learning Research},
  month = 	 {15--19 Aug},
  publisher =    {PMLR},
  url = 	 {https://proceedings.mlr.press/v134/zhou21a.html}
}

@INPROCEEDINGS{huang2025sequentialchangedetectionlearning,
  author={Huang, Yu-Han and Veeravalli, Venugopal V.},
  booktitle={2025 IEEE International Symposium on Information Theory (ISIT)}, 
  title={Sequential Change Detection for Learning in Piecewise Stationary Bandit Environments}, 
  year={2025},
  volume={},
  number={},
  pages={1-5},
  doi={10.1109/ISIT63088.2025.11195515}}

@article{jaksch2010bugdget_restart_tabular_modelbased,
  title={Near-optimal Regret Bounds for Reinforcement Learning},
  author={Jaksch, Thomas and Ortner, Ronald and Auer, Peter},
  journal={Journal of Machine Learning Research},
  volume={11},
  pages={1563--1600},
  year={2010}
}

@article{gajane2018sliding_window_tabular_modelbased,
  title={A sliding-window algorithm for {M}arkov decision processes with arbitrarily changing rewards and transitions},
  author={Gajane, Pratik and Ortner, Ronald and Auer, Peter},
  journal={arXiv preprint arXiv:1805.10066},
  year={2018}
}

@inproceedings{ortner2020variational_budget_restart_tabular_modelbased,
  title={Variational regret bounds for reinforcement learning},
  author={Ortner, Ronald and Gajane, Pratik and Auer, Peter},
  booktitle={Uncertainty in Artificial Intelligence},
  pages={81--90},
  year={2019}
}

@InProceedings{cheung2020reinforce_sliding_window_tabular_modelbased,
  title = 	 {Reinforcement Learning for Non-Stationary {M}arkov Decision Processes: The Blessing of ({M}ore) Optimism},
  author =       {Cheung, Wang Chi and Simchi-Levi, David and Zhu, Ruihao},
  booktitle = 	 {Proceedings of the 37th International Conference on Machine Learning},
  pages = 	 {1843--1854},
  year = 	 {2020},
  editor = 	 {III, Hal Daumé and Singh, Aarti},
  volume = 	 {119},
  series = 	 {Proceedings of Machine Learning Research},
  month = 	 {13--18 Jul},
  publisher =    {PMLR},
  url = 	 {https://proceedings.mlr.press/v119/cheung20a.html}
}

@inproceedings{domingues2020kernel_sliding_window_discounted_modelbased,
  title={A kernel-based approach to non-stationary reinforcement learning in metric spaces},
  author={Domingues, Omar Darwiche and M{\'e}nard, Pierre and Pirotta, Matteo and Kaufmann, Emilie and Valko, Michal},
  booktitle={International Conference on Artificial Intelligence and Statistics},
  pages={3538--3546},
  year={2021}
}

@inproceedings{besbes2014nsmabs,
	title        = {{Stochastic Multi-Armed-Bandit Problem with Non-stationary Rewards}},
	author       = {Besbes, Omar and Gur, Yonatan and Zeevi, Assaf},
	year         = 2014,
	booktitle    = {Advances in Neural Information Processing Systems},
	publisher    = {Curran Associates, Inc.},
	volume       = 27,
	pages        = {},
	url          = {https://proceedings.neurips.cc/paper_files/paper/2014/file/903ce9225fca3e988c2af215d4e544d3-Paper.pdf},
	editor       = {Z. Ghahramani and M. Welling and C. Cortes and N. Lawrence and K.Q. Weinberger}
}

@inproceedings{wei2021master,
	title        = {{Non-stationary Reinforcement Learning without Prior Knowledge: an Optimal Black-box Approach}},
	author       = {Wei, Chen-Yu and Luo, Haipeng},
	year         = 2021,
	month        = {15--19 Aug},
	booktitle    = {Proceedings of Thirty Fourth Conference on Learning Theory},
	publisher    = {PMLR},
	series       = {Proceedings of Machine Learning Research},
	volume       = 134,
	pages        = {4300--4354},
	url          = {https://proceedings.mlr.press/v134/wei21b.html},
	editor       = {Belkin, Mikhail and Kpotufe, Samory}
}

@inproceedings{peng_papadim_2024,
	title        = {{The complexity of non-stationary reinforcement learning}},
	author       = {Peng, Binghui and Papadimitriou, Christos},
	year         = 2024,
	month        = {25--28 Feb},
	booktitle    = {Proceedings of The 35th International Conference on Algorithmic Learning Theory},
	publisher    = {PMLR},
	series       = {Proceedings of Machine Learning Research},
	volume       = 237,
	pages        = {972--996},
	editor       = {Vernade, Claire and Hsu, Daniel}
}

@inproceedings{cao2019nearly,
	title        = {{Nearly Optimal Adaptive Procedure with Change Detection for Piecewise-Stationary Bandit}},
	author       = {Cao, Yang and Wen, Zheng and Kveton, Branislav and Xie, Yao},
	year         = 2019,
	month        = {16--18 Apr},
	booktitle    = {Proceedings of the Twenty-Second International Conference on Artificial Intelligence and Statistics},
	publisher    = {PMLR},
	series       = {Proceedings of Machine Learning Research},
	volume       = 89,
	pages        = {418--427},
	url          = {https://proceedings.mlr.press/v89/cao19a.html},
	editor       = {Chaudhuri, Kamalika and Sugiyama, Masashi}
}

@article{liu2018change,
	title        = {{A Change-Detection Based Framework for Piecewise-Stationary Multi-Armed Bandit Problem}},
	author       = {Liu, Fang and Lee, Joohyun and Shroff, Ness},
	year         = 2018,
	month        = {Apr.},
	journal      = {Proceedings of the AAAI Conference on Artificial Intelligence},
	volume       = 32,
	number       = 1,
	doi          = {10.1609/aaai.v32i1.11746},
	url          = {https://ojs.aaai.org/index.php/AAAI/article/view/11746},
	abstractnote = {&lt;p&gt; The multi-armed bandit problem has been extensively studied under the stationary assumption. However in reality, this assumption often does not hold because the distributions of rewards themselves may change over time. In this paper, we propose a change-detection (CD) based framework for multi-armed bandit problems under the piecewise-stationary setting, and study a class of change-detection based UCB (Upper Confidence Bound) policies, CD-UCB, that actively detects change points and restarts the UCB indices. We then develop CUSUM-UCB and PHT-UCB, that belong to the CD-UCB class and use cumulative sum (CUSUM) and Page-Hinkley Test (PHT) to detect changes. We show that CUSUM-UCB obtains the best known regret upper bound under mild assumptions. We also demonstrate the regret reduction of the CD-UCB policies over arbitrary Bernoulli rewards and Yahoo! datasets of webpage click-through rates. &lt;/p&gt;}
}

@article{besson2022efficient,
	title        = {{Efficient Change-Point Detection for Tackling Piecewise-Stationary Bandits}},
	author       = {Lilian Besson and Emilie Kaufmann and Odalric-Ambrym Maillard and Julien Seznec},
	year         = 2022,
	journal      = {Journal of Machine Learning Research},
	volume       = 23,
	number       = 77,
	pages        = {1--40}
}

@inproceedings{auer2019adswitch,
	title        = {{Adaptively Tracking the Best Bandit Arm with an Unknown Number of Distribution Changes}},
	author       = {Auer, Peter and Gajane, Pratik and Ortner, Ronald},
	year         = 2019,
	month        = {25--28 Jun},
	booktitle    = {Proceedings of the Thirty-Second Conference on Learning Theory},
	publisher    = {PMLR},
	series       = {Proceedings of Machine Learning Research},
	volume       = 99,
	pages        = {138--158},
	editor       = {Beygelzimer, Alina and Hsu, Daniel}
}

@inproceedings{russac2019weighted,
	title        = {{Weighted Linear Bandits for Non-Stationary Environments}},
	author       = {Russac, Yoan and Vernade, Claire and Capp\'{e}, Olivier},
	year         = 2019,
	booktitle    = {Advances in Neural Information Processing Systems},
	publisher    = {Curran Associates, Inc.},
	volume       = 32,
	pages        = {},
	url          = {https://proceedings.neurips.cc/paper_files/paper/2019/file/263fc48aae39f219b4c71d9d4bb4aed2-Paper.pdf},
	editor       = {H. Wallach and H. Larochelle and A. Beygelzimer and F. d\textquotesingle Alch\'{e}-Buc and E. Fox and R. Garnett}
}

@inproceedings{kocsis2006discounted,
	title        = {{Discounted ucb}},
	author       = {Kocsis, Levente and Szepesv{\'a}ri, Csaba},
	year         = 2006,
	booktitle    = {2nd PASCAL Challenges Workshop},
	volume       = 2,
	pages        = {51--134}
}

@inproceedings{garivier2011sw,
	title        = {{On Upper-Confidence Bound Policies for Switching Bandit Problems}},
	author       = {Garivier, Aur{\'e}lien and Moulines, Eric},
	year         = 2011,
	booktitle    = {Algorithmic Learning Theory},
	publisher    = {Springer Berlin Heidelberg},
	address      = {Berlin, Heidelberg},
	pages        = {174--188},
	isbn         = {978-3-642-24412-4},
	editor       = {Kivinen, Jyrki and Szepesv{\'a}ri, Csaba and Ukkonen, Esko and Zeugmann, Thomas}
}

@inproceedings{wang2023weight,
	title        = {{Revisiting Weighted Strategy for Non-stationary Parametric Bandits}},
	author       = {Wang, Jing and Zhao, Peng and Zhou, Zhi-Hua},
	year         = 2023,
	month        = {25--27 Apr},
	booktitle    = {Proceedings of The 26th International Conference on Artificial Intelligence and Statistics},
	publisher    = {PMLR},
	series       = {Proceedings of Machine Learning Research},
	volume       = 206,
	pages        = {7913--7942},
	url          = {https://proceedings.mlr.press/v206/wang23k.html},
	editor       = {Ruiz, Francisco and Dy, Jennifer and van de Meent, Jan-Willem}
}

@misc{russac2020algnsglbs,
	title        = {{Algorithms for Non-Stationary Generalized Linear Bandits}},
	author       = {Yoan Russac and Olivier Cappé and Aurélien Garivier},
	year         = 2020,
	url          = {https://arxiv.org/abs/2003.10113},
	eprint       = {2003.10113},
	archiveprefix = {arXiv},
	primaryclass = {cs.LG}
}

@misc{faury2021regboundsnsglbs,
	title        = {{Regret Bounds for Generalized Linear Bandits under Parameter Drift}},
	author       = {Louis Faury and Yoan Russac and Marc Abeille and Clément Calauzènes},
	year         = 2021,
	url          = {https://arxiv.org/abs/2103.05750},
	eprint       = {2103.05750},
	archiveprefix = {arXiv},
	primaryclass = {cs.LG}
}

@inproceedings{russac2021nsweightscbs,
	title        = {{ Self-Concordant Analysis of Generalized Linear Bandits with Forgetting }},
	author       = {Russac, Yoan and Faury, Louis and Capp{\'e}, Olivier and Garivier, Aur{\'e}lien},
	year         = 2021,
	month        = {13--15 Apr},
	booktitle    = {Proceedings of The 24th International Conference on Artificial Intelligence and Statistics},
	publisher    = {PMLR},
	series       = {Proceedings of Machine Learning Research},
	volume       = 130,
	pages        = {658--666},
	url          = {https://proceedings.mlr.press/v130/russac21a.html},
	editor       = {Banerjee, Arindam and Fukumizu, Kenji}
}

@inproceedings{cheung2019swlbs,
	title        = {{Learning to Optimize under Non-Stationarity}},
	author       = {Cheung, Wang Chi and Simchi-Levi, David and Zhu, Ruihao},
	year         = 2019,
	month        = {16--18 Apr},
	booktitle    = {Proceedings of the Twenty-Second International Conference on Artificial Intelligence and Statistics},
	publisher    = {PMLR},
	series       = {Proceedings of Machine Learning Research},
	volume       = 89,
	pages        = {1079--1087},
	url          = {https://proceedings.mlr.press/v89/cheung19b.html},
	editor       = {Chaudhuri, Kamalika and Sugiyama, Masashi}
}

@inproceedings{zhao2020rnsglbs,
	title        = {{A Simple Approach for Non-stationary Linear Bandits}},
	author       = {Zhao, Peng and Zhang, Lijun and Jiang, Yuan and Zhou, Zhi-Hua},
	year         = 2020,
	month        = {26--28 Aug},
	booktitle    = {Proceedings of the Twenty Third International Conference on Artificial Intelligence and Statistics},
	publisher    = {PMLR},
	series       = {Proceedings of Machine Learning Research},
	volume       = 108,
	pages        = {746--755},
	url          = {https://proceedings.mlr.press/v108/zhao20a.html},
	editor       = {Chiappa, Silvia and Calandra, Roberto}
}

@inproceedings{hong2023opkb,
	title        = {{An Optimization-based Algorithm for Non-stationary Kernel Bandits without Prior Knowledge}},
	author       = {Hong, Kihyuk and Li, Yuhang and Tewari, Ambuj},
	year         = 2023,
	month        = {25--27 Apr},
	booktitle    = {Proceedings of The 26th International Conference on Artificial Intelligence and Statistics},
	publisher    = {PMLR},
	series       = {Proceedings of Machine Learning Research},
	volume       = 206,
	pages        = {3048--3085},
	url          = {https://proceedings.mlr.press/v206/hong23b.html},
	editor       = {Ruiz, Francisco and Dy, Jennifer and van de Meent, Jan-Willem}
}

@misc{huang2025cdbppsmabs,
	title        = {{Change Detection-Based Procedures for Piecewise Stationary MABs: A Modular Approach}},
	author       = {Yu-Han Huang and Argyrios Gerogiannis and Subhonmesh Bose and Venugopal V. Veeravalli},
	year         = 2025,
	url          = {https://arxiv.org/abs/2501.01291},
	eprint       = {2501.01291},
	archiveprefix = {arXiv},
	primaryclass = {cs.AI}
}

@inproceedings{zhou2021rgpucb-swgpucb,
	title        = {{No-Regret Algorithms for Time-Varying Bayesian Optimization}},
	author       = {Zhou, Xingyu and Shroff, Ness},
	year         = 2021,
	booktitle    = {2021 55th Annual Conference on Information Sciences and Systems (CISS)},
	volume       = {},
	number       = {},
	pages        = {1--6},
	doi          = {10.1109/CISS50987.2021.9400292}
}

@inproceedings{deng2022wgpucb,
	title        = {{ Weighted Gaussian Process Bandits for Non-stationary Environments }},
	author       = {Deng, Yuntian and Zhou, Xingyu and Kim, Baekjin and Tewari, Ambuj and Gupta, Abhishek and Shroff, Ness},
	year         = 2022,
	month        = {28--30 Mar},
	booktitle    = {Proceedings of The 25th International Conference on Artificial Intelligence and Statistics},
	publisher    = {PMLR},
	series       = {Proceedings of Machine Learning Research},
	volume       = 151,
	pages        = {6909--6932},
	url          = {https://proceedings.mlr.press/v151/deng22b.html},
	editor       = {Camps-Valls, Gustau and Ruiz, Francisco J. R. and Valera, Isabel}
}

@InProceedings{chen2019adailtcbp,
  title = 	 {A New Algorithm for Non-stationary Contextual Bandits: Efficient, Optimal and Parameter-free},
  author =       {Chen, Yifang and Lee, Chung-Wei and Luo, Haipeng and Wei, Chen-Yu},
  booktitle = 	 {Proceedings of the Thirty-Second Conference on Learning Theory},
  pages = 	 {696--726},
  year = 	 {2019},
  editor = 	 {Beygelzimer, Alina and Hsu, Daniel},
  volume = 	 {99},
  series = 	 {Proceedings of Machine Learning Research},
  month = 	 {25--28 Jun},
  publisher =    {PMLR},
  url = 	 {https://proceedings.mlr.press/v99/chen19b.html}
}

@InProceedings{luo2018adailtcb,
  title = 	 {Efficient Contextual Bandits in Non-stationary Worlds},
  author =       {Luo, Haipeng and Wei, Chen-Yu and Agarwal, Alekh and Langford, John},
  booktitle = 	 {Proceedings of the 31st  Conference On Learning Theory},
  pages = 	 {1739--1776},
  year = 	 {2018},
  editor = 	 {Bubeck, Sébastien and Perchet, Vianney and Rigollet, Philippe},
  volume = 	 {75},
  series = 	 {Proceedings of Machine Learning Research},
  month = 	 {06--09 Jul},
  publisher =    {PMLR},
  url = 	 {https://proceedings.mlr.press/v75/luo18a.html}
}

@inproceedings{cheng2023lowrankreach,
 author = {Cheng, Yuan and Yang, Jing and Liang, Yingbin},
 booktitle = {Advances in Neural Information Processing Systems},
 editor = {A. Oh and T. Naumann and A. Globerson and K. Saenko and M. Hardt and S. Levine},
 pages = {6330--6372},
 publisher = {Curran Associates, Inc.},
 title = {Provably Efficient Algorithm for Nonstationary Low-Rank MDPs},
 volume = {36},
 year = {2023}
}

@misc{nguyen2025nonstationarylipschitzbandits,
      title={Non-Stationary Lipschitz Bandits}, 
      author={Nicolas Nguyen and Solenne Gaucher and Claire Vernade},
      year={2025},
      eprint={2505.18871},
      archivePrefix={arXiv},
      primaryClass={stat.ML},
      url={https://arxiv.org/abs/2505.18871}, 
}

@inproceedings{
osband2020deepsea,
title={Behaviour Suite for Reinforcement Learning},
author={Ian Osband and Yotam Doron and Matteo Hessel and John Aslanides and Eren Sezener and Andre Saraiva and Katrina McKinney and Tor Lattimore and Csaba Szepesvari and Satinder Singh and Benjamin Van Roy and Richard Sutton and David Silver and Hado Van Hasselt},
booktitle={International Conference on Learning Representations},
year={2020},
url={https://openreview.net/forum?id=rygf-kSYwH}
}

@article{sutton1999fourroom,
title = {Between MDPs and semi-MDPs: A framework for temporal abstraction in reinforcement learning},
journal = {Artificial Intelligence},
volume = {112},
number = {1},
pages = {181-211},
year = {1999},
issn = {0004-3702},
doi = {https://doi.org/10.1016/S0004-3702(99)00052-1},
url = {https://www.sciencedirect.com/science/article/pii/S0004370299000521},
author = {Richard S. Sutton and Doina Precup and Satinder Singh}
}

@misc{rlberry2021nroom,
    author = {Domingues, Omar Darwiche and Flet-Berliac, Yannis and Leurent, Edouard and M{\'e}nard, Pierre and Shang, Xuedong and Valko, Michal},
    doi = {10.5281/zenodo.5544540},
    month = {10},
    title = {{rlberry - A Reinforcement Learning Library for Research and Education}},
    url = {https://github.com/rlberry-py/rlberry},
    year = {2021}
}

@inproceedings{
russo2023forkedriverswim,
title={Model-Free Active Exploration in Reinforcement Learning},
author={Alessio Russo and Alexandre Proutiere},
booktitle={Thirty-seventh Conference on Neural Information Processing Systems},
year={2023},
url={https://openreview.net/forum?id=YEtstXIpP3}
}

@article{archibald1995garnet,
 ISSN = {01605682, 14769360},
 URL = {http://www.jstor.org/stable/2584329},
 author = {T. W. Archibald and K. I. M. McKinnon and L. C. Thomas},
 journal = {The Journal of the Operational Research Society},
 number = {3},
 pages = {354--361},
 publisher = {Palgrave Macmillan Journals},
 title = {On the Generation of Markov Decision Processes},
 urldate = {2026-05-06},
 volume = {46},
 year = {1995}
}

@article{bhatnagar2009garnet,
title = {Natural actor–critic algorithms},
journal = {Automatica},
volume = {45},
number = {11},
pages = {2471-2482},
year = {2009},
issn = {0005-1098},
doi = {https://doi.org/10.1016/j.automatica.2009.07.008},
url = {https://www.sciencedirect.com/science/article/pii/S0005109809003549},
author = {Shalabh Bhatnagar and Richard S. Sutton and Mohammad Ghavamzadeh and Mark Lee}
}

@InProceedings{yang2019anchor,
  title = 	 {Sample-Optimal Parametric Q-Learning Using Linearly Additive Features},
  author =       {Yang, Lin and Wang, Mengdi},
  booktitle = 	 {Proceedings of the 36th International Conference on Machine Learning},
  pages = 	 {6995--7004},
  year = 	 {2019},
  editor = 	 {Chaudhuri, Kamalika and Salakhutdinov, Ruslan},
  volume = 	 {97},
  series = 	 {Proceedings of Machine Learning Research},
  month = 	 {09--15 Jun},
  publisher =    {PMLR},
  url = 	 {https://proceedings.mlr.press/v97/yang19b.html}
}

@inproceedings{agarwal2020lowrank,
 author = {Agarwal, Alekh and Kakade, Sham and Krishnamurthy, Akshay and Sun, Wen},
 booktitle = {Advances in Neural Information Processing Systems},
 editor = {H. Larochelle and M. Ranzato and R. Hadsell and M.F. Balcan and H. Lin},
 pages = {20095--20107},
 publisher = {Curran Associates, Inc.},
 title = {FLAMBE: Structural Complexity and Representation Learning of Low Rank MDPs},
 url = {https://proceedings.neurips.cc/paper_files/paper/2020/file/e894d787e2fd6c133af47140aa156f00-Paper.pdf},
 volume = {33},
 year = {2020}
}

%%%%%%%%%%%%%%%%%%%%%%%%%%%%%%%%%%%%%%%%%%%%%%%%%%%%%%%%%%%%%%%%%%%%%%%%%%%%%%%
%%%%%%%%%%%%%%%%%%%%%%%%%%%%%%%%%%%%%%%%%%%%%%%%%%%%%%%%%%%%%%%%%%%%%%%%%%%%%%%
% APPENDIX
%%%%%%%%%%%%%%%%%%%%%%%%%%%%%%%%%%%%%%%%%%%%%%%%%%%%%%%%%%%%%%%%%%%%%%%%%%%%%%%
%%%%%%%%%%%%%%%%%%%%%%%%%%%%%%%%%%%%%%%%%%%%%%%%%%%%%%%%%%%%%%%%%%%%%%%%%%%%%%%
\newpage
\appendix
\onecolumn

\section{Related Work in Non-Stationary Bandits}
\label{app:bandit_work}
Non-stationary (NS) bandits are a canonical testbed for studying learning under distribution shift, and they have strongly influenced how NS-RL algorithms are designed and analyzed.
A useful way to organize the NS bandit literature is along two largely orthogonal axes:
(i) the \emph{adaptation mechanism}---\emph{adaptive} methods that continuously emphasize recency versus \emph{restarting} methods that explicitly reset the learner---and
(ii) the extent of \emph{prior knowledge} required about the non-stationarity---\emph{prior-based} (requiring tuned parameters linked to variation/breakpoints) versus \emph{prior-free} (not requiring such tuning).
This taxonomy parallels the dominant paradigms in NS-RL (discounting/windowing, budget-restart, detection-restart), and helps clarify which assumptions are needed to obtain guarantees and practical performance \citep{garivier2011sw,besbes2014nsmabs,besson2022efficient}.

\textbf{Adaptive methods: discounting and sliding windows.} Adaptive approaches track change by continuously down-weighting or discarding older samples, typically via exponential discounting or fixed-length sliding windows.
These methods are conceptually simple and widely applicable, but their performance depends on selecting a discount factor or window length that matches the (unknown) timescale of non-stationarity, rendering them typically \emph{prior-based}.
In NS-MABs, classical examples include discounted UCB and sliding-window UCB \citep{kocsis2006discounted,garivier2011sw}.
This paradigm has been extended to structured bandits, including NS linear bandits (NS-LBs) \citep{cheung2019swlbs,russac2019weighted,wang2023weight}, NS generalized linear bandits (NS-GLBs) \citep{faury2021regboundsnsglbs,russac2020algnsglbs,wang2023weight}, and NS self-concordant bandits (NS-SCBs) \citep{russac2021nsweightscbs,wang2023weight}.
Analogous ideas also appear in non-parametric settings such as kernelized bandits (NS-KBs), where recency-weighted or windowed estimators are combined with optimism \citep{deng2022wgpucb,zhou2021rgpucb-swgpucb}.
Overall, discounting/windowing provides a general-purpose route to adaptivity, but introduces a non-trivial tuning problem: too much forgetting increases variance, while too little forgetting yields bias under shift.

\textbf{Restarting methods: budgeted restarts.} A second family of approaches explicitly \emph{restarts} the learning process, typically on a schedule designed to control the amount of stale data.
In NS bandits, the most common restarting template is the \emph{budget-restart} strategy, which restarts at predetermined times (or on epochs of increasing lengths) selected using a variation/breakpoint budget.
This yields strong theoretical guarantees when the budget is known or can be tuned, but again is usually \emph{prior-based}.
Representative results include the classical NS-MAB framework in \cite{besbes2014nsmabs}, as well as extensions to structured settings such as NS-LBs/NS-GLBs \citep{zhao2020rnsglbs} and NS-KBs \citep{zhou2021rgpucb-swgpucb}.
Conceptually, budget-restart trades off two error sources: within-epoch learning (stationary regret) and cross-epoch mismatch (stale data), and the schedule is tuned to balance these terms.

\textbf{Restarting methods: detection-based restarts.} Detection-restart methods aim to remove the explicit dependence on a known non-stationarity budget by \emph{testing} for change and restarting only when evidence accumulates.
This is particularly natural in abrupt (piecewise-stationary) models, where changes are sparse but impactful.
In NS-MABs, prior-based detection-restart methods include algorithms that rely on thresholds calibrated to the change budget or minimal gap assumptions \citep{liu2018change,cao2019nearly}.
More recent prior-free approaches emphasize modular change detection primitives (e.g., GLR/CuSum-type tests) coupled with bandit exploration policies, enabling guarantees without knowing the number/timing of changes \citep{auer2019adswitch,besson2022efficient,huang2025cdbppsmabs}.
Beyond MABs, related detection-restart ideas have been developed for richer structured classes, including NS linear and kernelized bandits \citep{hong2023opkb}, NS Lipschitz bandits \cite{nguyen2025nonstationarylipschitzbandits} and NS contextual bandits \citep{luo2018adailtcb,chen2019adailtcbp}.
At a high level, these methods separate concerns: a base algorithm drives exploration/exploitation within a segment, while a statistical test monitors for distributional shifts and triggers a reset. Importantly, in the detection-based restart literature there exist two black-box, prior-free methodologies which are applicable to all the general bandit settings DAL \cite{gerogiannis2025dal} and MASTER \cite{wei2021master}.

\section{DARLING for Linear MDPs}
\label{app:linear_mdps}

Similar to DARLING for tabular MDPs, we need to design a change detection mechanism that augments the stationary algorithm $\mathcal{L}$. However, due to the infinite state space, it is impossible to probe all state-action-step triple over $T$ episodes. Therefore, it is essential to identify a small and finite probe set that allows DARLING to reliably detect changes. In addition, maintaining a transition history $\mathcal{H}^{(P)}_{(s,a,h,s')}$ for all $s' \in \mathcal{S}$ and $(s,a,h)$ in the probe set is infeasible due to the infinite state space. Consequently, we also need a new finite set of transition histories.

\textbf{Probe set construction: calibration.} To detect changes reliably within short delay, it is desirable to construct a finite probe set consisting of frequently visited state-action-step triples when DARLING uniformly samples all actions. Assumption \ref{assum:reachability_lin_mdps} guarantees the existence of $\mathcal{P}_{h}$, in which the visitation probability $q_{h,t}(s,a)$ is lower bounded by $1/(2d)$. Thus, we can use $\mathcal{P}_{h}$ as the probe set at step $h$. Since the state-action pairs in $\mathcal{P}_{h}$ have high visitation probabilities, we can identify $\mathcal{P}_{h}$ by choosing the $d$ most frequently occurring state-action pairs at step $h$. Therefore, after each restart, DARLING first employs the uniform sampling policy $\pi_{\mathrm{U}}$ for $n_{0}$ episodes, and then choose the $d$ most visited state-action pair $(s,a)$ at step $h$ as $\mathcal{P}_{h}$. This process is termed as calibration and is illustrated in Algorithm \ref{alg:DARLING_lin_mdps}.

\textbf{Transition histories.} To construct a finite number of histories, we leverage the linear underlying structure of the transition kernel. We first note that the expected value of the feature vector $\phi(s_{h+1}^{t}, a')$ conditioned on $(s^{t}_{h}, a^{t}_{h})$ changes if and only if the probability transition kernel $P_{h}^{t}(s_{h+1}^{t}|s^{t}_{h}, a^{t}_{h}) = \phi(s^{t}_{h}, a^{t}_{h})\mu_{h,t}(s^{t}_{h+1})$.
%
% \begin{proposition}
%     For a fixed $(s, a, h, a')$
% \end{proposition}
\begin{proposition}
\label{prop:sf_mean_kernel_change}
Fix $(s,a,h,a')\in\mathcal S\times\mathcal A\times[H]\times\mathcal A$ and two episodes
$t\neq t'$. Define
\[
\Delta P_h(\cdot\mid s,a)
\coloneqq
P_h^t(\cdot\mid s,a)-P_h^{t'}(\cdot\mid s,a).
\]
Then
\[
\mathbb E_{s'\sim P_h^t(\cdot\mid s,a)}[\phi(s',a')]
-
\mathbb E_{s'\sim P_h^{t'}(\cdot\mid s,a)}[\phi(s',a')]
=
\sum_{s'\in\mathcal S}
\Delta P_h(s'\mid s,a)\phi(s',a').
\]
Consequently,
\[
\mathbb E_{s'\sim P_h^t(\cdot\mid s,a)}[\phi(s',a')]
\neq
\mathbb E_{s'\sim P_h^{t'}(\cdot\mid s,a)}[\phi(s',a')]
\quad\Longrightarrow\quad
P_h^t(\cdot\mid s,a)\neq P_h^{t'}(\cdot\mid s,a).
\]
Moreover, if the map
\[
p\mapsto \sum_{s'\in\mathcal S}p(s')\phi(s',a')
\]
is injective over probability distributions on $\mathcal S$, then the converse also holds. Hence, under this identifiability condition,
\[
\mathbb E_{s'\sim P_h^t(\cdot\mid s,a)}[\phi(s',a')]
\neq
\mathbb E_{s'\sim P_h^{t'}(\cdot\mid s,a)}[\phi(s',a')]
\quad\Longleftrightarrow\quad
P_h^t(\cdot\mid s,a)\neq P_h^{t'}(\cdot\mid s,a).
\]
\end{proposition}

\begin{proof}
By definition,
\[
\mathbb E_{s'\sim P_h^t(\cdot\mid s,a)}[\phi(s',a')]
=
\sum_{s'\in\mathcal S}
P_h^t(s'\mid s,a)\phi(s',a'),
\]
and similarly,
\[
\mathbb E_{s'\sim P_h^{t'}(\cdot\mid s,a)}[\phi(s',a')]
=
\sum_{s'\in\mathcal S}
P_h^{t'}(s'\mid s,a)\phi(s',a').
\]
Subtracting gives
\[
\mathbb E_{P_h^t}[\phi(s',a')\mid s,a]
-
\mathbb E_{P_h^{t'}}[\phi(s',a')\mid s,a]
=
\sum_{s'\in\mathcal S}
\bigl(P_h^t(s'\mid s,a)-P_h^{t'}(s'\mid s,a)\bigr)\phi(s',a').
\]
Thus,
\[
\mathbb E_{P_h^t}[\phi(s',a')\mid s,a]
-
\mathbb E_{P_h^{t'}}[\phi(s',a')\mid s,a]
=
\sum_{s'\in\mathcal S}
\Delta P_h(s'\mid s,a)\phi(s',a').
\]
If the left-hand side is nonzero, then the signed measure
$\Delta P_h(\cdot\mid s,a)$ cannot be identically zero. Hence
\[
P_h^t(\cdot\mid s,a)\neq P_h^{t'}(\cdot\mid s,a).
\]

Conversely, suppose the map
\[
p\mapsto \sum_{s'\in\mathcal S}p(s')\phi(s',a')
\]
is injective over probability distributions on $\mathcal S$. If
\[
P_h^t(\cdot\mid s,a)\neq P_h^{t'}(\cdot\mid s,a),
\]
then injectivity implies
\[
\sum_{s'\in\mathcal S}
P_h^t(s'\mid s,a)\phi(s',a')
\neq
\sum_{s'\in\mathcal S}
P_h^{t'}(s'\mid s,a)\phi(s',a').
\]
Therefore,
\[
\mathbb E_{s'\sim P_h^t(\cdot\mid s,a)}[\phi(s',a')]
\neq
\mathbb E_{s'\sim P_h^{t'}(\cdot\mid s,a)}[\phi(s',a')],
\]
which proves the equivalence under the stated identifiability condition.
\end{proof}
Thus, DARLING tracks the shift in the transition kernel by monitoring $\phi(s_{h+1}^{t}, a')$ for all $a' \in \mathcal{A}$. Note that it is possible to track all $\phi(s_{h+1}^{t}, a')$ since $\mathcal{A}$ is finite. 

\begin{algorithm*}
\small
\caption{\small \textbf{D}etection \textbf{A}ugmented \textbf{R}einforcement \textbf{L}earn\textbf{ING} (for Linear MDPs)}
\label{alg:DARLING_lin_mdps}
\textbf{Input:} stationary algorithm $\mathcal{L}$, detector $\mathcal{D}$, calibration period $n_{0}$,
probing frequencies $\{\alpha_k\}_{k\ge 1}$. \\
\textbf{Initialization:}  calibration endpoint $\tau\leftarrow n_{0}$, calibration counter $\hat{n}_{s,a,h} \leftarrow 0$ for all $(s,a,h) \in \mathcal{S} \times \mathcal{A} \times [H]$, counter $k \leftarrow 1$, probe set $\mathcal{P}_{h} \leftarrow \emptyset$ for all $h \in [H]$, reward history and transition history $\mathcal{H}^{(r)}_{(s,a,h)}, \mathcal{H}^{(P)}_{(s,a,h,j,a')}\leftarrow\emptyset$ for all $(s,a,h,j,a')\in\mathcal{S} \times \mathcal{A} \times [H] \times [d]\times\mathcal{A}$.

\begin{algorithmic}[1]
\FOR{$t=1,2,\dots,T$}

        %\STATE Set starting state $s_1^t\leftarrow s_1$
        \FOR{$h=1,2,\dots,H$}
            \IF{$t \leq \tau$ (calibration)}
            \STATE Set $s \leftarrow s^{t}_{h}$, sample action $a \in \mathcal{A}$ uniformly at random, and $\hat{n}_{(s,a,h)} \leftarrow \hat{n}_{(s,a,h)} + 1$.
            \IF{$t = \tau$}
            \STATE Choose the $d$ most visited state-action pair $(s,a)$ to append into $\mathcal{P}_{h}$ \hfill{$\triangleright$ probe set construction}
            \ENDIF
            \ELSIF{ $(t-\tau - 1) \bmod \lceil 1/\alpha_{k} \rceil = 0$ \textbf{and} $(s_h^t, a) \in \mathcal{P}_{h}$ for some $a \in \mathcal{A}$}
            \STATE Set $s\leftarrow s^t_h$ and select an action $a$ such that $(s_h^t, a) \in \mathcal{P}_{h}$ uniformly at random \hfill{$\triangleright$ forced probing}
                    \STATE Receive reward $R_h^t(s,a)$ and append to history $\mathcal{H}_{(s,a,h)}^{(r)}$
                    \STATE Add $[\phi(s_{h+1}^t,a')]_{j}$ to history $\mathcal{H}_{(s,a,h,j,a')}^{(P)}$ for all $(j,a')\in [d]\times \mathcal{A}$
                    \STATE \texttt{Test 1} $\leftarrow\mathcal{D}(\mathcal{H}^{(r)}_{(s,a,h)})$, \texttt{Test 2} $\leftarrow\! \mathcal{D}(\mathcal{H}^{(P)}_{(s,a,h,j,a')})$  $\forall (j,a')\in [d]\!\times\!\mathcal{A}$ \hfill{$\triangleright$ change detection}
            \ENDIF
            \STATE\textbf{else if} $(t-\tau - 1) \bmod \lceil 1/\alpha_{k} \rceil = 0$ \textbf{then} Select action according to $\mathcal{L}$, but \underline{don't update} $\mathcal{L}$
            \STATE\textbf{else} Run and \underline{update} $\mathcal{L}$ \hfill{$\triangleright$ stationary learning}

        \ENDFOR
        \IF{\textbf{Test 1} or \textbf{Test 2} signals {\texttt{Restart}}}
        \STATE Reset the RL algorithm $\mathcal{L}$; empty all histories $\mathcal{H}$ used for detection \hfill{$\triangleright$ restart learning process}
        \STATE $\tau\leftarrow t + n_{0}, \quad k \leftarrow k + 1$, \textbf{Restart} $\leftarrow$ \texttt{False}
        \ENDIF
\ENDFOR
\end{algorithmic}
\end{algorithm*}

\section{Theoretical Proofs}

\subsection{Errors in the Regret Analysis of MASTER}
\label{sec:error_master}

To ensure readability of this section, we will use the notations in \cite{wei2021master} rather than the ones we introduced in Section \ref{sec:problem_form}.
To observe the error in the regret analysis of MASTER, we focus on Lemma 17 in \cite{wei2021master}. Recall that $t_{n}$ and $E_{n}$ are the stopping times at which a block of index $n$ starts and ends, respectively. More specifically, $E_{n}$ is either $t_{n} + 2^{n} - 1$ or the time at which either Test 1 or Test 2 in Algorithm 3 in \cite{wei2021master} gets triggered. The interval $\{ t_{n}, \dots, t_{n}+2^{n} - 1\}$ is then divided into multiple nearly stationary intervals over which the nonstationarity is upper bounded. More specifically, let
\begin{equation}
    t_{n} = s_{1} \leq e_{1} = s_{2} - 1 < s_{2} \leq e_{2} = s_{3} - 1 < \dots <  s_{K} \leq e_{K} = t_{n} + 2^{n} - 1
\end{equation}
such that for any $i \in [K]$,
\begin{equation}
    \Delta_{\{ s_{i}, \dots, e_{i}\}} \coloneqq \sum_{t=s_{i}}^{e_{i} - 1} \max_{\pi \in \Pi} \lvert f_{t}(\pi) - f_{t+1}(\pi) \rvert \leq \rho(e_{i} - s_{i} + 1)
\end{equation}
where $f_{t}(\pi) = V^{t,\pi}_{1}(s^{t}_{1})$ in the tabular and linear MDPs, and $\rho$ is a function satisfying the property in Assumption 1 in \cite{wei2021master}. In the PS setting, the intervals $\mathcal{I}_{i} \coloneqq \{s_{i}, \dots, e_{i}\}$ are the stationary intervals within which there are no changes. Let $e_{i}' \coloneqq \min \{e_{i}, E_{n}\}$ and  $\mathcal{I}_{i}' \coloneqq \{s_{i}, \dots, e_{i}'\}$ for $i \in [K]$. The time indices and intervals introduced above are illustrated in the following graph.
\begin{figure}[h]
\centering
\begin{tikzpicture}[scale=0.5]
    \draw (-15,1)-- (12,1); 
    \draw (-15,1.5)-- (-15,0.5); 
    \node at (-14.6,2) {$t_{n}$}; 
    \node at (-14.6,0) {$s_{1}$};
    
    \draw (-11.5,1.5)-- (-11.5,0.5); 
    \node at (-11.8,2) {$e_{1}'$}; 
    \node at (-11.8,0) {$e_{1}$};
    \draw [decorate,decoration={brace,amplitude=5pt,mirror,raise=4ex}](-15,2.2) -- (-11.5,2.2) node[midway,yshift=-3em]{$\mathcal{I}_{1}$};
    \draw [decorate,decoration={brace,amplitude=5pt,raise=4ex}](-15,-0.2) -- (-11.5,-0.2) node[midway,yshift=3em]{$\mathcal{I}_{1}'$};
    \node at (-11.1,0) {$s_{2}$};
    \node at (-11.1,1.9) {$\nu_{k}$};
    
    \draw (-7.5,1.5)-- (-7.5,0.5); 
    \node at (-7.8,2) {$e_{2}'$}; 
    \node at (-7.8,0) {$e_{2}$};
    \draw [decorate,decoration={brace,amplitude=5pt,mirror,raise=4ex}](-11.5,2.2) -- (-7.5,2.2) node[midway,yshift=-3em]{$\mathcal{I}_{2}$};
    \draw [decorate,decoration={brace,amplitude=5pt,raise=4ex}](-11.5,-0.2) -- (-7.5,-0.2) node[midway,yshift=3em]{$\mathcal{I}_{2}'$};
    \node at (-7.1,0) {$s_{3}$};
    \node at (-6.7,1.9) {$\nu_{k+1}$};

    \node at (-3.5,0) {$\cdots$};
    \node at (-3.5,1.9) {$\cdots$};

    \draw (0,1.5)-- (0,0.5); 
    \draw (4,1.5)-- (4,0.5);
    \draw (5,0.5)-- (5,1);  
    \node at (0.3,1.9) {$\nu_{\ell}$}; 
    \node at (0.3,0) {$s_{i}$};
    \node at (3.6,2) {$e_{i}'$}; 
    \node at (3.6,0) {$E_{n}$};
    \node at (4.6,0) {$e_{i}$};
    \draw [decorate,decoration={brace,amplitude=5pt,mirror,raise=4ex}](0,2.2) -- (5,2.2) node[midway,yshift=-3em]{$\mathcal{I}_{i}$};
    \draw [decorate,decoration={brace,amplitude=5pt,raise=4ex}](0,-0.2) -- (4,-0.2) node[midway,yshift=3em]{$\mathcal{I}_{i}'$};
    \node at (-7.1,0) {$s_{3}$};
    \node at (-6.7,1.9) {$\nu_{k+1}$};

    \node at (8,0) {$\cdots$};
   \node[align=center] at (6,2) {$\mathcal{I}_{j}' = \emptyset$\\for $j > i$};
    
    \draw (12,1.5)-- (12,0.5); 
    \node at (11,1.9) {$t_{n} + 2^{n} - 1$};
    \node at (11.6,0.1) {$e_{K}$};

    \end{tikzpicture}
    \caption{Illustration of a block of index $n$ and its stationary intervals}
    \label{Ineq}
\end{figure}

Let $f^{*}_{t} \coloneqq \max_{\pi \in \Pi} f_{t}(\pi)$ be the optimal value function and $\tilde{g}_{t}$ be the output of the stationary algorithm instance $\texttt{alg}$ that is currently active. Also, let $\hat{\rho}(t) = 6(\log_{2} T + 1)\log(T/\delta)\rho(t)$ (see Lemma 3 in \cite{wei2021master}). Define the following stopping time for each $m \in [n]\cup\{0\}$ and $i \in [K]$:
\begin{equation}
    \tau_{i}(m) \coloneqq \inf \left\{ t \in \mathcal{I}_{i}': f^{*}_{t} - \tilde{g}_{t} \geq 12\hat{\rho}(2^{m})\right\}.
\end{equation}

Now, fix an arbitrary $m\in [n]\cup\{0\}$ in Lemma 17 in \cite{wei2021master}. We refer to an order-$m$ instance as a stationary algorithm of length $2^{m}$ scheduled by MALG (Algorithm 2 in \cite{wei2021master}). Let \texttt{alg.s} and \texttt{alg.e} denote the start and the end of a stationary algorithm instance \texttt{alg}, respectively. We now recall the definition of the following events:
\begin{align}
W_{t} &:= \left\{\,\tau_{i}(m) \leq t \leq e_{i} - 2\cdot 2^{m} \text{ with } i \text{ s.t. } t\in \mathcal{I}_{i}\,\right\},\\
X_{t} &:= \{\,t\le E_{n}-2\cdot 2^m \,\},\\
Y_{t} &:= \{\,t\le E_{n} \text{\:and\:} (t-t_{n}) \bmod 2^m=0 \,\},\\
Z_{t} &:= \{\,\exists\ \text{order-$m$} \texttt{\:alg\:} \text{s.t. } \texttt{alg.s}=t\,\},\\
V_{t} &:= \Bigl\{\,\exists s\in\{t_{n},\ldots,t\}\ \text{s.t. } \mathds{1}\{W_{s,m}\cap Y_{s,m}\cap Z_{s,m}\}=1\,\Bigr\}.
\end{align}
We would like to emphasize that at the start of the block $t_{n}$, MALG generates a set of $2^{n-m}$ i.i.d. Bernoulli random variables $\{B_{j}: j \in [2^{n-m}] \}$ with parameter (success probability) $\rho(2^{n})/\rho(2^{m})$, and schedules an order-$m$ stationary algorithm instance starting at $t_{n} + (j-1)2^{m}$ if $B_{j} = 1$. 

Now, let us focus on $\textbf{term}_{3}$ on Page 24. The authors said that the event $Z_{t}$ occurs with probability $\rho(2^{n})/\rho(2^{m})$ conditioned on $Y_{t} \cap W_{t}$. Unfortunately, this is not correct since $Y_{t} \subseteq \{t \leq E_{n}\}$. Conditioned on $Y_{t}$ changes the probability of event $Z_{t}$, as $E_{n}$ depends on $\{Z_{t}:t = t_{n} + (j-1)2^{m} \text{\:for\:some\:} j \in [2^{n-m}]\}$. In fact, as long as we condition on any event involving $E_{n}$, the events $\{Z_{t}:t = t_{n} + (j-1)2^{m} \text{\:for\:some\:} j \in [2^{n-m}]\}$ are not independent anymore as well, but the authors treat them as independent events when they are counting the number of trials to the first success. Now, suppose that we can remove $t \leq E_{n}$ from the definition of $Y_{t}$ with some different derivation. Let $\tilde{t}$ be the start of the block which $t$ is at, i.e.,
\begin{equation}
    \tilde{t} = \sup\{\tau \leq t: \tau \text{ is a starting point of a block}\}.
\end{equation}
Note that when $t > E_{n}$, $\tilde{t}$ is not $t_{n}$ anymore. We then expand the event $Z_{t}$
\begin{align}
    Z_{t} &= \{\,\exists\ \text{order-$m$} \texttt{\:alg\:} \text{s.t. } \texttt{alg.s}=t\,\} \nonumber\\
    &= \{ (t - \tilde{t}) \bmod 2^{m} = 0 \} \cap \{B_{j} = 1 \text{ where $j$ corresponds to $t$}\}.
\end{align}
Then, we observe that even when conditioned on $\{(t-t_{n}) \bmod 2^m=0\}$, the probability of $Z_{t}$ occuring is not equal to $\rho(2^{n})/\rho(2^{m})$, as $\{ (t - \tilde{t}) \bmod 2^{m} = 0 \}$ might not occur when $t > E_{n}$. We therefore doubt that there is an easy way to fix the regret analysis of MASTER, and we believe that this error could render the regret upper bound of MASTER invalid unless we can prove it with a completely different approach.

\subsection{Proofs of Theorems}

\subsection{Proof of Theorem \ref{thrm:ps-mdp-lower}}
\label{app:ps-mdp-lower-proof}
\begin{proof}
Assume there are $N_T$ changes and hence $N_T+1$ stationary segments of equal length. Consider the family of
$2^{N_T+1}$ PS tabular episodic MDPs $\{\mathcal{M}_{\mathbf{i}}\}_{\mathbf{i}\in\{0,1\}^{N_T+1}}$
indexed by $\mathbf{i}=(i_1,\ldots,i_{N_T+1})\in\{0,1\}^{N_T+1}$. Each $\mathcal{M}_{\mathbf{i}}$ has
state space $\mathcal{S}$ with $|\mathcal{S}|=S$, action space $\mathcal{A}$ with $|\mathcal{A}|=A$, horizon $H$, and $T$ episodes. Without loss of generality, we set $\mathcal{A} = [A]$. 
Fix an arbitrary policy $\pi$, and let $\mathbb{P}_{\mathbf{i},\pi}$ and
$\mathbb{E}_{\mathbf{i},\pi}$ denote the probability measure and expectation induced by executing $\pi$ in
$\mathcal{M}_{\mathbf{i}}$.

Recall that $\nu_{k}$ denote the $k^{\mathrm{th}}$ change-point and that $\nu_{0} = 1$ and $\nu_{N_{T} + 1} = T+1$. The change-points are evenly separated over the $T$ episodes: for the $k^{\mathrm{th}}$ stationary segment, its interval length 
$\nu_{k} - \nu_{k-1} = \lceil T/(N_T+1)\rceil$ if $k\le T\bmod (N_T+1)$ and $\lfloor T/(N_T+1)\rfloor$ otherwise.
For each $k\in[N_T+1]$, $h\in[H]$, $a\in\mathcal{A}$, and $s\in\mathcal{S}$, let $n_k(s,a,h)$ be the number of
visits to $(s,a)$ at step $h$ during episodes $t\in\{\nu_{k-1},\ldots,\nu_k-1\}$, i.e.,
\begin{align}
n_k(s,a,h):=\sum_{t=\nu_{k-1}}^{\nu_k-1}\mathbf{1}\{s_h^t=s,\ a_h^t=a\}.
\end{align}

We construct a hard instance following the structure of \cite{domingues2021lowerbound}. Assume $S\ge 6$ and $A\ge 2$,
and that there exists an integer $D$ such that
\begin{align}
S-3=\sum_{j=0}^{D-1}A^j=\frac{A^{D}-1}{A-1}.
\end{align}
Additionally, we assume $H\ge 3D$.\footnote{When $A=1$, the construction reduces to a contextual bandit instance rather than an episodic
MDP.} The state space contains a waiting state $s_{\mathrm{w}}$, a root state $s_{\mathrm{root}}$, an $A$-ary tree of depth $D-1$
with leaves $\{\mathrm{leaf}_\ell\}_{\ell=1}^{L}$ where $L=A^{D-1}$, a good absorbing state $s_{\mathrm{g}}$ and a bad absorbing state $s_{\mathrm{b}}$. The states are illustrated in Figure \ref{fig:tree}.

\begin{figure}[hbt]
    \centering
    \includegraphics[width=0.8\linewidth]{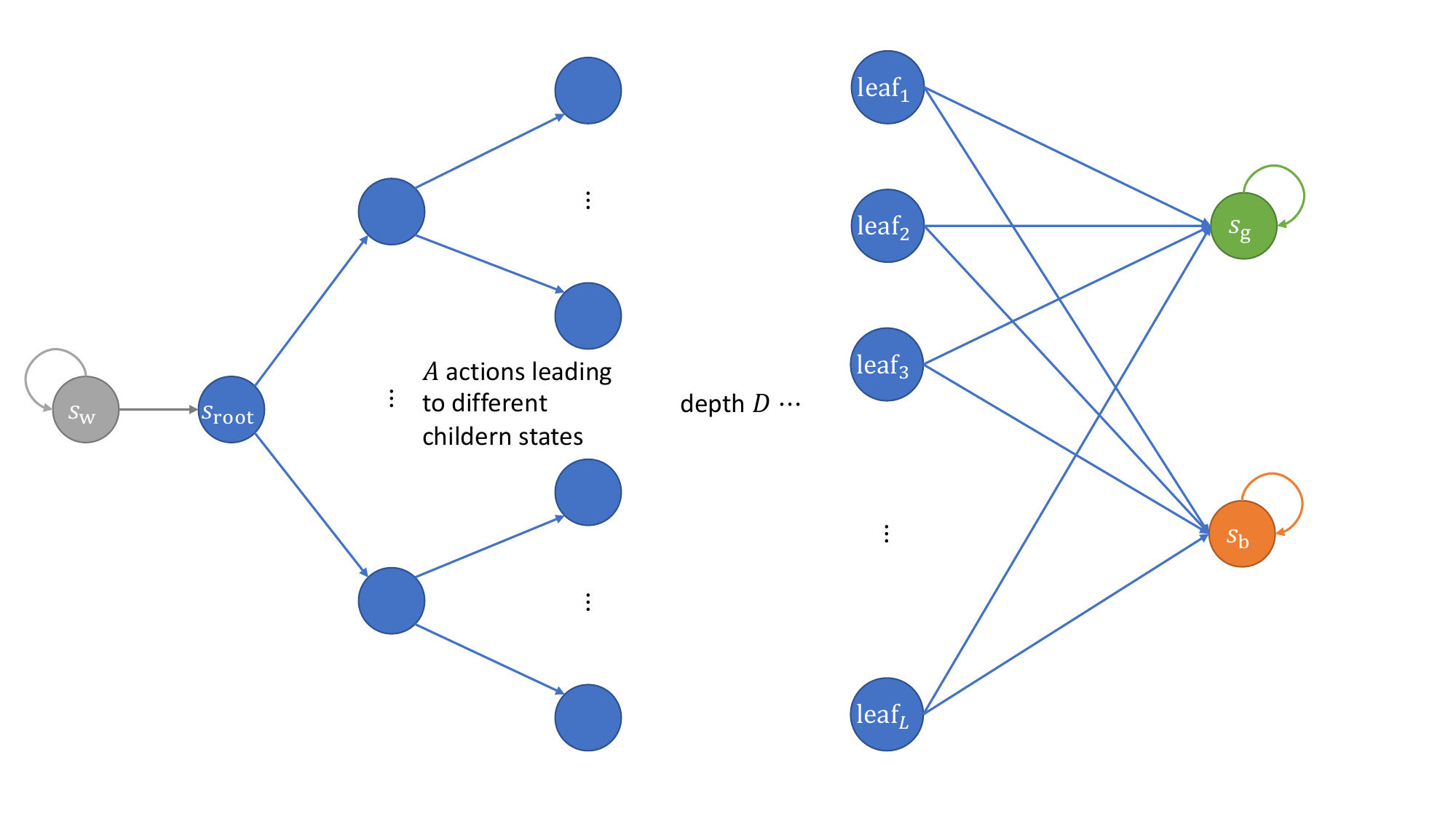}
    \caption{States of the MDP $\mathcal{M}_{\mathbf{i}}$'s}
    \label{fig:tree}
\end{figure}

Let $r_{h,\mathbf{i}}^{k}$ and $P_{h,\mathbf{i}}^{k}$ denote the reward function and the transition kernel of MDP $\mathcal{M}_{\mathbf{i}}$ at step $h$ over the $k^{\mathrm{th}}$ stationary segment, i.e., $r_{h}^{t} = r_{h,\mathbf{i}}^{k}$ and $P_{h}^{t} = P_{h,\mathbf{i}}^{k}$ for $\mathcal{M}_{\mathbf{i}}$ at $t \in \{ \nu_{k-1}, \dots, \nu_{k} - 1\}$. We set the rewards of the MDP $\mathcal{M}_{\mathbf{i}}$'s to be deterministic and binary. To be specific, The reward function is defined as follows: for all $k\in[N_{T} + 1]$, $h\in[H]$, $\mathbf{i} \in \{ 0, 1 \}^{N_{T} + 1}$, $s \in \mathcal{S}$, and $a \in \mathcal{A}$,
\begin{align}
r_{h, \mathbf{i}}^{k} (s,a)=
\begin{cases}
1, & s=s_{\mathrm{g}},\ \ h\ge \bar H+D+1,\\
0, & \text{otherwise},
\end{cases}
\end{align}
for some $\bar H\in [H]$ whose value is determined later in the proof. In other words, if the agent ends up at the good absorbing state $s_{\mathrm{g}}$  Notice that the reward function is invariant across all change-points and all MDPs. 

The transition kernels are defined as follows: Consider an arbitrary MDP $\mathcal{M}_{\mathbf{i}}$. The agent starts at the waiting state $s_{\mathrm{w}}$, i.e., $s_{0}^{t} = s_{\mathrm{w}}$. At $s_{\mathrm{w}}$, for $h<\bar H$, the agent
moves to $s_{\mathrm{root}}$ deterministically when the leaving action $a_{\mathrm{leave}}$ is chosen. When other action is chosen, the agent remains at $s_{\mathrm{w}}$ deterministically, i.e.,
\begin{align}
P_{h, \mathbf{i}}^{k} (s | s_{\mathrm{w}}, a) = 
\begin{dcases}
    1,& (s, a) = (s_{\mathrm{root}}, a_{\mathrm{leave}}),\\
    1,& s = s_{\mathrm{w}} \:\mathrm{and}\: a \neq a_{\mathrm{leave}},\\
    0,& \mathrm{otherwise}.
\end{dcases}
\end{align}
Without loss of generality, we set $a_{\mathrm{leave}}$ = 1.
At the
$\bar{H}$ step, the next state is $s_{\mathrm{root}}$ deterministically regardless of the chosen action, i.e.,
\begin{align}
    P_{H, \mathbf{i}}^{k} (s | s_{\mathrm{w}}, a) = 
\begin{dcases}
    1,& s = s_{\mathrm{root}},\\
    0,& \mathrm{otherwise}.
\end{dcases}
\end{align}
At any internal
tree node, choosing action $a$ deterministically moves to the $a$-th child. 
Now, consider the leaf nodes at the $k^{\mathrm{th}}$ stationary segment of MDP $\mathcal{M}_{\mathbf{i}}$ with $i_k=0$. Let $\varepsilon_k>0$ be a bias parameter that we tune later in the proof. 
Then, if the agent chooses the good action $a_{\mathrm{g}}$ at leaf node $\mathrm{leaf}_{1}$ at step $1+D$, it goes to the good absorbing state $s_{\mathrm{g}}$ with probability $\frac{1}{2} + \varepsilon_{k}$, and goes to the bad absorbing state $s_{\mathrm{b}}$ with probability $\frac{1}{2} - \varepsilon_{k}$, i.e.,
\begin{align}\label{eq:good_leaf}
P^{k}_{1+D, \mathbf{i}}(s\mid \mathrm{leaf}_{1}, a_{\mathrm{g}})=
\begin{cases}
\frac12+\varepsilon_k, & s=s_{\mathrm{g}},\\
\frac12-\varepsilon_k, & s=s_{\mathrm{b}},\\
0, & \mathrm{otherwise}.
\end{cases}
\end{align}
If the agent chooses other actions at leaf node $\mathrm{leaf}_{1}$ at step $1+D$, then the agent goes to the two absorbing states with equal probability.
If the agent is at other leaves or at other step, the agent goes to the two absorbing states with equal probability regardless of the chosen action, i.e., for all $(h, \ell, a) \neq (1+D, 1, a_{\mathrm{g}})$
\begin{align}\label{eq:other_leaf}
P^{k}_{h,\mathbf{i}}(s\mid \mathrm{leaf}_\ell,a)=
\begin{cases}
\frac12, & s\in\{s_{\mathrm{g}},s_{\mathrm{b}}\},\\
0, & \text{otherwise}.
\end{cases}
\end{align}
Without loss of generality, we set $a_{\mathrm{g}} = 1$. 
In this case, the optimal policy is the one that leads the agent to $\mathrm{leaf}_{1}$ at step $1+D$, and then selects the good action $a_{\mathrm{g}}$ to reach the good absorbing state with higher probability.
Now, consider the leaf nodes at the $k^{\mathrm{th}}$ stationary segment of MDP $\mathcal{M}_{\mathbf{i}}$ with $i_k=1$.
Let $\mathbf{j}$ denote the $(N_{T}+1)$-dimensional binary vector obtained by flipping the $k$-th bit of $\mathbf{i}$, i.e., $j_k=0$ and $j_l=i_l$ for $l\neq k$.
Recall that $\mathbb{E}_{\mathbf{i},\pi}$ denote the expectation induced by executing $\pi$ in
$\mathcal{M}_{\mathbf{i}}$.
Define
\begin{align}\label{eq:least_visit}
(\tilde h_{\mathbf{i}},\tilde \ell_{\mathbf{i}},\tilde a_{\mathbf{i}})
=
\arg\min_{(h,\ell,a)\neq(1+D,1,a_{\mathrm{g}})}
\mathbb{E}_{\tilde{\mathbf{i}},\pi}\!\left[n_k(\mathrm{leaf}_\ell,a,h)\right].
\end{align}
In other words, the expected number of times the policy $\pi$ chooses action $\tilde a_{\mathbf{i}}$ at leaf node $\mathrm{leaf}_{\tilde \ell_{\mathbf{i}}}$ at step $\tilde{h}_{\mathbf{i}}$ is the least compared to those when choosing action $a$ at leaf node $\mathrm{leaf}_{\ell}$ at step $h$ such that $(h,\ell,a)\neq(1+D,1,a_{\mathrm{g}})$.  Then, when the agent selects $\tilde a_{\mathbf{i}}$ at leaf node $\mathrm{leaf}_{\tilde \ell_{\mathbf{i}}}$ at step $\tilde{h}_{\mathbf{i}}$, it goes to the good absorbing state with probability $\frac{1}{2} + 2\varepsilon_{k}$, and goes to the bad absorbing state with probability $\frac{1}{2} - 2\varepsilon_{k}$, i.e.,
\begin{align}
P^{k}_{\tilde h_{\mathbf{i}},\mathbf{i}}\left(s\Big| \mathrm{leaf}_{\tilde \ell_{\mathbf{i}}},\tilde a_{\mathbf{i}}\right)
=
\begin{cases}
\frac12+2\varepsilon_k, & s=s_{\mathrm{g}},\\
\frac12-2\varepsilon_k, & s=s_{\mathrm{b}},\\
0, & \mathrm{otherwise},
\end{cases}
\end{align}
The rest of the value of the transition kernel follows the same distributions in \eqref{eq:good_leaf} and \eqref{eq:other_leaf}, i.e.,
\begin{align}
P^{k}_{1+D, \mathbf{i}}(s\mid \mathrm{leaf}_{1}, a_{\mathrm{g}})=
\begin{cases}
\frac12+\varepsilon_k, & s=s_{\mathrm{g}},\\
\frac12-\varepsilon_k, & s=s_{\mathrm{b}},\\
0, & \mathrm{otherwise}.
\end{cases}
\end{align}
and for all other triples $(h,\ell,a)\notin\{(1+d,\mathrm{leaf}^{*},1),(\tilde h_{\mathbf{i}},\tilde \ell_{\mathbf{i}},\tilde a_{\mathbf{i}})\}$,
\begin{align}
P^{\mathbf{i}}_{k,h}(s\mid \mathrm{leaf}_\ell,a)=
\begin{cases}
\frac12, & s\in\{s_{\mathrm{g}},s_{\mathrm{b}}\},\\
0, & \mathrm{otherwise}.
\end{cases}
\end{align}
Hence, in this case, the optimal policy is the one that leads the agent to $\mathrm{leaf}_{\tilde \ell_{\mathbf{i}}}$ at step $\tilde{h}_{\mathbf{i}}$, and then selects the good action $\tilde a_{\mathbf{i}}$ to reach the good absorbing state with higher probability. At the absorbing states $s_{\mathrm{g}}$ and $s_{\mathrm{b}}$, the process stays in the same state deterministically regardless of the action. The transition kernel is illustrated in Figures \ref{fig:MDP_trans_0} and \ref{fig:MDP_trans_1}.

\begin{figure}[hbt]
    \centering
    \includegraphics[width=0.8\linewidth]{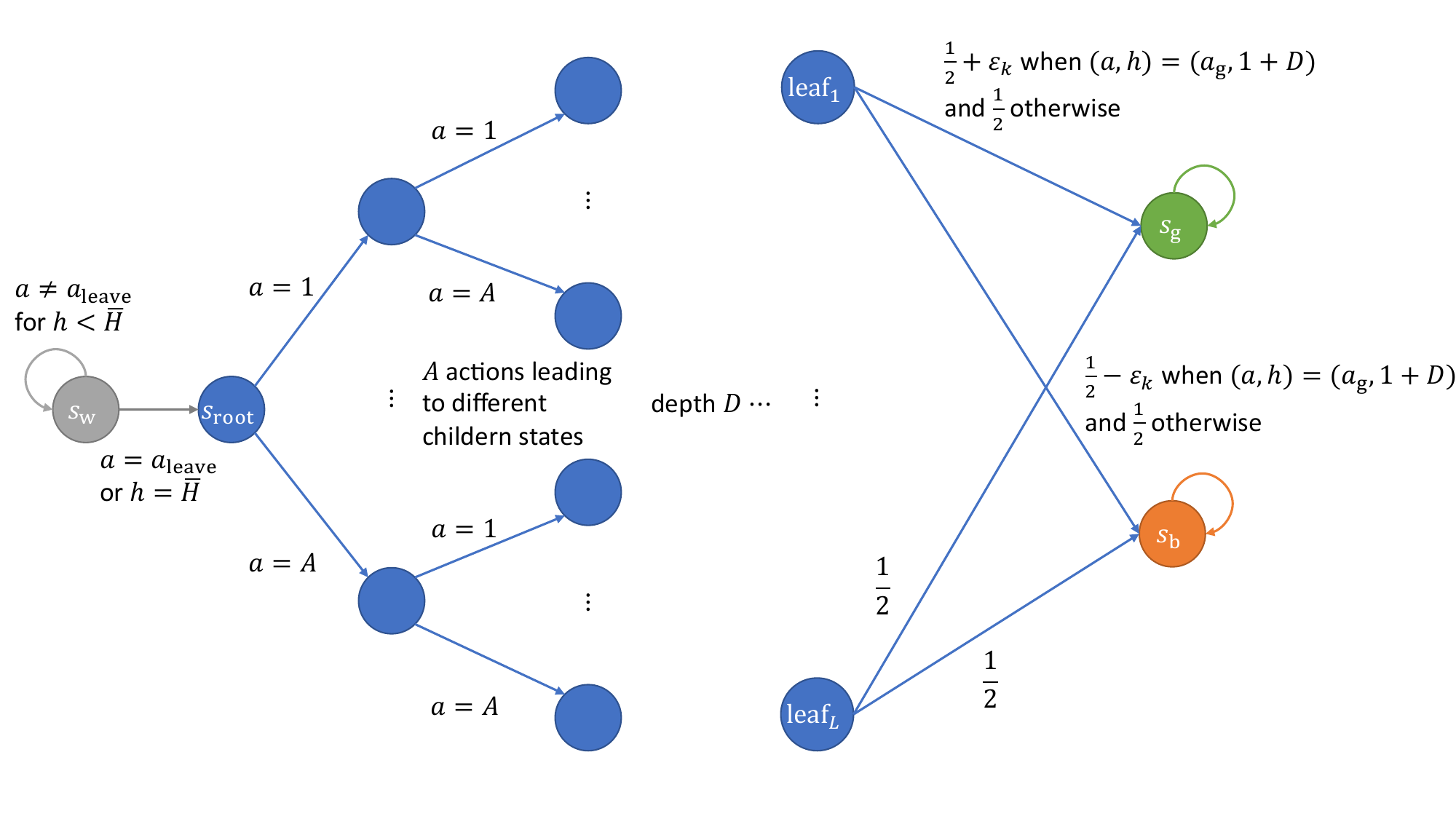}
    \caption{The transition kernel of MDP $\mathcal{M}_{\mathbf{i}}$ with $i_{k} = 0$}
    \label{fig:MDP_trans_0}
\end{figure}

\begin{figure}[hbt]
    \centering
    \includegraphics[width=0.8\linewidth]{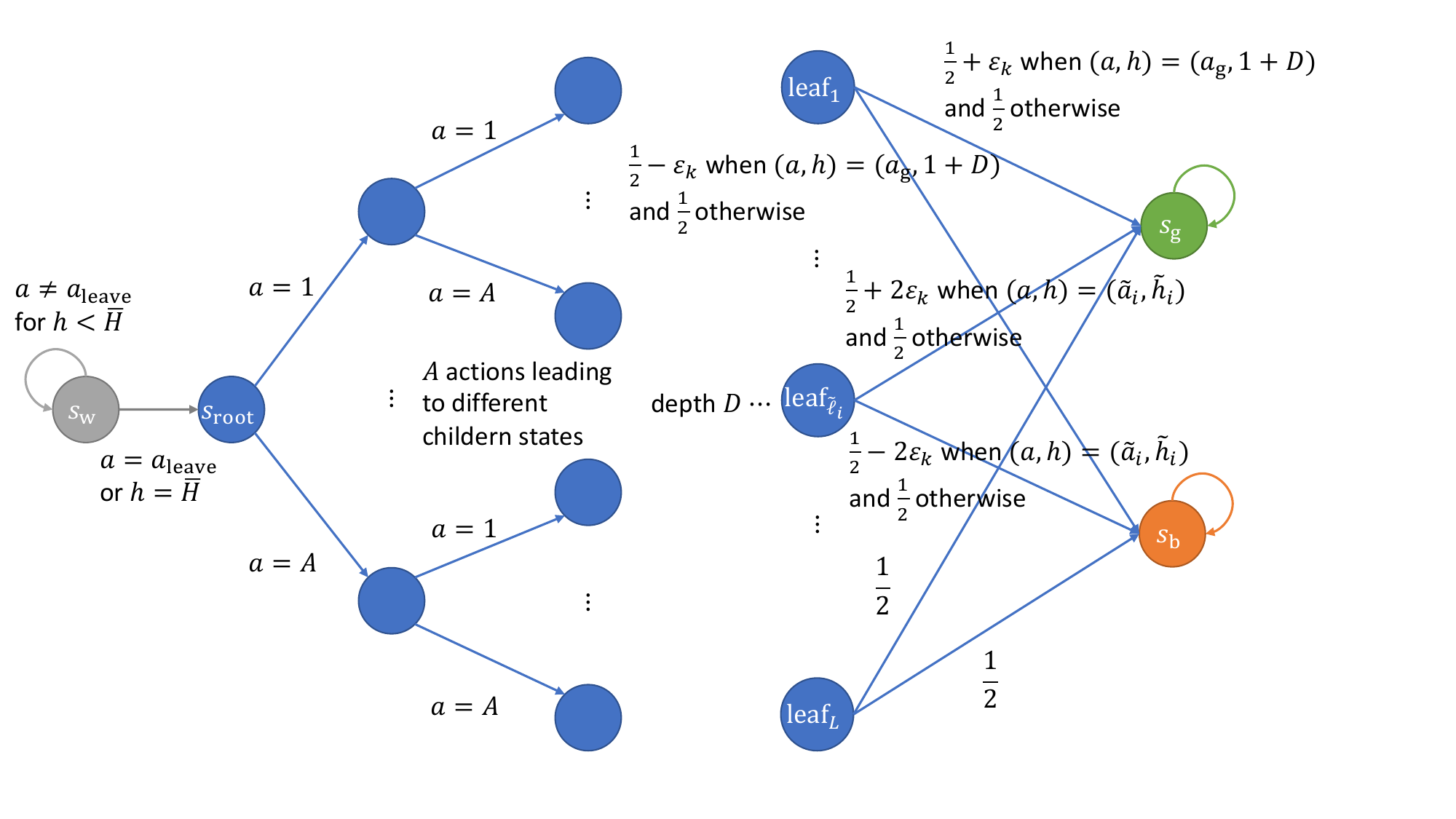}
    \caption{The transition kernel of MDP $\mathcal{M}_{\mathbf{i}}$ with $i_{k} = 1$}
    \label{fig:MDP_trans_1}
\end{figure}

We start constructing the transition kernel of MDP $\mathcal{M}_{\mathbf{i}}$ with $\mathbf{i}$ being the all-zero vector. 
Next, we proceed to assign the transition probability of MDPs with one-hot index vectors.
Then, we proceed to MDPs with one more $1$ bit in their index vectors. 
We continue this process until all transition probabilities are assigned in the MDP with all-one index vector. This process is illustrated in Figure \ref{fig:MDP_constr} for the case where $N_{T} = 2$.

\begin{figure}
    \centering
    \includegraphics[width=\linewidth]{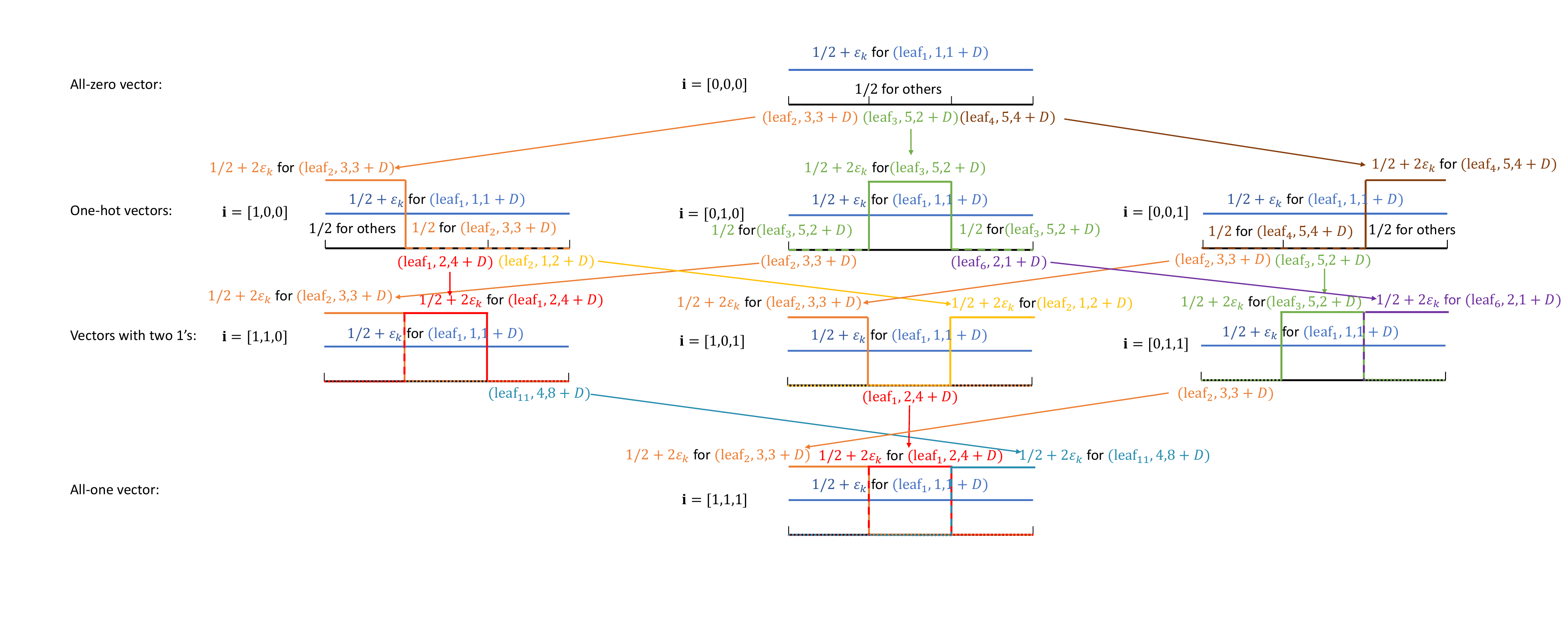}
    \caption{The process of transition probability assignment for the MDPs with $N_{T} = 2$. The lines represent the value of the transition probability $P_{h,\mathbf{i}}^{k}(s_{\mathrm{g}}|\mathrm{leaf}_{\ell}, a)$, and the intervals are the stationary segments. The colored triple underneath each interval denotes the triple $(\mathrm{leaf}_{\tilde{\ell}_{\mathbf{i}}}, \tilde{a}_{\mathbf{i}}, \tilde{h}_{\mathbf{i}})$ in \eqref{eq:least_visit}.}
    \label{fig:MDP_constr}
\end{figure}

Let $\mathcal{R}_{\mathbf{i},k}(\pi)$ be the dynamic regret of $\pi$ over the $k^{\mathrm{th}}$ stationary segment on instance $\mathcal{M}_{\mathbf{i}}$,
i.e.,
\begin{align*}
\mathcal{R}_{\mathbf{i},k}(\pi):=\sum_{t=\nu_{k-1}}^{\nu_k-1}\Bigl(V_1^{t,\star}(s_1^t)-V_1^{t,\pi}(s_1^t)\Bigr).
\end{align*}

Fix $\mathbf{i}\in\{0,1\}^{N_T+1}$ and $k\in[N_T+1]$, and let $\mathbf{j}$ be the index obtained by flipping the $k^{\mathrm{th}}$ bit, i.e., $j_k\neq i_k$ and $j_l=i_l$ for all $l\neq k$; the map $\mathbf{i}\mapsto\mathbf{j}$ is
bijective. Without loss of generality, we assume that $i_{k} = 0$ and $j_{k} = 1$. Then
\begin{align}
\mathcal{R}_{\mathbf{i},k}(\pi)
&= \sum_{\substack{(h,\ell,a)\neq(1+D,1,a_{\mathrm{g}})}}
\varepsilon_k (H-\bar H-D)\,
\mathbb{E}_{\mathbf{i},\pi}\!\left[n_k(\mathrm{leaf}_\ell,a,h)\right] \nonumber\\
&=
\varepsilon_k (H-\bar H-D)\, \left( \nu_{k} - \nu_{k-1} - 
\mathbb{E}_{\mathbf{i},\pi}\!\left[n_k(\mathrm{leaf}_{1},a_{\mathrm{g}},1+D)\right] \right)\nonumber\\
&\geq
\varepsilon_k(H-\bar H-D)\nonumber\\
&\quad\cdot
\left( \nu_{k} - \nu_{k-1} - 
\mathbb{E}_{\mathbf{i},\pi}\!\left[n_k(\mathrm{leaf}_{1},a_{\mathrm{g}},1+D) \bigg| n_k(\mathrm{leaf}_{1},a_{\mathrm{g}},1+D)\le \frac{\nu_k-\nu_{k-1}}{2} \right] \right) \nonumber\\
&\quad\cdot
\mathbb{P}_{\mathbf{i},\pi}\!\left(
n_k(\mathrm{leaf}_{1},a_{\mathrm{g}},1+D)\le \frac{\nu_k-\nu_{k-1}}{2}
\right) \nonumber\\
&\geq
\varepsilon_k(H-\bar H-D)\,
\frac{\nu_k-\nu_{k-1}}{2}\,
\mathbb{P}_{\mathbf{i},\pi}\!\left(
n_k(\mathrm{leaf}_{1},a_{\mathrm{g}},1+D)\le \frac{\nu_k-\nu_{k-1}}{2}
\right).
\end{align}
Similarly,
\begin{align}
\mathcal{R}_{\mathbf{j},k}(\pi)
&= \varepsilon_k(H-\bar H-D)\,
\mathbb{E}_{\mathbf{j},\pi}\!\left[n_k(\mathrm{leaf}_1,a_{\mathrm{g}},1+D)\right] \nonumber\\
&\quad+ \sum_{\substack{(h,\ell,a)\notin \{(\tilde h_{\mathbf{i}},\tilde \ell_{\mathbf{i}},\tilde a_{\mathbf{i}}), (1+D, 1, a_{\mathrm{g}})\}}}
2\varepsilon_k(H-\bar H-D)\,
\mathbb{E}_{\mathbf{j},\pi}\!\left[n_k(\mathrm{leaf}_\ell,a,h)\right] \nonumber\\
&\geq \varepsilon_k(H-\bar H-D)\,
\mathbb{E}_{\mathbf{j},\pi}\!\left[n_k(\mathrm{leaf}_1,a_{\mathrm{g}},1+D)\right] \nonumber\\
&\geq
\varepsilon_k(H-\bar H-D)\,
\mathbb{E}_{\mathbf{j},\pi}\!\left[n_k(\mathrm{leaf}_1,a_{\mathrm{g}},1+D)\bigg|n_k(\mathrm{leaf}_{1},a_{\mathrm{g}},1+D)
> \frac{\nu_k-\nu_{k-1}}{2}\right]\nonumber\\
&\quad\cdot\mathbb{P}_{\mathbf{j},\pi}\!\left(
n_k(\mathrm{leaf}_{1},a_{\mathrm{g}},1+D)
> \frac{\nu_k-\nu_{k-1}}{2}
\right)\nonumber\\
&\geq \varepsilon_{k}(H-\bar H-D)\frac{\nu_k-\nu_{k-1}}{2}\,
\mathbb{P}_{\mathbf{j},\pi}\!\left(
n_k(\mathrm{leaf}_{1},a_{\mathrm{g}},1+D)
> \frac{\nu_k-\nu_{k-1}}{2}
\right).
\end{align}
Therefore,
\begin{align}\label{eq:reg_low}
&\mathcal{R}_{\mathbf{i},k}(\pi)+\mathcal{R}_{\mathbf{j},k}(\pi) \nonumber
\\
&
\geq\varepsilon_k(H\!-\!\bar{H}\!-\!D)\frac{\nu_k\!-\!\nu_{k-1}}{2}\nonumber\\
&\quad\cdot\left[
\mathbb{P}_{\mathbf{i},\pi}\!\left(\!n_k(\mathrm{leaf}_{1},a_{\mathrm{g}},1\!+\!D)\!\le\! \frac{\nu_k\!-\!\nu_{k-1}}{2}\!\right)\!+
\mathbb{P}_{\mathbf{j},\pi}\!\left(\!
n_k(\mathrm{leaf}_{1},a_{\mathrm{g}},1\!+\!D)
\!\ge\! \frac{\nu_k\!-\!\nu_{k-1}}{2}\!
\right)
\right].
\end{align}
Let $\tilde{\mathbf{i}}\coloneqq(i_{1}, \dots, i_{k}, 0, \dots, 0)$ be the index vector obtained by making the bits of $\mathbf{i}$ after the $k^{\mathrm{th}}$ bit become $0$. Similarly, let $\tilde{\mathbf{j}}\coloneqq(j_{1}, \dots, j_{k}, 0, \dots, 0)$ be the index vector obtained by making the bits of $\mathbf{j}$ after the $k^{\mathrm{th}}$ bit become $0$. Then, due to the fact that $\mathbb{P}^{k}_{h, \mathbf{i}}$ and $\mathbb{P}^{k}_{h, \tilde{\mathbf{i}}}$ are identical up to the $k^{\mathrm{th}}$ interval, and that $\mathbb{P}^{k}_{h, \mathbf{j}}$ and $\mathbb{P}^{k}_{h, \tilde{\mathbf{j}}}$ are identical up to the $k^{\mathrm{th}}$ interval, we can perform change of measure and obtain
\begin{align}\label{eq:change_measure}
&\mathcal{R}_{\mathbf{i},k}(\pi)+\mathcal{R}_{\mathbf{j},k}(\pi) \nonumber
\\
&
\geq\varepsilon_k(H\!-\!\bar{H}\!-\!D)\frac{\nu_k\!-\!\nu_{k-1}}{2}\nonumber\\
&\quad\cdot
\left[
\mathbb{P}_{\tilde{\mathbf{i}},\pi}\!\left(\!n_k(\mathrm{leaf}_{1},a_{\mathrm{g}},1\!+\!D)\!\le\! \frac{\nu_k\!-\!\nu_{k-1}}{2}\!\right)\!+
\mathbb{P}_{\tilde{\mathbf{j}},\pi}\!\left(\!
n_k(\mathrm{leaf}_{1},a_{\mathrm{g}},1\!+\!D)
\!\ge\! \frac{\nu_k\!-\!\nu_{k-1}}{2}\!
\right)
\right].
\end{align}
By the Bretagnolle--Huber inequality, for any event $A$ and measures $\mathbb{P}$ and $\mathbb{Q}$,
\begin{align}
\mathbb{P}(A)+\mathbb{Q}(A^c)\ge \frac12 \exp\!\left(-D_{\mathrm{KL}}(\mathbb{P}\|\mathbb{Q})\right).
\end{align}
where $D_{\mathrm{KL}} (\mathbb{P}\|\mathbb{Q})$ denotes the KL divergence between $\mathbb{P}$ and $\mathbb{Q}$.
Thus, applying it to the bracketed term in the right-hand side of \eqref{eq:change_measure} above yields
\begin{align}
\mathcal{R}_{\mathbf{i},k}(\pi)+\mathcal{R}_{\mathbf{j},k}(\pi)
\ge
\varepsilon_k(H-\bar H-D)\,\frac{\nu_k-\nu_{k-1}}{2}\,
\exp\!\left(-D_{\mathrm{KL}} \left(\mathbb{P}_{\tilde{\mathbf{i}},\pi}\|\mathbb{P}_{\tilde{\mathbf{j}},\pi}\right)\right).
\end{align}

Since rewards are deterministic, all randomness comes from states and actions. Let the trajectory up to episode $t$
and step $h$ be
\begin{align}
\zeta_h^t \coloneqq (s_1^1,a_1^1,\ldots,a_{H}^1,s_{H+1}^1,\; s_{1}^2,a_1^2,\ldots,a_{H}^2,s_{H+1}^2,\ldots, a_{h-1}^t,s_h^t).
\end{align}
Then,
\begin{align}
D_{\mathrm{KL}}\left(\mathbb{P}_{\tilde{\mathbf{i}},\pi}\|\mathbb{P}_{\tilde{\mathbf{j}},\pi}\right)
&=
\mathbb{E}_{\tilde{\mathbf{i}},\pi}\!\left[
\log\frac{
\prod_{l=1}^{N_T+1}\prod_{t=\nu_{l-1}}^{\nu_l-1}\prod_{h=1}^{H}
\pi(a_h^t\mid \zeta_h^t)\,P^{l}_{h,\tilde{\mathbf{i}}}(s_{h+1}^t\mid s_h^t,a_h^t)
}{
\prod_{l=1}^{N_T+1}\prod_{t=\nu_{l-1}}^{\nu_l-1}\prod_{h=1}^{H}
\pi(a_h^t\mid \zeta_h^t)\,P^{l}_{h,\tilde{\mathbf{j}}}(s_{h+1}^t\mid s_h^t,a_h^t)
}
\right]\nonumber
\\
&\overset{(a)}{=}
\mathbb{E}_{\tilde{\mathbf{i}},\pi}\!\left[
\sum_{t=\nu_{k-1}}^{\nu_k-1}
\log\frac{
P^{k}_{\tilde{h}_{\mathbf{i}},\tilde{\mathbf{i}}}\left(s_{\tilde h_{\mathbf{i}}+1}^t\Big| s_{\tilde h_{\mathbf{i}}}^t,a_{\tilde h_{\mathbf{i}}}^t\right)
}{
P^{k}_{\tilde{h}_{\mathbf{i}},\tilde{\mathbf{j}}}\left(s_{\tilde h_{\mathbf{i}}+1}^t\Big| s_{\tilde h_{\mathbf{i}}}^t,a_{\tilde h_{\mathbf{i}}}^t\right)
}
\right].
\end{align}
In step $(a)$, we use the fact that the transition kernel of $\mathcal{M}_{\tilde{\mathbf{i}}}$ and that of $\mathcal{M}_{\tilde{\mathbf{j}}}$ differ only at step $\tilde h_{\mathbf{i}}$ during the $k^{\mathrm{th}}$ stationary segment.
Expanding log-likelihood-ratios yields
\begin{align}
D_{\mathrm{KL}}(\mathbb{P}_{\tilde{\mathbf{i}},\pi}\|\mathbb{P}_{\tilde{\mathbf{j}},\pi})
&\overset{(a)}{=}
\sum_{t=\nu_{k-1}}^{\nu_k-1}
\mathbb{E}_{\tilde{\mathbf{i}},\pi}\!\Bigl[
-\log(1+4\varepsilon_k)\,
\mathbf{1}\left\{s_{\tilde h_{\mathbf{i}}}^t=\mathrm{leaf}_{\tilde \ell_{\mathbf{i}}},\ a_{\tilde h_{\mathbf{i}}}^t=\tilde a_{\mathbf{i}},\ s_{\tilde h_{\mathbf{i}}+1}^t=s_{\mathrm{g}}\right\}
\Bigr]\nonumber
\\
&\qquad+
\sum_{t=\nu_{k-1}}^{\nu_k-1}
\mathbb{E}_{\tilde{\mathbf{i}},\pi}\!\Bigl[
-\log(1-4\varepsilon_k)\,
\mathbf{1}\left\{s_{\tilde h_{\mathbf{i}}}^t=\mathrm{leaf}_{\tilde \ell_{\mathbf{i}}},\ a_{\tilde h_{\mathbf{i}}}^t=\tilde a_{\mathbf{i}},\ s_{\tilde h_{\mathbf{i}}+1}^t=s_{\mathrm{b}}\right\}
\Bigr]\nonumber
\\
&=
\sum_{t=\nu_{k-1}}^{\nu_k-1}
-\log(1+4\varepsilon_k)\,
\mathbb{P}_{\tilde{\mathbf{i}},\pi}\bigl(s_{\tilde h_{\mathbf{i}}}^t=\mathrm{leaf}_{\tilde \ell_{\mathbf{i}}},\ a_{\tilde h_{\mathbf{i}}}^t=\tilde a_{\mathbf{i}},\ s_{\tilde h_{\mathbf{i}}+1}^t=s_{\mathrm{g}}\bigr)\nonumber
\\
&\qquad+
\sum_{t=\nu_{k-1}}^{\nu_k-1}
-\log(1-4\varepsilon_k)\,
\mathbb{P}_{\tilde{\mathbf{i}},\pi}\bigl(s_{\tilde h_{\mathbf{i}}}^t=\mathrm{leaf}_{\tilde \ell_{\mathbf{i}}},\ a_{\tilde h_{\mathbf{i}}}^t=\tilde a_{\mathbf{i}},\ s_{\tilde h_{\mathbf{i}}+1}^t=s_{\mathrm{b}}\bigr)\nonumber\\
&\overset{(b)}{=}
\sum_{t=\nu_{k-1}}^{\nu_k-1}
-\frac12\log(1+4\varepsilon_k)\,
\mathbb{P}_{\tilde{\mathbf{i}},\pi}\bigl(s_{\tilde h_{\mathbf{i}}}^t=\mathrm{leaf}_{\tilde \ell_{\mathbf{i}}},\ a_{\tilde h_{\mathbf{i}}}^t=\tilde a_{\mathbf{i}}\bigr)\nonumber
\\
&\qquad+
\sum_{t=\nu_{k-1}}^{\nu_k-1}
-\frac12\log(1-4\varepsilon_k)\,
\mathbb{P}_{\tilde{\mathbf{i}},\pi}\bigl(s_{\tilde h_{\mathbf{i}}}^t=\mathrm{leaf}_{\tilde \ell_{\mathbf{i}}},\ a_{\tilde h_{\mathbf{i}}}^t=\tilde a_{\mathbf{i}}\bigr)\nonumber
\\
&=
\sum_{t=\nu_{k-1}}^{\nu_k-1}
-\frac12\log(1-16\varepsilon_k^2)\,
\mathbb{P}_{\tilde{\mathbf{i}},\pi}\bigl(s_{\tilde h_{\mathbf{i}}}^t=\mathrm{leaf}_{\tilde \ell_{\mathbf{i}}},\ a_{\tilde h_{\mathbf{i}}}^t=\tilde a_{\mathbf{i}}\bigr)\nonumber
\\
&=
-\frac12\log(1-16\varepsilon_k^2)\,
\mathbb{E}_{\tilde{\mathbf{i}},\pi}\!\left[
\sum_{t=\nu_{k-1}}^{\nu_k-1}
\mathbf{1}\{s_{\tilde h_{\mathbf{i}}}^t=\mathrm{leaf}_{\tilde \ell_{\mathbf{i}}},\ a_{\tilde h_{\mathbf{i}}}^t=\tilde a_{\mathbf{i}}\}
\right]\nonumber
\\
&=
-\frac12\log(1-16\varepsilon_k^2)\,
\mathbb{E}_{\tilde{\mathbf{i}},\pi}\!\bigl[n_k(\mathrm{leaf}_{\tilde \ell_{\mathbf{i}}},\tilde a_{\mathbf{i}},\tilde h_{\mathbf{i}})\bigr].
\end{align}
Step $(a)$ follows from the fact that $P^{k}_{\tilde{h}_{\mathbf{i}},\tilde{\mathbf{i}}}(s_{\tilde h_{\mathbf{i}}+1}^t=s_{\mathrm{g}} | s_{\tilde h_{\mathbf{i}}}^t= \mathrm{leaf}_{\tilde{\ell}_{\mathbf{i}}}, a_{\tilde h_{\mathbf{i}}}^t = \tilde{a}_{\mathbf{i}})/P^{k}_{\tilde{h}_{\mathbf{i}},\tilde{\mathbf{j}}}(s_{\tilde h_{\mathbf{i}}+1}^t=s_{\mathrm{g}} | s_{\tilde h_{\mathbf{i}}}^t= \mathrm{leaf}_{\tilde{\ell}_{\mathbf{i}}}, a_{\tilde h_{\mathbf{i}}}^t = \tilde{a}_{\mathbf{i}})) = 1/(1 + 4\varepsilon_{k})$ and that $P^{k}_{\tilde{h}_{\mathbf{i}},\tilde{\mathbf{i}}}(s_{\tilde h_{\mathbf{i}}+1}^t=s_{\mathrm{b}} | s_{\tilde h_{\mathbf{i}}}^t= \mathrm{leaf}_{\tilde{\ell}_{\mathbf{i}}}, a_{\tilde h_{\mathbf{i}}}^t = \tilde{a}_{\mathbf{i}})/P^{k}_{\tilde{h}_{\mathbf{i}},\tilde{\mathbf{j}}}(s_{\tilde h_{\mathbf{i}}+1}^t=s_{\mathrm{b}} | s_{\tilde h_{\mathbf{i}}}^t= \mathrm{leaf}_{\tilde{\ell}_{\mathbf{i}}}, a_{\tilde h_{\mathbf{i}}}^t = \tilde{a}_{\mathbf{i}})) = 1/(1 - 4\varepsilon_{k})$. 
Step $(b)$ stems from the fact that $P^{k}_{\tilde{h}_{\mathbf{i}},\tilde{\mathbf{i}}}(s_{\tilde h_{\mathbf{i}}+1}^t=s_{\mathrm{g}} | s_{\tilde h_{\mathbf{i}}}^t= \mathrm{leaf}_{\tilde{\ell}_{\mathbf{i}}}, a_{\tilde h_{\mathbf{i}}}^t = \tilde{a}_{\mathbf{i}}) = P^{k}_{\tilde{h}_{\mathbf{i}},\tilde{\mathbf{i}}}(s_{\tilde h_{\mathbf{i}}+1}^t=s_{\mathrm{b}} | s_{\tilde h_{\mathbf{i}}}^t= \mathrm{leaf}_{\tilde{\ell}_{\mathbf{i}}}, a_{\tilde h_{\mathbf{i}}}^t = \tilde{a}_{\mathbf{i}}) = 1/2$.
Using $\log(1-x)\ge -\tfrac{x}{1-x}$ for $x\in[0,1)$ gives
\begin{align}
D_{\mathrm{KL}}(\mathbb{P}_{\tilde{\mathbf{i}},\pi}\|\mathbb{P}_{\tilde{\mathbf{j}},\pi})
\le
\frac{8\varepsilon_k^2}{1-16\varepsilon_k^2}\,
\mathbb{E}_{\tilde{\mathbf{i}},\pi}\!\bigl[n_k(\mathrm{leaf}_{\tilde \ell_{\mathbf{i}}},\tilde a_{\mathbf{i}},\tilde h_{\mathbf{i}})\bigr].
\end{align}
Recall that $L$ is the number of leaf nodes. By the definition of $(\tilde h_{\mathbf{i}},\tilde \ell_{\mathbf{i}},\tilde a_{\mathbf{i}})$ in \eqref{eq:least_visit}, we have
\begin{align}
\mathbb{E}_{\tilde{\mathbf{i}},\pi}\!\bigl[n_k(\mathrm{leaf}_{\tilde \ell_{\mathbf{i}}},\tilde a_{\mathbf{i}},\tilde h_{\mathbf{i}})\bigr]
\le
\frac{\nu_k-\nu_{k-1}}{AL\bar H-1},
\end{align}
and hence
\begin{align}
D_{\mathrm{KL}}(\mathbb{P}_{\tilde{\mathbf{i}},\pi}\|\mathbb{P}_{\tilde{\mathbf{j}},\pi})
\le
\frac{8\varepsilon_k^2}{1-16\varepsilon_k^2}\,
\frac{\nu_k-\nu_{k-1}}{AL\bar H-1}.
\end{align}
Combining the Bretagnolle--Huber lower bound with the above KL upper bound yields
\begin{align}
\mathcal{R}_{\mathbf{i},k}(\pi)+\mathcal{R}_{\mathbf{j},k}(\pi)
&\ge
\varepsilon_k(H-\bar H-D)\,\frac{\nu_k-\nu_{k-1}}{2}\,
\exp\!\left(
-\frac{8\varepsilon_k^2}{1-16\varepsilon_k^2}\frac{\nu_k-\nu_{k-1}}{AL\bar H-1}
\right).
\end{align}
Since $H\ge 3D$, we have $H-\bar H-D \ge \frac23H-\bar H$. Then, set $\varepsilon_{k} = [16 + 8(\nu_{k} - \nu_{k-1})/(AL\bar{H} - 1)]^{-1/2}$, we have
\begin{align}
\mathcal{R}_{\mathbf{i},k}(\pi)+\mathcal{R}_{\mathbf{j},k}(\pi)
\ge
\frac{1}{\sqrt{16 + 8(\nu_{k} - \nu_{k-1})/(AL\bar{H} - 1)}}\Bigl(\frac23H-\bar H\Bigr)\,\frac{\nu_k-\nu_{k-1}}{2e}.
\end{align}
Notice that $\varepsilon_{k} \leq 1/4$, ensuring that the transition probabilities remain in $[0,1]$. Recall that $L = A^{D-1} = (S-3)(1-1/A)+1/A$. Then, by setting $\bar{H} = H/3$, we have
\begin{align}
\mathcal{R}_{\mathbf{i},k}(\pi)+\mathcal{R}_{\mathbf{j},k}(\pi)
\geq
\frac{H}{6e}\sqrt{\frac{(\nu_{k} - \nu_{k-1})^{2}}{16 + 8 (\nu_{k} - \nu_{k-1})/((S-3)(A-1)H/3 + H/3 - 1)}}.
\end{align}
There are $2^{N_T}$ disjoint pairs $(\mathbf{i},\mathbf{j})$ for each fixed $k$, hence
\begin{align}
\sum_{\mathbf{i}\in\{0,1\}^{N_T+1}} \mathcal{R}_{\mathbf{i},k}(\pi)
\geq
2^{N_T}\cdot
\frac{H}{6e}\sqrt{\frac{(\nu_{k} - \nu_{k-1})^{2}}{16 + 8 (\nu_{k} - \nu_{k-1})/((S-3)(A-1)H/3 + H/3 - 1)}}.
\end{align}
Without loss of generality, we assume that $T$ is divisible by $N_{T} + 1$, meaning that $\nu_{k} - \nu_{k-1} = T/(N_{T} + 1)$. Thereupon, this implies that,
\begin{align}
&2^{-(N_T+1)}\sum_{\mathbf{i}\in\{0,1\}^{N_T+1}}\sum_{k=1}^{N_T+1} \mathcal{R}_{\mathbf{i},k}(\pi)
\nonumber\\
&\geq
\frac{H}{12e}\sqrt{\frac{1}{16/T^{2} + 8/T(N_{T}+1)((S-3)(A-1)H/3 + H/3 - 1)}}.
\end{align}
Consequently, there exists $\hat{\mathbf{i}}\in\{0,1\}^{N_T+1}$ such that
\begin{align}
\sum_{k=1}^{N_T+1} \mathcal{R}_{\hat{\mathbf{i}},k}(\pi)
\ge
\frac{H}{12e}\sqrt{\frac{1}{16/T^{2} + 8/T(N_{T}+1)((S-3)(A-1)H/3 + H/3 - 1)}}.
\end{align}

Since $\pi$ was arbitrary, it follows that for any algorithm (possibly history-dependent) there exists a PS tabular
MDP instance with exactly $N_T$ changes such that the expected dynamic regret (where $T$ is the number of episodes) satisfies
\begin{align}
\mathcal{R}(\pi, T)=\Omega ( \sqrt{SAH^{3}N_{T}T} ).
\end{align}
This completes the proof.
\end{proof}

\subsection{Proof of Theorem \ref{thrm:ps-lin-mdp-lower}}
\label{app:ps-lin-mdp-lower-proof}

\begin{proof}

To prove the regret lower bound for piecewise stationary finite-horizon episodic linear MDPs, we generalize the proof of Theorem 8 and Remark 23 in \cite{zhou2021lowerboundlin}, in which a set of hard-to-learn linear MDP instances are constructed, and the lower bound on the expected regret averaged over these instances is derived.

Assume that $d\geq 5$, $H\geq 4$, $\lfloor T/(N_T+1) \rfloor \geq (d-2)^{2}H/2$, and $T/(N_{T} + 1) \geq 8$. The "hard-to-learn" linear MDP has the following formulation: Let the state space be
\begin{equation}
    \mathcal{S} \coloneqq \{x_1,\ldots,x_H,x_{H+1},x_{H+2}\}
\end{equation}
and the action set be
\begin{equation}
    \mathcal{A} \coloneqq \{+1, -1\}^{d-2}.
\end{equation}
Recall that $\nu_{k}$ denotes the $k^{\mathrm{th}}$ change-point and that $\nu_{0} \coloneqq 1$ and $\nu_{N_{T} + 1} \coloneqq T+1$. Similar to the proof in Theorem \ref{thrm:ps-mdp-lower}, we set the change-points to be evenly separated over the $T$ episodes: for the $k^{\mathrm{th}}$ stationary segment, its interval length 
$\nu_{k} - \nu_{k-1} = \lceil T/(N_T+1)\rceil$ if $k\le T\bmod (N_T+1)$ and $\lfloor T/(N_T+1)\rfloor$ otherwise. With slight abuse of notations, we let $\mu_{h,k}$ denote the measure $\mu_{h,t}$ over the $k^{\mathrm{th}}$ stationary segment and let $\theta_{h,k}$ denote the vector $\theta_{h,t}$ over the $k^{\mathrm{th}}$ stationary segment, i.e., $\mu_{h,t} = \mu_{h,k}$ and $\theta_{h,t} = \theta_{h,k}$ for any $t \in \{\nu_{k-1}, \dots, \nu_{k} - 1\}$.
For constructing $\mu_{h,k}$ and $\theta_{h,k}$, we define the parameters $\delta \coloneqq 1/H$, $\Delta \coloneqq \sqrt{\delta/(32\lfloor T/(N_{T} + 1)\rfloor)}$, $\varrho \coloneqq \sqrt{1/(1 + \Delta(d-2))}$, and $\varsigma \coloneqq \sqrt{\Delta/(1 + \Delta(d-2))}$. Then, for any $h \in [H]$ and $ k \in [N_{T} + 1]$, define
\begin{align}
    \theta_{h,k} \coloneqq \begin{bmatrix}
        \mathbf{0}^{\top} & 1
    \end{bmatrix}^{\top},
\end{align}
where $\mathbf{0} \in \mathbb{R}^{d-1}$, and
\begin{equation}
    \mu_{h,k}(s') \coloneqq \begin{dcases}
        \begin{bmatrix}
            (1 - \delta)/\varrho & -\varphi^{\top}_{h,k}/\varsigma & 0
        \end{bmatrix}^{\top},& s' = x_{h+1},\\
        \begin{bmatrix}
            \delta/\varrho & \varphi^{\top}_{h,k}/\varsigma & 1
        \end{bmatrix}^{\top},& s' = x_{H+2},\\
        \mathbf{0}, &\text{otherwise,}
    \end{dcases}
\end{equation}
where $\mathbf{0} \in \mathbb{R}^{d}$ and $\varphi_{h,k} \in \{+\Delta, -\Delta\}^{d-2}$. Since our hard-to-learn linear MDP instances is determined by the set of vectors $\varphi =\{ \varphi_{h,k}: h \in [H],\, k \in [N_{T} + 1] \}$, we use $\varphi$ to denote a linear MDP. It is evident that $\lVert \theta_{h,k} \rVert_{2} = 1$ and
\begin{equation}
    \lVert\mu_{h,k}(\mathcal{S})\rVert_{2}^{2} = \left\lVert\begin{bmatrix}
        1/\varrho & \mathbf{0}^{\top} & 1
    \end{bmatrix}\right\rVert_{2}^{2} = 1 + \Delta(d-2) + 1 \overset{(a)}{\leq} d
\end{equation}
where we leverage the assumption that $T/(N_{T} + 1) \geq 11$ and $H \geq 3$ in step $(a)$. In addition, the feature map $\phi$ is defined as follows:
\begin{equation}
    \phi(s, a) = \begin{dcases}
        \begin{bmatrix}
            \varrho & \varsigma a^{\top} \quad 0
        \end{bmatrix}^{\top}, &s\neq x_{H+2},\\
        \begin{bmatrix}
            0 & \mathbf{0}^{\top} \quad 1
        \end{bmatrix}^{\top}, &s=x_{H+2}.
    \end{dcases}
\end{equation}
We can also see that for any $s \in \mathcal{S}$ and $a \in \mathcal{A}$,
\begin{equation}
    \lVert \phi(s,a) \rVert^{2}_{2} = \begin{dcases}
        \frac{1 + \Delta(d-2)}{1 + \Delta(d-2)} = 1,& s\neq x_{H+2},\\
        1,& s= x_{H+2}.
    \end{dcases}
\end{equation}
In all episode, the deterministic starting state is $x_{1}$, i.e., $s^{t}_{1} = x_{1}$ for all $t \in [T]$. In our hard-to-learn linear MDPs, the reward $r^{t}_{h}$ is $1$ if $s^{t}_{h} = x_{H+2}$ and $0$ otherwise. Therefore, $x_{H+2}$ can be viewed as the "good" state, while the others are the "bad" states. The transition kernel is illustrated in the following figure. At step $h$, the state $s_{h}^{t}$ can either be $x_{h}$ or $x_{H+2}$. If $s^{t}_{h} = x_{h}$, then the agent transitions to $x_{h+1}$ with probability $1 - \delta - \langle \varphi_{h,k}, a_{h}^{t} \rangle$ or $x_{H+2}$ with probability $\delta + \langle \varphi_{h,k}, a_{h}^{t} \rangle$. If $s^{t}_{h} = x_{h}$, then the agent remains at $x_{H+2}$ with probability $1$.
\begin{figure}[ht]
    \centering
    \includegraphics[width=0.95\linewidth]{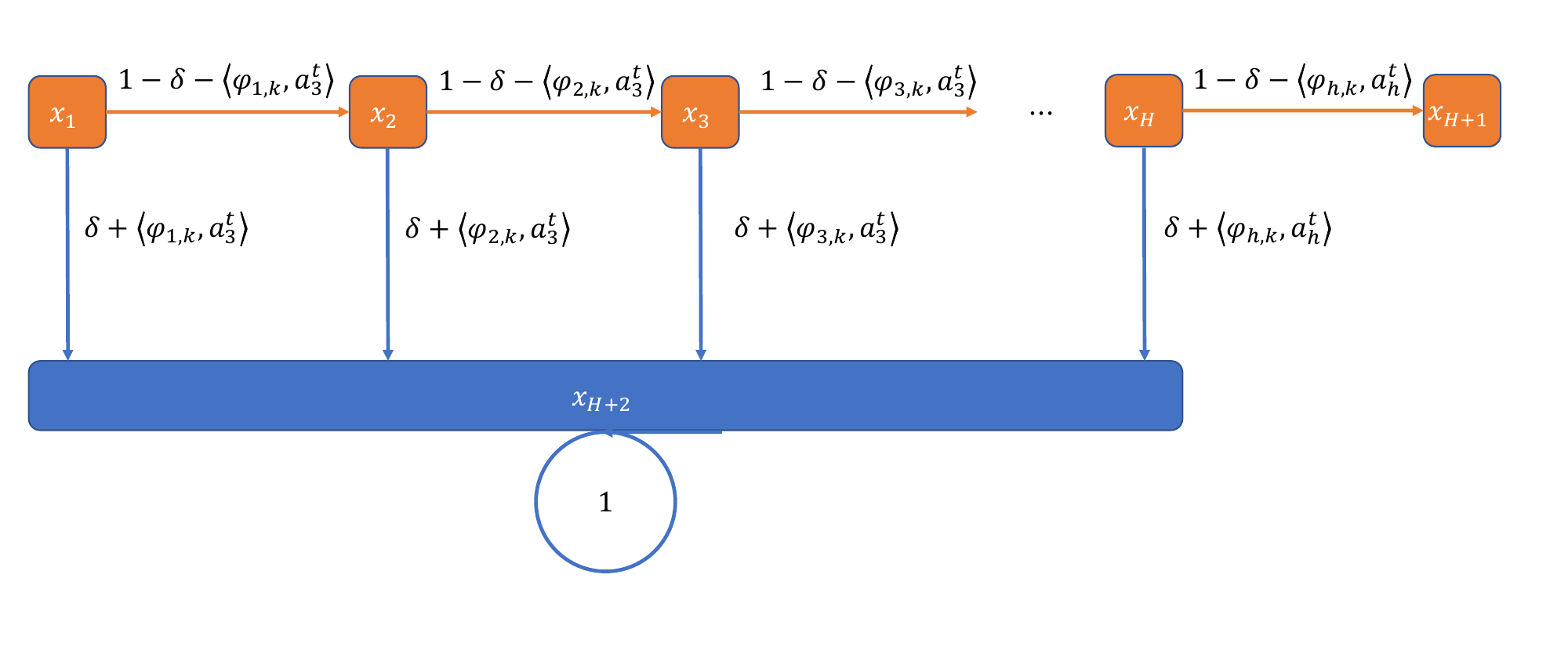}
    \caption{The transition kernel of a hard-to-learn linear MDP instance. The orange states are the "bad" states with zero reward, while the blue state is the "good" state with unit reward}
    \label{fig:linearMDP}
\end{figure}

We now proceed to prove the regret lower bound on these hard-to-learn linear MDP instances. Fix an arbitrary policy $\pi$. Let $\mathbb{P}_{\varphi, \pi}$ and $\mathbb{E}_{\varphi, \pi}$ denote the probability measure and the expectation when the agent operates policy $\pi$ on the piecewise stationary finite-horizon episodic linear MDP $\varphi$. Similar to the proof to Theorem \ref{thrm:ps-mdp-lower}, we use $\zeta^{t}_{h}$ to denote the trajectory up to episode $t$
and step $h$, i.e.,
\begin{align}
\zeta_h^t \coloneqq (s_1^1,a_1^1,\ldots,a_{H}^1,s_{H+1}^1,\; s_{1}^2,a_1^2,\ldots,a_{H}^2,s_{H+1}^2,\ldots, a_{h-1}^t,s_h^t).
\end{align}
We also define $V^{t,\pi}_{\varphi}$ as the value function on MDP $\mathcal{M}_{\mathbf{i}}$, i.e.,
\begin{equation}
    V^{t,\pi}_{\varphi}(s) \coloneqq \mathbb{E}_{\varphi, \pi}\left[\sum_{h=1}^{H} r^{t}_{h}(s^{t}_{h}, a^{t}_{h}) \Bigg| s_{1}^{t} = s\right],
\end{equation}
and define $V^{t,*}_{\varphi} \coloneqq \max_{\pi} V^{t,\pi}_{\varphi}$.
Here, we provide Lemma 24 in \cite{zhou2021lowerboundlin}, which we also leverage in our proof.

\begin{lemma}
Assume that $H \geq 3$ and $3(d-2)\Delta\leq \delta$. Then, for any $\varphi \in ((\{+\Delta, -\Delta\}^{d-2})^{H})^{N_{T} + 1}$ and $t$ in the $k^{\mathrm{th}}$ stationary segment,
\begin{align}
    &V_{\varphi}^{t,*}(x_{1}) - \mathbb{E}_{\varphi, \pi}\left[\sum_{h=1}^{H}r_{h}^{t}\left(s^{t}_{h}, a^{t}_{h}\right)\bigg|\zeta^{t}_{1}, s^{t}_{1} = x_{1} \right] \nonumber\\
    &\geq \frac{H}{10} \sum_{h=1}^{\lfloor H/2 \rfloor - 1} \left(\max_{a \in \mathcal{A}} \langle \varphi_{h,k}, a \rangle - \left\langle \varphi_{h,k}, \mathbb{E}_{\varphi,\pi}\left[a^{t}_{h}|\zeta_{1}^{t}, s^{t}_{1} = x_{1}, s^{t}_{h} = x_{h}\right] \right\rangle \right).
\end{align}
\end{lemma}
With this lemma, we can show that
\begin{align}
    &V_{\varphi}^{t,*}(x_{1}) - V_{\varphi}^{t,\pi}(x_{1})\nonumber\\
    &= \mathbb{E}_{\varphi, \pi}\left[V_{\varphi}^{t,*}(x_{1}) - \mathbb{E}_{\varphi, \pi}\left[\sum_{h=1}^{H}r_{h}^{t}\left(s^{t}_{h}, a^{t}_{h}\right)\bigg|\zeta^{t}_{1}, s^{t}_{1} = x_{1} \right] \Bigg| s^{t}_{1} = x_{1}\right] \nonumber\\
    &\geq \mathbb{E}_{\varphi, \pi}\left[\frac{H}{10} \sum_{h=1}^{\lfloor H/2 \rfloor - 1} \left(\max_{a \in \mathcal{A}} \langle \varphi_{h,k}, a \rangle - \left\langle \varphi_{h,k}, \mathbb{E}_{\varphi, \pi}\left[a^{t}_{h}|\zeta_{1}^{t}, s^{t}_{1} = x_{1}, s^{t}_{h} = x_{h}\right] \right\rangle \right) \Bigg| s^{t}_{1} = x_{1}\right] \nonumber\\
    &= \frac{H}{10} \sum_{h=1}^{\lfloor H/2 \rfloor - 1} \left(\max_{a \in \mathcal{A}} \langle \varphi_{h,k}, a \rangle - \left\langle \varphi_{h,k}, \mathbb{E}_{\varphi, \pi}\left[\mathbb{E}_{\varphi, \pi}\left[a^{t}_{h}|\zeta_{1}^{t}, s^{t}_{1} = x_{1}, s^{t}_{h} = x_{h}\right] | s^{t}_{1} = x_{1}\right]\right\rangle \right)\nonumber\\
    &= \frac{H}{10} \sum_{h=1}^{\lfloor H/2 \rfloor - 1} \left(\max_{a \in \mathcal{A}} \langle \varphi_{h,k}, a \rangle - \left\langle \varphi_{h,k}, \mathbb{E}_{\varphi, \pi}\left[a^{t}_{h}| s^{t}_{1} = x_{1}, s^{t}_{h} = x_{h}\right] \right\rangle\right).
\end{align}
Now, let $\mathcal{R}_{\varphi,\pi}$ denote the regret of policy $\pi$ on the linear MDP $\varphi$, i.e.,
\begin{equation}
    \mathcal{R}_{\varphi, \pi} \coloneqq \sum_{t=1}^{T} \left(V^{t,*}_{\varphi}(x_{1}) - V^{t,\pi}_{\varphi}(x_{1})\right).
\end{equation}
Let $a^{t}_{h}(j)$ be the element in the $j^{\mathrm{th}}$ coordinate of $a^{t}_{h}$ and $\varphi_{h,k}(j)$ be that of $\varphi_{h,k}$. Then, we can show the following minimax lower bound on the regret over all hard-to-learn linear MDP instances $\varphi \in ((\{+\Delta, -\Delta\}^{d-2})^{H})^{N_{T}+1}$.
\begin{align}
    &2^{(d-2)H(N_{T}+1)}\sup_{\varphi} \mathcal{R}_{\varphi,\pi} \nonumber\\
    &\geq \sum_{\varphi} \mathcal{R}_{\varphi,\pi} \nonumber\\
    &= \sum_{\varphi} \sum_{k=1}^{N_{T} + 1} \sum_{t = \nu_{k-1}}^{\nu_{k} - 1}\left(V^{t,*}_{\varphi}(x_{1}) - V^{t,\pi}_{\varphi}(x_{1})\right) \nonumber\\
    &\geq \sum_{\varphi} \sum_{k=1}^{N_{T} + 1} \sum_{t = \nu_{k-1}}^{\nu_{k} - 1}\frac{H}{10} \sum_{h=1}^{\lfloor H/2 \rfloor - 1} \left(\max_{a \in \mathcal{A}} \langle \varphi_{h,k}, a \rangle - \left\langle \varphi_{h,k}, \mathbb{E}_{\varphi, \pi}\left[a^{t}_{h}| s^{t}_{1} = x_{1}, s^{t}_{h} = x_{h}\right] \right\rangle\right) \nonumber\\
    &= \frac{H}{10} \sum_{k=1}^{N_{T} + 1}  \sum_{h=1}^{\lfloor H/2 \rfloor - 1}\sum_{\varphi} \sum_{t = \nu_{k-1}}^{\nu_{k} - 1} \mathbb{E}_{\varphi, \pi}\left[ \max_{a \in \mathcal{A}} \langle \varphi_{h,k}, a \rangle - \left\langle \varphi_{h,k}, a^{t}_{h} \right\rangle \bigg| s^{t}_{1} = x_{1}, s^{t}_{h} = x_{h}\right] \nonumber\\
    &= \frac{H}{10} \sum_{k=1}^{N_{T} + 1}  \sum_{h=1}^{\lfloor H/2 \rfloor - 1}\sum_{\varphi} \sum_{t = \nu_{k-1}}^{\nu_{k} - 1} \mathbb{E}_{\varphi, \pi}\!\!\left[ \sum_{j=1}^{d-2} 2\Delta\mathds{1}\{\mathrm{sgn}\left(a^{t}_{h}(j)\varphi_{h,k}(j)\right) = -1\} \Bigg| s^{t}_{1} = x_{1}, s^{t}_{h} = x_{h}\right] \nonumber\\
    &\overset{(a)}{\geq} \frac{H\Delta}{5} \!\sum_{k=1}^{N_{T} + 1}  \!\sum_{h=1}^{\lfloor H/2 \rfloor - 1}\!\sum_{\varphi} \sum_{t = \nu_{k-1}}^{\nu_{k} - 1} \!\!\mathbb{E}_{\varphi, \pi}\!\left[ \sum_{j=1}^{d-2} \mathds{1}\left\{\mathrm{sgn}\left(a^{t}_{h}(j)\varphi_{h,k}(j)\right)= -1\right\}  \mathds{1} \left\{s^{t}_{1} = x_{1}, s^{t}_{h} = x_{h} \right\}\right] \nonumber\\
    &= \frac{H\Delta}{5} \!\sum_{k=1}^{N_{T} + 1}  \!\sum_{h=1}^{\lfloor H/2 \rfloor - 1}\sum_{j=1}^{d-2} \sum_{\varphi} \sum_{t = \nu_{k-1}}^{\nu_{k} - 1}\!\mathbb{E}_{\varphi, \pi}\!\left[ \mathds{1}\!\left\{\mathrm{sgn}\left(a^{t}_{h}(j)\varphi_{h,k}(j)\right)= -1, s^{t}_{1} = x_{1}, s^{t}_{h} = x_{h}\right\}\right] \label{eq:minimax_lin_step_1}
\end{align}
where step $(a)$ results from the fact that $\mathbb{P}(s^{t}_{1} = x_{1}, s^{t}_{h} = x_{h}) \leq 1$.

Now, fix $h \in [H]$, $k \in [N_{T} + 1]$, $t \in \{\nu_{k-1}, \dots, \nu_{k} - 1\}$, and $j \in [d-2]$. Also, fix an arbitrary linear MDP instance $\varphi \in ((\{+\Delta, -\Delta\}^{d-2})^{H})^{N_{T}+1}$ and let $\varphi^{(j,h,k)}$ denote the linear MDP instance such that all the elements in $\varphi^{(j,h,k)}$ are the same as those in $\varphi$ except for the $j^{\mathrm{th}}$ coordinate of $\varphi_{h,k}$. We can derive that
\begin{align}
     &\mathbb{E}_{\varphi, \pi}\!\left[ \mathds{1}\left\{\mathrm{sgn}\left(a^{t}_{h}(j)\varphi_{h,k}(j)\right) = -1, s^{t}_{1} = x_{1}, s^{t}_{h} = x_{h}\right\} \right] \nonumber\\
     &\quad+ \mathbb{E}_{\varphi^{(j,h,k)}, \pi}\!\left[ \mathds{1}\left\{\mathrm{sgn}\left(a^{t}_{h}(j)\varphi_{h,k}^{(j,h,k)}(j)\right) = -1, s^{t}_{1} = x_{1}, s^{t}_{h} = x_{h}\right\}\right] \nonumber\\
     &=1 + \mathbb{E}_{\varphi, \pi}\!\left[  \mathds{1}\left\{\mathrm{sgn}\left(a^{t}_{h}(j)\varphi_{h,k}(j)\right) = -1, s^{t}_{1} = x_{1}, s^{t}_{h} = x_{h}\right\} \right] \nonumber\\
     &\quad- \mathbb{E}_{\varphi^{(j,h,k)}, \pi}\!\left[  \mathds{1}\left\{\mathrm{sgn}\left(a^{t}_{h}(j)\varphi_{h,k}(j)\right) = -1, s^{t}_{1} = x_{1}, s^{t}_{h} = x_{h}\right\} \right] \nonumber\\
     &\overset{(a)}{\geq}1 - \mathrm{TV}\left(\mathbb{P}_{\varphi, \pi}, \mathbb{P}_{\varphi^{(j,h,k)}, \pi}\right) \nonumber\\
     &\overset{(b)}{\geq}1 - \sqrt{1/2}\sqrt{D_{\mathrm{KL}}\left(\mathbb{P}_{\varphi, \pi}|| \mathbb{P}_{\varphi^{(j,h,k)}, \pi}\right)} \label{eq:exp_diff} 
\end{align}
where $\mathrm{TV}$ denotes the total variation and $D_{\mathrm{KL}}$ denotes the KL divergence. In step $(a)$, we apply Exercise 14.4, as the indicator function is bounded in $[0,1]$. In step $(b)$, we apply Pinsker's inequality.
Let $P^{t}_{h,\varphi}$ and $P^{t}_{h,\varphi^{(j,h,k)}}$ be the transition kernels of the linear MDP instances $\varphi$, respectively. We compute the KL divergence and obtain:
\begin{align}
&D_{\mathrm{KL}}\left(\mathbb{P}_{\varphi, \pi}\|\mathbb{P}_{\varphi^{(j,h,k)}, \pi}\right) \nonumber\\
&=
\mathbb{E}_{\varphi,\pi}\!\left[
\log\frac{
\prod_{l=1}^{N_T+1}\prod_{t'=\nu_{l-1}}^{\nu_l-1}\prod_{h'=1}^{H}
\pi(a_{h'}^{t'}\mid \zeta_{h'}^{t'})\,P^{t'}_{h',\varphi}(s_{h'+1}^{t'}\mid s_{h'}^{t'},a_{h'}^{t'})
}{
\prod_{l=1}^{N_T+1}\prod_{t'=\nu_{l-1}}^{\nu_l-1}\prod_{h'=1}^{H}
\pi(a_{h'}^{t'}\mid \zeta_{h'}^{t'})\,P^{t}_{h',\varphi^{(j,h,k)}}(s_{h'+1}^{t'}\mid s_{h'}^{t'},a_{h'}^{t'})
} 
\right]\nonumber
\\
&\overset{(a)}{=}
\mathbb{E}_{\varphi,\pi}\!\left[
\sum_{t'=\nu_{k-1}}^{\nu_k-1}
\mathds{1}\left\{s^{t'}_{h} = x_{h}\right\}\log\frac{
P^{t'}_{h,\varphi}\left(s_{h+1}^{t'}\Big| s_{h}^{t'},a_{h}^{t'}\right)
}{
P^{t'}_{h,\varphi^{(j,h,k)}}\left(s_{h+1}^{t'}\Big| s_{h}^{t'},a_{h}^{t'}\right)
}
\right] \nonumber\\
&= \sum_{t'=\nu_{k-1}}^{\nu_k-1}
\mathbb{P}\left(s^{t'}_{h} = x_{h}\right) \mathbb{E}_{\varphi,\pi}\!\left[
\log\frac{
P^{t'}_{h,\varphi}\left(s_{h+1}^{t'}\Big| s_{h}^{t'},a_{h}^{t'}\right)
}{
P^{t'}_{h,\varphi^{(j,h,k)}}\left(s_{h+1}^{t'}\Big| s_{h}^{t'},a_{h}^{t'}\right)
} \Bigg| s^{t'}_{h} = x_{h}
\right] \nonumber\\
&\leq \sum_{t'=\nu_{k-1}}^{\nu_k-1}
\mathbb{E}_{\varphi,\pi}\!\left[
\log\frac{
P^{t'}_{h,\varphi}\left(s_{h+1}^{t'}\Big| s_{h}^{t'},a_{h}^{t'}\right)
}{
P^{t'}_{h,\varphi^{(j,h,k)}}\left(s_{h+1}^{t'}\Big| s_{h}^{t'},a_{h}^{t'}\right)
} \Bigg| s^{t'}_{h} = x_{h}
\right]. \label{eq:sum_exp_log}
\end{align}
%
%In step $(a)$, we leverage the assumption that $\mathbb{P}_{\varphi, \pi}\left(s^{t}_{1} = x_{1}, s^{t}_{h} = x_{h}\right) \geq \mathbb{P}_{\varphi^{(j,h,k)}, \pi}\left(s^{t}_{1} = x_{1}, s^{t}_{h} = x_{h}\right)$. In step $(b)$, we exploit the condition $s^{t}_{1} = x_{1}$ and $s^{t}_{h} = x_{h}$. 
In step $(a)$, we use the fact that all the elements in $\varphi^{(j,h,k)}$ are the same as those in $\varphi$ except for the $j^{\mathrm{th}}$ coordinate of $\varphi_{h,k}$. Let $\mathrm{Bern}(p)$ denote the Bernoulli distribution with parameter $p \in [0,1]$. The expectation inside the summation in \eqref{eq:sum_exp_log} can be upper bounded as follows:
\begin{align}
&\mathbb{E}_{\varphi,\pi}\!\left[
\log\frac{
P^{t'}_{h,\varphi}\left(s_{h+1}^{t'}\Big| s_{h}^{t'},a_{h}^{t'}\right)
}{
P^{t'}_{h,\varphi^{(j,h,k)}}\left(s_{h+1}^{t'}\Big| s_{h}^{t'},a_{h}^{t'}\right)
} \Bigg| s^{t'}_{h} = x_{h}
\right] \nonumber\\
&=\mathbb{E}_{\varphi,\pi}\!\left[\mathbb{E}_{\varphi,\pi}\!\left[
\log\frac{
P^{t'}_{h,\varphi}\left(s_{h+1}^{t'}\Big| s_{h}^{t'},a_{h}^{t'}\right)
}{
P^{t'}_{h,\varphi^{(j,h,k)}}\left(s_{h+1}^{t'}\Big| s_{h}^{t'}, a_{h}^{t'}\right)
} \Bigg| s^{t'}_{h} = x_{h}, a^{t'}_{h} \right]\Bigg| s^{t'}_{h} = x_{h}
\right] \nonumber\\
&=\mathbb{E}_{\varphi,\pi}\!\left[ D_{\mathrm{KL}}\left(\mathrm{Bern}\left(\delta + \left\langle \varphi_{h,k}, a_{h}^{t'} \right\rangle \right), \mathrm{Bern}\left(\delta + \left\langle \varphi_{h,k}^{(j,h,k)}, a_{h}^{t'} \right\rangle \right)\right) \Bigg| s^{t'}_{h} = x_{h}
\right] \nonumber\\
&\overset{(a)}{\leq}\mathbb{E}_{\varphi,\pi}\!\left[ \frac{2\left(\delta + \left\langle \varphi_{h,k}, a_{h}^{t'} \right\rangle - \delta - \left\langle \varphi_{h,k}^{(j,h,k)}, a_{h}^{t'} \right\rangle\right)^{2}}{\delta + \left\langle \varphi_{h,k}^{(j,h,k)}, a_{h}^{t'} \right\rangle} \Bigg| s^{t'}_{h} = x_{h}
\right] \nonumber\\
&=\mathbb{E}_{\varphi,\pi}\!\left[ \frac{2\left(\left\langle \varphi_{h,k} - \varphi_{h,k}^{(j,h,k)}, a_{h}^{t'} \right\rangle \right)^{2}}{\delta + \left\langle \varphi_{h,k}^{(j,h,k)}, a_{h}^{t'} \right\rangle} \Bigg| s^{t'}_{h} = x_{h}
\right] \nonumber\\
&\leq\mathbb{E}_{\varphi,\pi}\!\left[ \frac{8\Delta^{2}}{\delta - (d-2)\Delta} \Bigg| s^{t'}_{h} = x_{h}
\right] \nonumber\\
&\overset{(b)}{\leq} \frac{128\Delta^{2}}{3\delta}. \label{eq:log_exp_results}
\end{align}
In step $(a)$, we exploit the fact that $\delta + \langle \varphi_{h,k}^{(j,h,k)}, a_{h}^{t'} \rangle \leq \delta + (d-2)\Delta \leq 1/2$, as $\delta = 1/H$, $\Delta =\sqrt{\delta/(32\lfloor T/(N_{T} + 1)\rfloor)}$, $H \geq 4$, and $\lfloor T/(N_{T} + 1)\rfloor \geq (d-2)^{2}/(2\delta)$. 
In step $(b)$, we use the fact that $\delta - (d-2)\Delta \geq 3/16$, as $H \geq 4$ and $\lfloor T/(N_{T} + 1)\rfloor \geq (d-2)^{2}/(2\delta)$.
Plugging \eqref{eq:log_exp_results} into \eqref{eq:sum_exp_log}, we have
\begin{align}
D_{\mathrm{KL}}\left(\mathbb{P}_{\varphi, \pi}\|\mathbb{P}_{\varphi^{(j,h,k)}, \pi}\right) \leq \left\lceil \frac{T}{N_{T} + 1} \right\rceil\frac{128\Delta^{2}}{3\delta}.
\end{align}
Then, by plugging this upper bound into \eqref{eq:exp_diff}, we obtain
\begin{align}
    &\mathbb{E}_{\varphi, \pi}\!\left[ \mathds{1}\left\{\mathrm{sgn}\left(a^{t}_{h}(j)\varphi_{h,k}(j)\right) = -1, s^{t}_{1} = x_{1}, s^{t}_{h} = x_{h}\right\} \right] \nonumber\\
    &\quad+ \mathbb{E}_{\varphi^{(j,h,k)}, \pi}\!\left[ \mathds{1}\left\{\mathrm{sgn}\left(a^{t}_{h}(j)\varphi_{h,k}^{(j,h,k)}(j)\right) = -1, s^{t}_{1} = x_{1}, s^{t}_{h} = x_{h}\right\}\right] \nonumber\\
    &\geq 1 - \sqrt{\frac{64\Delta^{2}}{3\delta}\left\lceil \frac{T}{N_{T} + 1} \right\rceil} \nonumber\\
    &\overset{(a)}{\geq} 1 - \sqrt{\frac{3}{4}}
\end{align}
where step $(a)$ stems from the assumption that $T/(N_{T}+1) \geq 8$. Then, by \eqref{eq:minimax_lin_step_1}, we can finally show that
\begin{align}
\sup_{\varphi} \mathcal{R}_{\varphi,\pi} &\geq \frac{H\Delta}{10} \sum_{k=1}^{N_{T} + 1}  \sum_{h=1}^{\lfloor H/2 \rfloor - 1}\sum_{j=1}^{d-2} \sum_{t = \nu_{k-1}}^{\nu_{k} - 1}\left(1 - \sqrt{\frac{3}{4}}\right) \nonumber\\
&\geq \left(1 - \sqrt{\frac{3}{4}}\right)\frac{H\Delta}{10} (N_{T} + 1) \left( \lfloor H/2 \rfloor - 1 \right) (d-2) \left\lfloor \frac{T}{N_{T} + 1} \right\rfloor \nonumber\\
&= \frac{1 - \sqrt{3/4}}{40\sqrt{2}} (d-2) (N_{T} + 1) \left( \lfloor H/2 \rfloor - 1 \right) \sqrt{H\left\lfloor \frac{T}{N_{T} + 1} \right\rfloor}  \nonumber\\
&=\Omega(d\sqrt{H^{3}N_{T}T}).
\end{align}
This completes the proof.
\end{proof}

\subsection{Proof of Theorem \ref{thrm:darling_reg}}
\label{app:darling_reg_proof}
\begin{proof}
Consider a piecewise stationary episodic MDP over $T$ episodes with horizon $H$ and $N_{T}$ change-points. Recall that the $k^{\mathrm{th}}$ change-point is denoted by $\nu_{k}$, and that $\nu_{0} \coloneqq 1$ and $\nu_{N_{T} + 1} \coloneqq T+1$.
Over a stationary segment $\{ \nu_{k-1}, \dots, \nu_{k} - 1\}$, the environment remains
stationary in the sense that there exist transition kernels $\{P_h^{(k)}\}_{h\in[H]}$ and mean reward functions
$\{r_h^{(k)}\}_{h\in[H]}$ such that for all $t\in\{ \nu_{k-1}, \dots, \nu_{k} - 1\}$ and $(s,a,h)\in\mathcal{S}\times\mathcal{A}\times[H]$,
\begin{align}
P_h^t(\cdot\mid s,a)=P_h^{(k)}(\cdot\mid s,a),
\qquad
r_h^t(s,a)=r_h^{(k)}(s,a).
\end{align}
Equivalently, one may denote the stationary MDP over the $k^{\mathrm{th}}$ stationary segment by $(\mathcal{S},\mathcal{A},H,P^{(k)},r^{(k)})$ with
$P^{(k)}=\{P_h^{(k)}\}_{h\in[H]}$ and $r^{(k)}=\{r_h^{(k)}\}_{h\in[H]}$. Let $\tau_{k}$ be the $k^{\mathrm{th}}$ episode at which DARLING restarts, i.e., for $k \in \mathbb{N}$,
\begin{equation}
    \tau_{k} \coloneqq \inf \{ t > \tau_{k-1}: \texttt{Restart} = \texttt{True} \}
\end{equation}
with $\tau_{0} \coloneqq 0$. We then define the following events:
\begin{align}
\mathcal{G}_{k} &\coloneqq\left\{ \forall\,l\in [k-1],\;\tau_{l}\in\left\{ \nu_{l}, \dots, \nu_{l} + \ell_{l} - 1\right\}\right\} \cap \left\{ \tau_{k} > \nu_{k} \right\},\, k \in \left[ N_{T} \right]. \label{eq:good_event_k}
%\\ \mathcal{G} &\coloneqq \mathcal{G}_{N_{T}} \cap \left\{\tau_{N_{T}+1} > \nu_{N_{T} + 1} \right\}.\label{eq:good_event}    
\end{align}
The event $\mathcal{G}_{k}$ represents the ``good event" up to the $k^{\mathrm{th}}$ restart time-step $\mathcal{G}_{k}$ in which the first $k$ changes are detected within the latencies $\ell_{l}$'s.
For notational convenience, we define $\mathcal{G}_{0}$ to be the universal space. In this section, let $\mathbb{E}$ denote the expectation under the probability measure $\mathbb{P}$ induced by executing policy $\pi$ on the PS episodic MDP, and let $\pi^{*}$ be the optimal policy over the piecewise stationary episodic MDP. For brevity and clarity of the notations, we omitted the conditioning on $s_{1}^{t}$'s, as $s_{1}^{t}$ are fixed states chosen by an oblivious adversary. Then, we have the following:
\begin{align}
    &\mathcal{R}(\pi, T) \nonumber\\
    &= \sum_{t=1}^{T} \left( V_{1}^{t,\star}(s_{1}^{t}) - V_{1}^{t,\pi}(s_{1}^{t}) \right) \nonumber\\
    &=\sum_{k = 1}^{N_{T} + 1} \sum_{t=\nu_{k-1}}^{\nu_{k}-1} \mathbb{E}\left[\sum_{h=1}^{H} r^t_{h}\left(s^{t}_{h}, (\pi^{*})_{h}^{t}\left(s^{t}_{h}\right)\right) - r^t_{h}\left(s^{t}_{h}, \pi_{h}^t\left(s^{t}_{h}\right)\right) \right] \nonumber\\ 
    & = \sum_{k = 1}^{N_{T} + 1}  \mathbb{E}\left[\sum_{t=\nu_{k-1}}^{\nu_{k}-1} \sum_{h=1}^{H} r^t_{h}\left(s^{t}_{h}, (\pi^{*})_{h}^{t}\left(s^{t}_{h}\right)\right) - r^t_{h}\left(s^{t}_{h}, \pi_{h}^t\left(s^{t}_{h}\right)\right) \right] \nonumber\\ 
    &=\sum_{k = 1}^{N_{T} + 1} \mathbb{P} \left( \mathcal{G}^{c}_{k} \right) \mathbb{E} \left[ \sum_{t=\nu_{k-1}}^{\nu_{k}-1} \sum_{h=1}^{H} r^t_{h}\left(s^{t}_{h}, (\pi^{*})_{h}^{t}\left(s^{t}_{h}\right)\right) - r^t_{h}\left(s^{t}_{h}, \pi_{h}^t\left(s^{t}_{h}\right)\right) \Bigg| \mathcal{G}^{c}_{k}\right] 
    \nonumber\\
    &\quad\enspace + \sum_{k = 1}^{N_{T} + 1} \mathbb{E} \left[ \mathbf{1}\left\{\mathcal{G}_{k}\right\} \sum_{t=\nu_{k-1}}^{\nu_{k}-1} \sum_{h=1}^{H} r^t_{h}\left(s^{t}_{h}, (\pi^{*})_{h}^{t}\left(s^{t}_{h}\right)\right) - r^t_{h}\left(s^{t}_{h}, \pi_{h}^t\left(s^{t}_{h}\right)\right) \right] \nonumber\\
    &\overset{(a)}{\leq} \sum_{k = 1}^{N_{T} + 1} 
    H\left( \nu_{k} - \nu_{k-1} \right) \mathbb{P} \left( \mathcal{G}^{c}_{k} \right) \nonumber\\
    &\quad\enspace+ \sum_{k = 1}^{N_{T} + 1} \mathbb{E} \left[ \mathbf{1}\left\{\mathcal{G}_{k}\right\} \sum_{t=\nu_{k-1}}^{\nu_{k}-1} \sum_{h=1}^{H} r^t_{h}\left(s^{t}_{h}, (\pi^{*})_{h}^{t}\left(s^{t}_{h}\right)\right) - r^t_{h}\left(s^{t}_{h}, \pi_{h}^t\left(s^{t}_{h}\right)\right)  \right]\label{eq:reg_decomp}
\end{align}
where step $(a)$ follows from the fact that the rewards are bounded in $[0,1]$. Now, define
\begin{align}
\mathcal{E}_{k} &\coloneqq\left\{ \forall\,l\in [k-1],\;\tau_{l}\in\left\{ \nu_{l}, \dots, \nu_{l} + \ell_{l} - 1\right\}\right\},\, k \in \left[ N_{T} \right]. \label{eq:good_event_E_k}
%\\ \mathcal{G} &\coloneqq \mathcal{G}_{N_{T}} \cap \left\{\tau_{N_{T}+1} > \nu_{N_{T} + 1} \right\}.\label{eq:good_event}    
\end{align}
$\mathbb{P}\left(\mathcal{G}^{c}_{k} \right)$ is upper bounded by the following modified union bound, which decomposes the bad event into false alarm events and late detection events:
\begin{align}
    \mathbb{P} \left( \mathcal{G}^{c}_{k} \right) &= \mathbb{P} \left( \left\{ \exists\, l \in [k-1],\; \tau_{l} \notin \left\{ \nu_{l}, \dots, \nu_{l} + \ell_{l} - 1 \right\} \right\} \cup \left\{ \tau_{k} \leq \nu_{k} \right\}  \right) \nonumber\\ 
    &=\sum_{l = 1}^{k-1} \mathbb{P} \left( \tau_{l} \notin \left\{ \nu_{s}, \dots, \nu_{l} + \ell_{l} - 1 \right\}, \mathcal{E}_{l-1} \right)  + \mathbb{P} \left( \tau_{k} \leq \nu_{k}, \mathcal{E}_{k-1} \right)\nonumber\\ 
    &=\sum_{l = 1}^{k-1} \mathbb{P} \left( \mathcal{E}_{l-1} \right) \mathbb{P} \left( \tau_{l} \notin \left\{ \nu_{l}, \dots, \nu_{l} + \ell_{l} - 1 \right\} \big| \mathcal{E}_{l-1} \right) + \mathbb{P} \left( \mathcal{E}_{k-1} \right) \mathbb{P} \left( \tau_{k} \leq \nu_{k} \big| \mathcal{E}_{k-1} \right)\nonumber\\ 
    &\overset{(a)}{\leq} \sum_{l = 1}^{k-1} \mathbb{P} \left( \tau_{l} \notin \left\{ \nu_{l}, \dots, \nu_{l} + \ell_{l} - 1 \right\} \big| \mathcal{E}_{l-1} \right) + \mathbb{P} \left( \tau_{k} \leq \nu_{k} \big| \mathcal{E}_{k-1} \right)\nonumber\\
    &= \sum_{l = 1}^{k} \underbrace{\mathbb{P} \left( \tau_{l} < \nu_{l} \big| \mathcal{E}_{l-1} \right)}_{\Phi_{1}} + \sum_{l = 1}^{k-1} \underbrace{\mathbb{P} \left( \tau_{l} \geq \nu_{l} + \ell_{l} \big| \mathcal{E}_{l-1} \right)}_{\Phi_{2}}\label{eq:bad_union_thm1}
\end{align}
where $(a)$ is due to the fact that $\mathbb{P} \left\{ \mathcal{E}_{k-1} \right\} \leq 1$. We then separately bound $\Phi_{1}$ and $\Phi_{2}$.

\textbf{Upper-Bounding $\Phi_1$.} Recall DARLING illustrated in Algorithm in \ref{alg:DARLING}. Between the restart episodes $\tau_{k-1}$ and $\tau_{k}$, DARLING executes forced probing (change detection) every $\lceil 1/\alpha_{k} \rceil$ rounds, where $\alpha_{k} \in (0,1)$ is the (adaptive) forced probing frequency. For each episode $t>\tau_{k-1}$, if
\begin{align}
(t-\tau_{k-1}-1)\ \mathrm{mod}\ \lceil 1/\alpha_{k} \rceil = 0,
\end{align}
then episode $t$ is a probing episode. Since the probe set is $\mathcal{P}$ is $\mathcal{S}\times\mathcal{A}\times[H]$, the agent chooses an action from the action set $\mathcal{A}$ uniformly at random, add the received reward into the reward history $\mathcal{H}^{(r)}_{(s^{t}_{h}, a^{t}_{h}, h)}$, and add the binary value into the transition history $\mathcal{H}^{(P)}_{(s^{t}_{h}, a^{t}_{h}, h, s')}$ for each $s' \in \mathcal{S}$. Otherwise, the agent runs the stationary RL algorithm $\mathcal{L}$ for that episode. 
For any $h \in [H]$, $s \in \mathcal{S}$, $a \in \mathcal{A}$, and $u \in \mathbb{N}$, we define $t_{(s,a,h), u}$ to be the $u^{\mathrm{th}}$ episode after $\tau_{l-1}$ at which $(s_{h}^{t}, a_{h}^{t})  = (s, a)$ and $(t-\tau_{l-1}-1) \mod \left\lceil 1/\alpha_{l} \right\rceil = 0$, i.e.,
\begin{align}
    t_{(s,a,h), u} \coloneqq \inf \left\{ t > t_{(s,a,h), u-1}: (s_{h}^{t}, a_{h}^{t}) = (s, a), (t-\tau_{l-1}-1) \!\!\!\mod \left\lceil 1/\alpha_{l} \right\rceil = 0 \right\}\label{eq:t'cis}
\end{align}
with $t_{(s,a,h), 0} = \tau_{l-1}$. Then, we define $n_{s,a,h}\left( t\right)$ to be the number of episodes between $\tau_{l-1}+1$ and $t$ at which $(s_{h}^{t}, a_{h}^{t}) = (s, a)$ and $(t-\tau_{l-1}-1) \mod \left\lceil 1/\alpha_{l} \right\rceil = 0$, which is the number of samples obtained due to force exploration and added in the reward history $\mathcal{H}_{(s, a,h)}^{(r)}$ and the transition history $\mathcal{H}^{(P)}_{(s^{t}_{h}, a^{t}_{h}, h, s')}$ given that there are no restarts after $\tau_{l-1}$, i.e.,
\begin{align}
    n_{(s, a, h)}\left( t\right)\coloneqq\sum_{s=\tau_{ l-1}+1}^{t}\mathds{1}\left\{(s_{h}^{t}, a_{h}^{t}) = (s, a), \left( t - \tau_{l-1} -1 \right) \!\!\!\mod \left\lceil 1/\alpha_{l} \right\rceil = 0\right\}.\label{eq:pull}
\end{align}
Recall that $\mathcal{D}$ represents the change detector, which outputs \texttt{True} if $\mathcal{D}$ detects a change. Let $\tau_{(s,a,h)}^{(r)}$ denote the stopping time at which the change detector monitoring $\mathcal{H}_{(s,a,h)}^{(r)}$ declares a change after the $(l-1)^{\mathrm{th}}$ restart episode $\tau_{l-1}$, i.e., 
\begin{align}\label{eq:tau_r}
    \tau_{(s,a,h)}^{(r)} \coloneqq \inf\left\{ u \in \mathbb{N}: \mathcal{D} \left(\mathcal{H}_{(s,a,h)}^{(r)}\right) = \texttt{True}\:\textrm{at episode}\enspace t_{(s,a,h), u} \right\}. 
\end{align}
Similarly, let $\tau_{(s,a,h,s')}^{(P)}$ denote the stopping time at which the change detector monitoring $\mathcal{H}_{(s,a,h,s')}^{(P)}$ declares a change after the $(l-1)^{\mathrm{th}}$ restart episode $\tau_{l-1}$, i.e., 
\begin{align}\label{eq:tau_P}
    \tau_{(s,a,h,s')}^{(P)} \coloneqq \inf\left\{ u \in \mathbb{N}: \mathcal{D} \left(\mathcal{H}_{(s,a,h,s')}^{(P)}\right) = \texttt{True}\:\textrm{at episode}\enspace t_{(s,a,h), u} \right\}. 
\end{align}
Let $\mathbb{P}_{\infty}$ denote the probability measure at which $f_{t} = f_{\nu_{l}}$ for all $t > \nu_{l}$, i.e., the probability measure under which the MDP becomes stationary after the $k^{\mathrm{th}}$ change-point.
Then, for all $l\in [N_{T}+1]$, we have
\begin{align}
\begin{aligned}
    &\mathbb{P} \left( \tau_{l} < \nu_{l} | \mathcal{E}_{l-1} \right) \\
    &= \mathbb{P} ( \{ \exists\, \left( s,a,h \right): s \in \mathcal{S}, a \in \mathcal{A}, h \in [H],\, \tau_{(s,a,h)}^{(r)} \leq n_{(s,a,h)} \left( \nu_{l} - 1 \right) \} \\
    &\quad\quad \cup \{ \exists (s, a, h, s'): s,s' \in \mathcal{S}, a \in \mathcal{A}, h \in [H],\, \tau_{(s,a,h,s')}^{(P)} \leq n_{(s,a,h)} \left( \nu_{l} - 1 \right) \} \big| \mathcal{E}_{l-1} ) \\
    &\overset{(a)}{\leq}\sum_{h=1}^{H} \sum_{s \in \mathcal{S}} \sum_{a \in \mathcal{A}} \mathbb{P} \left( \tau_{(s,a,h)}^{(r)} \leq  n_{(s,a,h)} \left( \nu_{l} - 1 \right) \Big| \mathcal{E}_{l-1} \right) \\
    &\quad\enspace+\sum_{h=1}^{H} \sum_{s \in \mathcal{S}} \sum_{a \in \mathcal{A}} \sum_{s' \in \mathcal{S}} \mathbb{P} \left( \tau_{(s,a,h,s')}^{(P)} \leq n_{(s,a,h)} \left( \nu_{l} - 1 \right) \Big| \mathcal{E}_{l-1} \right) \\
    &\overset{(b)}{\leq}\sum_{h=1}^{H} \sum_{s \in \mathcal{S}} \sum_{a \in \mathcal{A}} \mathbb{P}_{\infty} \left( \tau_{(s,a,h)}^{(r)} \leq T \Big| \mathcal{E}_{l-1} \right) +\sum_{h=1}^{H} \sum_{s \in \mathcal{S}} \sum_{a \in \mathcal{A}} \sum_{s' \in \mathcal{S}} \mathbb{P}_{\infty} \left( \tau_{(s,a,h,s')}^{(P)} \leq T \Big| \mathcal{E}_{l-1} \right)\\
    &\overset{(c)}{\leq}\sum_{h=1}^{H} \sum_{s \in \mathcal{S}} \sum_{a \in \mathcal{A}} \delta_{\mathrm{F}} + \sum_{h=1}^{H} \sum_{s \in \mathcal{S}} \sum_{a \in \mathcal{A}} \sum_{s' \in \mathcal{S}} \delta_{\mathrm{F}}\\
    &=HS(S+1)A \delta_{\mathrm{F}}. \label{eq:false_alarm}
\end{aligned}
\end{align}
where step $(a)$ results from a union bound. Due to the fact that the rewards at step $h$ conditioned on the same state-action pair between $\tau_{l-1}$ and $\nu_{l}$ are i.i.d. given the past event $\mathcal{E}_{l-1}$ (as there are no changes between $\tau_{l-1}$ and $\nu_{l}$), we can change the measure to $\mathbb{P}_{\infty}$ in step $(b)$. Similarly, the next state conditioned on the same current state-action pair between $\tau_{l-1}$ and $\nu_{l}$ are i.i.d. given the past event $\mathcal{E}_{l-1}$, which allows for changing measure to $\mathbb{P}_{\infty}$. In addition, because $n_{(s,a,h)} \left( \nu_{l} - 1 \right) \leq T $, we have $\{ \tau_{(s,a,h)}^{(r)} \leq n_{(s,a,h)} \left( \nu_{l} - 1 \right) \} \subseteq \{ \tau_{(s,a,h)}^{(r)} \leq T \}$ and $\{ \tau_{(s,a,h,s')}^{(P)} \leq n_{(s,a,h)} \left( \nu_{l} - 1 \right) \} \subseteq \{ \tau_{(s,a,h,s')}^{(P)} \leq T \}$. In step $(c)$, we can apply the false alarm probability upper bound for the change detectors in Section \ref{sec:detector}, as the sequence of rewards conditioned on the same state-action pair are i.i.d. sub-Gaussian, and so are the sequence of the $j^{\mathrm{th}}$ entries of the feature vector evaluated at $(s_{h+1}^{t},a')$ conditioned on the same current state-action pair.

\textbf{Upper Bounding $\Phi_2$.} Let $(s^{*}, a^{*}, h^{*})$ be the state-action-step triple at which the mean reward or the transition kernel shifts the most at $\nu_{l}$, i.e., 
\begin{align}
&(s^{*}, a^{*}, h^{*})\nonumber\\
&\!\coloneqq \underset{h\in[H], s \in \mathcal{S}, a \in \mathcal{A}}{\mathrm{argmax}} \max \!\left\{\! \left| r^{(l+1)}_{h} (s, a) - r^{(l)}_{h} (s, a)\right|, \max_{s' \in \mathcal{S}} \!\left\{\! \left| P_{h}^{(l+1)}(s'|s,a) - P_{h}^{(l)}(s'|s,a) \right| \!\right\} \!\!\right\}.
\end{align}
We define the events $\mathcal{M}_{l}$ and $\mathcal{L}_{l}$ as follows:
\begin{align}
    &\mathcal{M}_{l}\coloneqq\left\{ n_{(s^{*},a^{*},h^{*})}(\nu_{l} - 1)\geq m_{\mathcal{D}}\right\},\\
    &\mathcal{L}_{l}\coloneqq\left\{n_{(s^{*},a^{*},h^{*})}(\nu_{l} + \ell_{l} - 1) - n_{(s^{*},a^{*},h^{*})}(\nu_{l} - 1)\geq \ell_{\mathcal{D}}\!\right\}.
\end{align}
When $\tau_{l} \geq \nu_{l} + \ell_{l}$, there are at least $m_{\mathcal{D}}$ reward samples with mean $r^{(l)}_{h^{*}}(s^{*}, a^{*})$ in $\mathcal{H}_{(s^{*}, a^{*}, h^{*})}^{(r)}$ under the event $\mathcal{M}_{l}$, and there are at least $\ell_{\mathcal{D}}$ reward samples with mean $r^{(l)}_{h^{*}}(s^{*}, a^{*})$ in $\mathcal{H}_{(s^{*}, a^{*}, h^{*})}^{(r)}$ under the event $\mathcal{L}_{l}$. Similarly, given that $\tau_{l} \geq \nu_{l} + \ell_{l}$, there are at least $m_{\mathcal{D}}$  samples with mean $P_{h}^{(l)}(s'|s,a)$ in $\mathcal{H}_{(s^{*}, a^{*}, h^{*}, s')}^{(P)}$ under the event $\mathcal{M}_{l}$, and there are at least $\ell_{\mathcal{D}}$ samples with mean $P_{h}^{(l+1)}(s'|s,a)$ in $\mathcal{H}_{(s^{*}, a^{*}, h^{*}, s')}^{(P)}$ for some $s' \in \mathcal{S}$ under the event $\mathcal{L}_{l}$. Then, we have,
\begin{align}
    &\mathbb{P} \left( \tau_{l} \geq \nu_{l} + \ell_{l} \big| \mathcal{E}_{l-1} \right) \nonumber\\
    &\leq\mathbb{P} \left( \left\{ \tau_{l} \geq \nu_{l} + \ell_{l} \right\} \cup \mathcal{M}_{l}^{c} \cup \mathcal{L}_{l}^{c} \big| \mathcal{E}_{l-1} \right) \nonumber\\
    &=\mathbb{P} \left(\mathcal{M}_{l}^{c} \cup \mathcal{L}_{l}^{c} \big| \mathcal{E}_{l-1} \right) + \mathbb{P} \left( \left\{ \tau_{l} \geq \nu_{l} + \ell_{l} \right\} \cap \mathcal{M}_{l} \cap \mathcal{L}_{l} \big| \mathcal{E}_{l-1} \right) \nonumber \\
    &=\mathbb{P} \left(\mathcal{M}_{l}^{c} \cup \mathcal{L}_{l}^{c} \big| \mathcal{E}_{l-1} \right) + \mathbb{P} \left( \mathcal{M}_{l} \cap \mathcal{L}_{l} \big| \mathcal{E}_{l-1} \right) \mathbb{P} \left( \tau_{l} \geq \nu_{l} + \ell_{l} \big| \mathcal{M}_{l} \cap \mathcal{L}_{l} \cap \mathcal{E}_{l-1} \right) \nonumber \\
    &\overset{(a)}{\leq}\mathbb{P} \left(\mathcal{M}_{l}^{c} \big| \mathcal{E}_{l-1} \right) + \mathbb{P} \left(\mathcal{L}_{l}^{c} \big| \mathcal{E}_{l-1} \right) + \mathbb{P} \left( \tau_{l} \geq \nu_{l} + \ell_{l} \big| \mathcal{M}_{l} \cap \mathcal{L}_{l} \cap \mathcal{E}_{l-1} \right) \label{eq:ld_bound}
\end{align}
where step $(a)$ follows from a union bound and the fact that $\mathbb{P} \left( \mathcal{M}_{l} \cap \mathcal{L}_{l} \big| \mathcal{E}_{l-1} \right) \leq 1$. Recall that $n_{(s,a,h)} \left( t \right)$ is the number of episodes between $\tau_{l-1}+1$ and $t$ at which $s_{h}^{t}=s$ and $a_{h}^{t} = a$. Then, we have
\begin{align}
    &\mathbb{E}\left[n_{(s^{*}, a^{*}, h^{*})} \left( \nu_{l} - 1 \right) - n_{(s^{*}, a^{*}, h^{*})} \left( \tau_{l-1} \right) | \mathcal{E}_{l-1} \right] \nonumber\\
    &\overset{(a)}{\geq}\mathbb{E}\left[n_{(s^{*}, a^{*}, h^{*})} \left( \nu_{l} - 1 \right) - n_{(s^{*}, a^{*}, h^{*})} \left( \nu_{l} - m_{l} - 1 \right) | \mathcal{E}_{l-1} \right] \nonumber\\
    &= \mathbb{E}\left[ \sum_{t=\nu_{l} - m_{l}}^{\nu_{l} - 1}\mathds{1}\left\{ (s_{h^{*}}^{t}, a_{h^{*}}^{t}) = (s^{*}, a^{*}), \left( t - \tau_{l-1} -1 \right) \!\!\!\mod \left\lceil 1/\alpha_{l} \right\rceil = 0 \right\} \Bigg| \mathcal{E}_{l-1}\right] \nonumber\\
    &=\sum_{t=\nu_{l} - m_{l}}^{\nu_{l} - 1} \mathbb{P}\left((s_{h^{*}}^{t}, a_{h^{*}}^{t}) = (s^{*}, a^{*}) | \mathcal{E}_{l-1}\right) \mathds{1}\left\{\left( t-\tau_{l}-1\right)\!\!\!\mod\left\lceil1/\alpha_{l}\right\rceil=0\right\}\nonumber\\
    &\overset{(b)}{=} \sum_{t=\nu_{l} - m_{l}}^{\nu_{l} - 1}\mathbb{P}\left((s_{h^{*}}^{t}, a_{h^{*}}^{t}) = (s^{*}, a^{*}) \right)\mathds{1}\left\{ \left( t-\tau_{l}-1\right)\!\!\!\mod\left\lceil1/\alpha_{l}\right\rceil=i-1\right\} \nonumber\\
    &\overset{(c)}{\geq} \frac{p_{\mathrm{m}}}{A} \sum_{t=\nu_{l} - m_{l}}^{\nu_{l} - 1}\mathds{1}\left\{ \left( t-\tau_{l}-1\right)\!\!\!\mod\left\lceil1/\alpha_{l}\right\rceil=i-1\right\} \nonumber\\
    &\overset{(d)}{=} \frac{p_{\mathrm{m}}}{A}\left\lfloor \frac{m_{l}}{\left\lceil 1 / \alpha_{l} \right\rceil} \right\rfloor \nonumber\\
    &= \frac{p_{\mathrm{m}}}{A}\left\lceil \frac{m_{\mathcal{D}}A}{p_{\mathrm{m}}} + \frac{A^{2}\log T}{4p_{\mathrm{m}}^{2}} + \sqrt{\frac{A^{3}m_{\mathcal{D}}\log T}{2p_{\mathrm{m}}^{3}} + \frac{A^{4}(\log T)^{2}}{16p_{\mathrm{m}}^{4}}} \right\rceil, \label{eq:enough_pre_change_1}
\end{align}
and
\begin{align}
    &\mathbb{E}\left[n_{(s^{*}, a^{*}, h^{*})} \left( \nu_{l} + \ell_{l} - 1 \right) - n_{(s^{*}, a^{*}, h^{*})} \left( \nu_{l} - 1 \right) \right] \nonumber\\
    &= \mathbb{E}\left[ \sum_{t=\nu_{l}}^{\nu_{l} + \ell_{l} - 1}\mathds{1}\left\{(s_{h^{*}}^{t}, a_{h^{*}}^{t}) = (s^{*}, a^{*}), \left( t-\tau_{k}-1\right)\!\!\!\mod\left\lceil1/\alpha_{l}\right\rceil=0\right\} \right] \nonumber\\
    &=\sum_{t=\nu_{l}}^{\nu_{l} + \ell_{l} - 1} \mathbb{P}\left((s_{h^{*}}^{t}, a_{h^{*}}^{t}) = (s^{*}, a^{*}) | \mathcal{E}_{l-1}\right) \mathds{1}\left\{\left( t-\tau_{l}-1\right)\!\!\!\mod\left\lceil1/\alpha_{l}\right\rceil=0\right\}\nonumber\\
    &\overset{(e)}{=} \sum_{t=\nu_{l}}^{\nu_{l} + \ell_{l} - 1} \mathbb{P}\left((s_{h^{*}}^{t}, a_{h^{*}}^{t}) = (s^{*}, a^{*}) \right) \mathds{1}\left\{\left( t-\tau_{l}-1\right)\!\!\!\mod\left\lceil1/\alpha_{l}\right\rceil=0\right\}\nonumber\\
    &\overset{(f)}{\geq} \frac{p_{\mathrm{m}}}{A}  \sum_{t=\nu_{l}}^{\nu_{l} + \ell_{l} - 1}\mathds{1}\left\{ \left( t-\tau_{l}-1\right)\!\!\!\mod\left\lceil1/\alpha_{l}\right\rceil=0\right\} \nonumber\\
    &\overset{(g)}{=} \frac{p_{\mathrm{m}}}{A}\left\lfloor \frac{\ell_{l}}{\left\lceil 1/\alpha_{l} \right\rceil} \right\rfloor \nonumber\\
    &= \frac{p_{\mathrm{m}}}{A}\left\lceil \frac{\ell_{\mathcal{D}}A}{p_{\mathrm{m}}} + \frac{A^{2}\log T}{4p_{\mathrm{m}}^{2}} + \sqrt{\frac{A^{3}\ell_{\mathcal{D}}\log T}{2p_{\mathrm{m}}^{3}} + \frac{A^{4}(\log T)^{2}}{16p_{\mathrm{m}}^{4}}} \right\rceil. \label{eq:enough_post_change_1}
\end{align}
In step $(a)$, since $\tau_{l-1} \leq \nu_{l-1} + \ell_{l-1} - 1$ given $\mathcal{E}_{l-1}$ and $\nu_{l} - \nu_{l-1} \geq \ell_{l-1} + m_{l}$ by Assumption \ref{assum:cp_assum}, $\tau_{l-1} \leq \nu_{l} - m_{l} - 1$ and thus $n_{(s^{*}, a^{*}. h^{*})} \left( \nu_{l} - 1 \right) \leq n_{(s^{*}, a^{*}. h^{*})} \left( \nu_{l} - m_{l} - 1 \right)$. 
Steps $(b)$ and $(e)$ follow from the independence between $\{ (s_{h^{*}}^{t}, a_{h^{*}}^{t}) \}_{t>\tau_{l}: (t - \tau_{l} - 1) \!\!\!\mod \lceil 1/\alpha_{l} \rceil = 0}$ and $\mathcal{E}_{l-1}$.  
Steps $(c)$ and $(f)$ stem from the definition of $p_{\mathrm{m}}$ in Assumption \ref{assum:reachability} and the fact that each action in the exploration action set is chosen uniformly at random. 
Steps $(d)$ and $(g)$ result from the fact that $m_{l}$ and $\ell_{l}$ are divisible by $\lceil 1/\alpha_{l} \rceil$. 
Therefore,
\begin{align}
&\mathbb{P} \left(\mathcal{M}_{l}^{c} \big| \mathcal{E}_{l-1} \right) \nonumber\\
&= \mathbb{P} \left( n_{(s^{*},a^{*},h^{*})}(\nu_{l} - 1)< m_{\mathcal{D}} | \mathcal{E}_{l-1} \right) \nonumber\\
&\overset{(a)}{\leq} \exp\left(  \frac{-2\left(\mathbb{E}\left[n_{(s^{*},a^{*},h^{*})}(\nu_{l} - 1) \right] - m_{\mathcal{D}}\right)^{2}}{\sum_{t=\tau_{l}+1}^{\nu_{l}-1}\mathds{1}\left\{\left( t-\tau_{l}-1\right)\!\!\!\mod\!\left\lceil 1/\alpha_l\right\rceil=0\right\}}\right) \nonumber\\
&\overset{(b)}{\leq} \exp\left(  \frac{-2\left( p_{\mathrm{m}}\left\lceil \frac{m_{\mathcal{D}}A}{p_{\mathrm{m}}} + \frac{A^{2}\log T}{4p_{\mathrm{m}}^{2}} + \sqrt{\frac{m_{\mathcal{D}}\log (T)A^{3}}{2p_{\mathrm{m}}^{3}} + \frac{(\log T)^{2}A^{4}}{16p_{\mathrm{m}}^{4}}} \right\rceil - m_{\mathcal{D}} \right)^{2}}{A\left\lceil \frac{m_{\mathcal{D}}A}{p_{\mathrm{m}}} + \frac{A^{2}\log T}{4p_{\mathrm{m}}^{2}} + \sqrt{\frac{m_{\mathcal{D}}\log (T)A^{3}}{2p_{\mathrm{m}}^{3}} + \frac{(\log T)^{2}A^{4}}{16p_{\mathrm{m}}^{4}}} \right\rceil}\right) \nonumber\\
&\leq T^{-1}, \label{eq:M_l_bound}
\end{align}
and
\begin{align}
&\mathbb{P} \left(\mathcal{L}_{l}^{c} \big| \mathcal{E}_{l-1} \right) \nonumber\\
&= \mathbb{P} \left( n_{(s^{*},a^{*},h^{*})}(\nu_{l} + \ell_{l} - 1) - n_{(s^{*},a^{*},h^{*})}(\nu_{l} - 1) < \ell_{\mathcal{D}} | \mathcal{E}_{l-1} \right) \nonumber\\
&\overset{(c)}{\leq} \exp\left(  \frac{-2\left(\mathbb{E}\left[n_{(s^{*},a^{*},h^{*})}(\nu_{l} + \ell_{l} - 1) - n_{(s^{*},a^{*},h^{*})}(\nu_{l} - 1) \right] - \ell_{\mathcal{D}}\right)^{2}}{\sum_{t=\nu_{l}}^{\nu_{l}+\ell_{l}-1}\mathds{1}\left\{\left( t-\tau_{l}-1\right)\!\!\!\mod\!\left\lceil 1/\alpha_l\right\rceil=0\right\}}\right) \nonumber\\
&\overset{(d)}{\leq} \exp\left(  \frac{-2\left( p_{\mathrm{m}}\left\lceil \frac{\ell_{\mathcal{D}}A}{p_{\mathrm{m}}} + \frac{A^{2}\log T}{4p_{\mathrm{m}}^{2}} + \sqrt{\frac{\ell_{\mathcal{D}}\log (T)A^{3}}{2p_{\mathrm{m}}^{3}} + \frac{(\log T)^{2}A^{4}}{16p_{\mathrm{m}}^{4}}} \right\rceil - \ell_{\mathcal{D}} \right)^{2}}{ A\left\lceil \frac{\ell_{\mathcal{D}}A}{p_{\mathrm{m}}} + \frac{A^{2}\log T}{4p_{\mathrm{m}}^{2}} + \sqrt{\frac{\ell_{\mathcal{D}}\log (T)A^{3}}{2p_{\mathrm{m}}^{3}} + \frac{(\log T)^{2}A^{4}}{16p_{\mathrm{m}}^{4}}} \right\rceil}\right) \nonumber\\
&\leq T^{-1}. \label{eq:L_l_bound}
\end{align}
In steps $(a)$ and $(c)$, we apply Hoeffding's inequality, as $\{\mathds{1}\{s_{h^{*}}^{t} = s^{*}, a_{h^{*}}^{t} = a^{*}\}\}_{t>\tau_{l}: (t - \tau_{l} - 1) \!\!\!\mod \lceil 1/\alpha_{l} \rceil = 0}$ is a sequence of i.i.d. Bernoulli random variables with parameter greater than $p_{\mathrm{m}}/A$. In steps $(b)$ and $(d)$, we apply \eqref{eq:enough_post_change_1}.

Before bounding the third term in \eqref{eq:ld_bound}, recall the definitions of the stopping times of the change detectors in \eqref{eq:tau_r} and \eqref{eq:tau_P}. Without loss of generality, we assume that $\nu_{l} \leq T - \ell_{l}$; otherwise, there is no need to detect the change because the horizon will end soon after the change occurs. Let $\Pr_{\nu}$ denote the probability measure whose distribution changes at the $\nu^{\mathrm{th}}$ sample. For the case where $| r^{(l+1)}_{h} (s, a) - r^{(l)}_{h} (s, a)| \geq \max_{s' \in \mathcal{S}} \{ | P_{h}^{(l+1)}(s'|s,a) - P_{h}^{(l)}(s'|s,a)| \}$, we can derive
\begin{align}\nonumber
    &\mathbb{P} \left( \tau_{l} \geq \nu_{l} + \ell_{l} | \mathcal{E}_{l-1} \cap \mathcal{M}_{l} \cap \mathcal{L}_{l} \right)\\\nonumber
    &=\mathbb{P}(\forall\,h \in [H], \forall\,s, s' \in \mathcal{S}, \forall\,a \in \mathcal{A}, \\\nonumber
    &\quad\quad\enspace\tau_{(s,a,h)}^{(r)} > n_{(s,a,h)}\left(\nu_{l} + \ell_{l} - 1 \right), \tau_{(s,a,h,s')}^{(P)} > n_{(s,a,h)}\left(\nu_{l} + \ell_{l} - 1 \right) |\mathcal{E}_{l-1} \cap \mathcal{M}_{l} \cap \mathcal{L}_{l} )\\\nonumber
    &\overset{(a)}{\leq}\mathbb{P}\left(\tau_{(s^{*}, a^{*}, h^{*})}^{(r)} > n_{(s^{*}, a^{*}, h^{*})}\left(\nu_{l} + \ell_{l} - 1 \right)\big|\mathcal{E}_{l-1} \cap \mathcal{M}_{l} \cap \mathcal{L}_{l} \right)\\\nonumber
    &\overset{(b)}{\leq} \mathbb{P}\left(\tau_{(s^{*}, a^{*}, h^{*})}^{(r)} > n_{(s^{*}, a^{*}, h^{*})}\left(\nu_{l} - 1 \right)+\ell_{\mathcal{D}}\big|\mathcal{E}_{l-1} \cap \mathcal{M}_{l} \cap \mathcal{L}_{l}\right) \nonumber\\
    &\overset{(c)}{\leq}\sup_{\nu \in \left\{m_{\mathcal{D}}+1, \dots, T - \ell_{\mathcal{D}}\right\}}\mathbb{P}_{\nu} \left(\tau_{(s^{*}, a^{*}, h^{*})}^{(r)}\geq\nu+\ell_{\mathcal{D}} \big| \mathcal{E}_{l-1} \cap \mathcal{M}_{l} \cap \mathcal{L}_{l}\right)\nonumber\\
    &\overset{(d)}{\leq}\delta_{\mathrm{D}}. \label{eq:late_detect_r}
\end{align}
For the other case where $| r^{(l+1)}_{h} (s, a) - r^{(l)}_{h} (s, a)| < \max_{s' \in \mathcal{S}} \{ | P_{h}^{(l+1)}(s'|s,a) - P_{h}^{(l)}(s'|s,a)| \}$, let $s^{\prime *} \coloneqq \arg\max_{s' \in \mathcal{S}} \{ | P_{h}^{(l+1)}(s'|s,a) - P_{h}^{(l)}(s'|s,a)| \}$. We can similarly obtain
\begin{align}\nonumber
    &\mathbb{P} \left( \tau_{l} \geq \nu_{l} + \ell_{l} | \mathcal{E}_{l-1} \cap \mathcal{M}_{l} \cap \mathcal{L}_{l} \right)\\\nonumber
    &=\mathbb{P}(\forall\,h \in [H], \forall\,s, s' \in \mathcal{S}, \forall\,a \in \mathcal{A}, \\\nonumber
    &\quad\quad\enspace\tau_{(s,a,h)}^{(r)} > n_{(s,a,h)}\left(\nu_{l} + \ell_{l} - 1 \right), \tau_{(s,a,h,s')}^{(P)} > n_{(s,a,h)}\left(\nu_{l} + \ell_{l} - 1 \right) |\mathcal{E}_{l-1} \cap \mathcal{M}_{l} \cap \mathcal{L}_{l} )\\\nonumber
    &\overset{(e)}{\leq}\mathbb{P}\left(\tau_{(s^{*}, a^{*}, h^{*}, s^{\prime *})}^{(P)} > n_{(s^{*}, a^{*}, h^{*})}\left(\nu_{l} + \ell_{l} - 1 \right)\big|\mathcal{E}_{l-1} \cap \mathcal{M}_{l} \cap \mathcal{L}_{l} \right)\\\nonumber
    &\overset{(f)}{\leq} \mathbb{P}\left(\tau_{(s^{*}, a^{*}, h^{*}, s^{\prime *})}^{(P)} > n_{(s^{*}, a^{*}, h^{*})}\left(\nu_{l} - 1 \right)+\ell_{\mathcal{D}}\big|\mathcal{E}_{l-1} \cap \mathcal{M}_{l} \cap \mathcal{L}_{l}\right) \nonumber\\
    &\overset{(g)}{\leq}\sup_{\nu \in \left\{m_{\mathcal{D}}+1, \dots, T - \ell_{\mathcal{D}}\right\}}\mathbb{P}_{\nu} \left(\tau_{(s^{*}, a^{*}, h^{*}, s^{\prime *})}^{(P)}\geq\nu+\ell_{\mathcal{D}} \big| \mathcal{E}_{l-1} \cap \mathcal{M}_{l} \cap \mathcal{L}_{l}\right)\nonumber\\
    &\overset{(h)}{\leq}\delta_{\mathrm{D}}. \label{eq:late_detect_P}
\end{align}
In steps $(a)$ and $(e)$, DARLING restarts at the minimum of the stopping time, leading to the inequalities. 
Steps $(b)$ and $(f)$ stem from the fact that $n_{(s^{*}, a^{*}, h^{*})}\left(\nu_{l} + \ell_{l} - 1 \right) - n_{(s^{*}, a^{*}, h^{*})}\left(\nu_{l} - 1 \right) \geq \ell_{\mathcal{D}}$ given $\mathcal{L}_{l}$. 
Steps $(c)$ and $(g)$ result from the fact that $n_{(s^{*}, a^{*}, h^{*})}\left(\nu_{l} - 1 \right) \geq m_{\mathcal{D}}$ given $\mathcal{M}_{l}$ and $\nu_{l} \leq T - \ell_{l}$. Recall the definition of $t_{(s,a,h),u}$ in \eqref{eq:t'cis}. Step $(d)$ follows from the definition of latency in Section \ref{sec:detector}, as the rewards at step $h^{*}$ conditioned on the state-action pair $(s^{*}, a^{*})$ are independent sub-Gaussian whose distribution changes at $\nu$, given $\mathcal{E}_{l-1}, \mathcal{L}_{l}$, and $\mathcal{M}_{l}$. 
Step $(h)$ also follows from the definition of latency in Section \ref{sec:detector}, as the sequence of the events $\{s^{t}_{h^{*}+1} = s^{\prime *}\}$ conditioned on the current state-action pair $(s^{*}, a^{*})$ are independent sub-Gaussian whose distribution changes at $\nu$, given $\mathcal{E}_{l-1}, \mathcal{L}_{l}$, and $\mathcal{M}_{l}$. Plugging \eqref{eq:M_l_bound}, \eqref{eq:L_l_bound},  \eqref{eq:late_detect_r}, and \eqref{eq:late_detect_P} into \eqref{eq:ld_bound}, we have
\begin{equation}\label{eq:late_detect}
    \mathbb{P} \left( \tau_{l} \geq \nu_{l} + \ell_{l} \big| \mathcal{E}_{l-1} \right) \leq 2T^{-1} + \delta_{\mathrm{D}}.
\end{equation}
This completes bounding $\Phi_1$ and $\Phi_2$. Plugging \eqref{eq:false_alarm} and \eqref{eq:late_detect} into \eqref{eq:bad_union_thm1}, we obtain
\begin{equation}
\mathbb{P}\left\{\mathcal{G}^{c}_{k}\right\}\leq kHS(S+1)A \delta_{\mathrm{F}} + \left( k-1 \right)\left(2T^{-1} + \delta_{\mathrm{D}} \right)
\label{eq:bad_event}.
\end{equation}
This bounds the first term in \eqref{eq:reg_decomp}.

For convenience in bounding the second term in \eqref{eq:reg_decomp}, we define $\bar{\alpha} \coloneqq \max_{k = 1, \dots, N_{T} + 1} \alpha_{k}$. For any $k\in\left[ N_{T} + 1\right]$, if $\left( t-\tau_{k-1}-1\mod\left\lceil 1/\alpha_{k}\right\rceil\right) \neq 0$, then $A_t$ follows the stationary RL algorithm $\mathcal{L}$. Thus, the second term in \eqref{eq:reg_decomp} can then be decomposed as follows:
\begin{align}
    &\mathbb{E} \left[ \mathbf{1}\left\{\mathcal{G}_{k}\right\} \sum_{t=\nu_{k-1}}^{\nu_{k}-1} \sum_{h=1}^{H} r^t_{h}\left(s^{t}_{h}, (\pi^{*})_{h}^{t}\left(s_{h}\right)\right) - r^t_{h}\left(s_{h}, \pi_{h}^t\left(s_{h}\right)\right) \right] \nonumber \\
    &\overset{(a)}{\leq} \ell_{k-1} + \left\lceil\frac{\nu_{k}-\nu_{k-1}}{\left\lceil 1/\alpha_{k}\right\rceil}\right\rceil \nonumber\\
    &\quad\enspace+ \mathbb{E}\left[\mathds{1}\{\mathcal{G}_{k}\}\sum_{t=\tau_{k-1} + 1: \left( t-\tau_{k-1}-1 \right)\!\!\!\mod \lceil 1/\alpha_{k} \rceil \neq 0}^{\nu_{k}-1} \sum_{h=1}^{H} r^t_{h}\left(s_{h}, (\pi^{*})_{h}^{t}\left(s_{h}\right)\right) - r^t_{h}\left(s_{h}, \pi_{h}^t\left(s_{h}\right)\right)\right] \nonumber \\
    &\overset{(b)}{\leq} \ell_{k-1} +  \left[ \alpha_{k} \left( \nu_{k}-\nu_{k-1}\right) + 1 \right] + R_{\mathcal{L}}\left( \nu_{k}-\nu_{k-1}\right) \nonumber \\
    &\leq \ell_{k-1} + \left[ \bar{\alpha} \left( \nu_{k}-\nu_{k-1}\right) + 1 \right] +R_{\mathcal{L}} \left( \nu_{k}-\nu_{k-1} \right)
    \label{eq:bound_inter}
\end{align}
where in step $(a)$, the first term bounds the regret due to the delay of the change detector, and the second term bounds the regret incurred due to probing. In step $(b)$, as the rewards and the trajectories in the history of the  $\mathcal{L}$ are independent of those in $\mathcal{H}_{s,a,h}^{(r)}$ and $\mathcal{H}_{s,a,h,j,a'}^{(P)}$, and that $\mathcal{G}_{k}$ only depends on samples in $\mathcal{H}_{s,a,h}^{(r)}$ and $\mathcal{H}_{s,a,h,j,a'}^{(P)}$, the regret bound of $\mathcal{L}$ applies. 
We also apply the fact that $R_{\mathcal{L}}\left( T \right)$ is increasing with $T$. For the tabular MDP case, we can plug \eqref{eq:bound_inter} and \eqref{eq:bad_event} into \eqref{eq:reg_decomp} and obtain:
\begin{align}
&\mathcal{R}(\pi, T) \nonumber\\
&\leq \sum_{k = 1}^{N_{T} + 1}
    H\left( \nu_{k} - \nu_{k-1} \right) \left(kHd^{2}A \delta_{\mathrm{F}} + \left( k-1 \right)\left(2T^{-1} + \delta_{\mathrm{D}} \right) \right)\nonumber\\
    &\quad\enspace+ \sum_{k = 1}^{N_{T} + 1} \left(  H\ell_{k-1} +  H\left[ \bar{\alpha} \left( \nu_{k}-\nu_{k-1}\right) + 1 \right] +R_{\mathcal{L}} \left( \nu_{k}-\nu_{k-1} \right) \right) \nonumber\\
    &\leq \sum_{k = 1}^{N_{T} + 1}
    H\left( \nu_{k} - \nu_{k-1} \right) \left((N_{T} + 1)Hd^{2}A \delta_{\mathrm{F}} + N_{T}\left(2T^{-1} + \delta_{\mathrm{D}} \right) \right) \nonumber\\
    &\quad\enspace+ \sum_{k = 1}^{N_{T} + 1} \left( H\ell_{k-1} + H\left[ \bar{\alpha} \left( \nu_{k}-\nu_{k-1}\right) + 1 \right] +R_{\mathcal{L}} \left( \nu_{k}-\nu_{k-1} \right) \right) \nonumber\\
    &=T H^{2}S(S+1)A\left( N_{T} + 1\right)\delta_{\mathrm{F}} + 2HN_{T} + THN_{T} \delta_{\mathrm{D}} + H\sum_{k=1}^{N_{T}} \ell_{k} + H\left( \bar{\alpha} T + 1 \right) \nonumber\\
    &\quad+ H\sum_{k = 1}^{N_{T} + 1} R_{\mathcal{L}} \left( \nu_{k}-\nu_{k-1} \right) \nonumber\\
    &\overset{(a)}{\leq} T H^{2}S(S+1)A\left( N_{T} + 1\right)\delta_{\mathrm{F}} + 2N_{T} + T N_{T} H \delta_{\mathrm{D}} + H\sum_{k=1}^{N_{T}} \ell_{k} + H\left( \bar{\alpha} T + 1 \right) \nonumber\\
    &\quad+ \left( N_{T} + 1 \right) R_{\mathcal{L}}\left( \frac{T}{N_{T}+1} \right) \nonumber\\
    &\overset{(b)}{=}\tilde{O}(\sqrt{SAH^{3}TN_{T}}).
\end{align}
In step $(a)$, we apply Jensen's inequality to the concave function $R_{\mathcal{L}}$. In step $(b)$, we use the fact that $\mathcal{R}_{\mathcal{L}}(n) = \tilde{O}(\sqrt{SAH^{3}n})$, $\bar{\alpha} = \tilde{O}(\sqrt{SAHN_{T}/T})$, $\sum_{k=1}^{N_{T}} 1/\alpha_{k} = \tilde{O}(\sqrt{TN_{T}}/SAH)$, and $\delta_{\mathrm{F}} = \delta_{\mathrm{D}} = T^{-\gamma}$ for some $\gamma>1$. This completes the proof.
\end{proof}

\subsection{Proof of Theorem \ref{thrm:darling_reg_lin_mdps}}
\label{app:darling_reg_lin_mdps_proof}
\begin{proof}

The proof proceeds similar to the one for Theorem \ref{thrm:darling_reg}. The main difference is that we need to take the error event of calibration into consideration and add the regret incurred during calibration. Recall that we are considering a linear MDP with $T$ episodes, horizon $H$, and $N_{T}$ change-points. We reintroduce the notations necessary for the proof: The $k^{\mathrm{th}}$ change-point is denoted by $\nu_{k}$, and that $\nu_{0} \coloneqq 1$ and $\nu_{N_{T} + 1} \coloneqq T+1$.
Let $\{P_h^{(k)}\}_{h\in[H]}$ and $\{r_h^{(k)}\}_{h\in[H]}$ denote the transition kernel and the mean reward function over the $k^{\mathrm{th}}$ stationary segment $\{ \nu_{k-1}, \dots, \nu_{k} - 1\}$. Let $\tau_{k}$ be the $k^{\mathrm{th}}$ episode at which DARLING restarts, i.e., for $k \in \mathbb{N}$,
\begin{equation}
    \tau_{k} \coloneqq \inf \{ t > \tau_{k-1}: \texttt{Test 1} \text{ or } \texttt{Test 2} \text{ signals } \texttt{Restart} \}
\end{equation}
with $\tau_{0} \coloneqq 0$. Let $\hat{\mathcal{P}}_{h}^{(k)}$ denote the empirical probe set chosen during the calibration over the $k^{\mathrm{th}}$ stationary segment. We then define the following events for all $k \in \left[ N_{T} \right]$, which is different from the one in the proof of Theorem \ref{thrm:darling_reg}:
\begin{align}
\mathcal{G}_{k} &\coloneqq \left\{ \forall\,l\in [k], \forall\, h \in [H], \mathcal{P}_{h}^{(l)} = \hat{\mathcal{P}}_{h}^{(l)}\right\} \cap \left\{ \forall\,l\in [k-1], \;\tau_{l}\in\left\{ \nu_{l}, \dots, \nu_{l} + \ell_{l} - 1\right\}\right\} \nonumber\\
&\quad\quad\cap \left\{ \tau_{k} > \nu_{k} \right\}. \label{eq:good_event_k_linear}  
\end{align}
The event $\mathcal{G}_{k}$ represents the ``good event" up to the $k^{\mathrm{th}}$ restart time-step $\mathcal{G}_{k}$ in which the first $k$ changes are detected within the latencies $\ell_{l}$'s, and the probe sets $\hat{P}_{h}^{(l)}$ for the first $k$ stationary segments are successfully identified.
For notational convenience, we define $\mathcal{G}_{0}$ to be the universal space. In this section, let $\mathbb{E}$ denote the expectation under the probability measure $\mathbb{P}$ induced by executing policy $\pi$ on the PS episodic linear MDP, and let $\pi^{*}$ be the optimal policy over the piecewise stationary episodic MDP. For brevity and clarity of the notations, we omitted the conditioning on $s_{1}^{t}$'s, as $s_{1}^{t}$ are fixed states chosen by an oblivious adversary. Then, we have the following:
\begin{align}
    &\mathcal{R}(\pi, T) \nonumber\\
    &= \sum_{t=1}^{T} \left( V_{1}^{t,\star}(s_{1}^{t}) - V_{1}^{t,\pi}(s_{1}^{t}) \right) \nonumber\\
    &=\sum_{k = 1}^{N_{T} + 1} \sum_{t=\nu_{k-1}}^{\nu_{k}-1} \mathbb{E}\left[\sum_{h=1}^{H} r^t_{h}\left(s^{t}_{h}, (\pi^{*})_{h}^{t}\left(s^{t}_{h}\right)\right) - r^t_{h}\left(s^{t}_{h}, \pi_{h}^t\left(s^{t}_{h}\right)\right) \right] \nonumber\\ 
    & = \sum_{k = 1}^{N_{T} + 1}  \mathbb{E}\left[\sum_{t=\nu_{k-1}}^{\nu_{k}-1} \sum_{h=1}^{H} r^t_{h}\left(s^{t}_{h}, (\pi^{*})_{h}^{t}\left(s^{t}_{h}\right)\right) - r^t_{h}\left(s^{t}_{h}, \pi_{h}^t\left(s^{t}_{h}\right)\right) \right] \nonumber\\ 
    &=\sum_{k = 1}^{N_{T} + 1} \mathbb{P} \left( \mathcal{G}^{c}_{k} \right) \mathbb{E} \left[ \sum_{t=\nu_{k-1}}^{\nu_{k}-1} \sum_{h=1}^{H} r^t_{h}\left(s^{t}_{h}, (\pi^{*})_{h}^{t}\left(s^{t}_{h}\right)\right) - r^t_{h}\left(s^{t}_{h}, \pi_{h}^t\left(s^{t}_{h}\right)\right) \Bigg| \mathcal{G}^{c}_{k}\right] 
    \nonumber\\
    &\quad\enspace + \sum_{k = 1}^{N_{T} + 1} \mathbb{E} \left[ \mathbf{1}\left\{\mathcal{G}_{k}\right\} \sum_{t=\nu_{k-1}}^{\nu_{k}-1} \sum_{h=1}^{H} r^t_{h}\left(s^{t}_{h}, (\pi^{*})_{h}^{t}\left(s^{t}_{h}\right)\right) - r^t_{h}\left(s^{t}_{h}, \pi_{h}^t\left(s^{t}_{h}\right)\right) \right] \nonumber\\
    &\overset{(a)}{\leq} \sum_{k = 1}^{N_{T} + 1} 
    H\left( \nu_{k} - \nu_{k-1} \right) \mathbb{P} \left( \mathcal{G}^{c}_{k} \right) \nonumber\\
    &\quad\enspace+ \sum_{k = 1}^{N_{T} + 1} \mathbb{E} \left[ \mathbf{1}\left\{\mathcal{G}_{k}\right\} \sum_{t=\nu_{k-1}}^{\nu_{k}-1} \sum_{h=1}^{H} r^t_{h}\left(s^{t}_{h}, (\pi^{*})_{h}^{t}\left(s^{t}_{h}\right)\right) - r^t_{h}\left(s^{t}_{h}, \pi_{h}^t\left(s^{t}_{h}\right)\right)  \right]\label{eq:reg_decomp_linear}
\end{align}
where step $(a)$ follows from the fact that the rewards are bounded in $[0,1]$. Now, define
\begin{align}
\mathcal{E}_{k} &\coloneqq\left\{ \forall\,l\in [k-1], \forall\, h \in [H],\; \mathcal{P}_{h}^{(l)} = \hat{\mathcal{P}}_{h}^{(l)},\;\tau_{l}\in\left\{ \nu_{l}, \dots, \nu_{l} + \ell_{l} - 1\right\}\right\},\, k \in \left[ N_{T} \right]. \label{eq:good_event_E_k_linear}
\end{align}
$\mathbb{P}\left(\mathcal{G}^{c}_{k} \right)$ is upper bounded by the following modified union bound, which decomposes the bad event into false alarm events and late detection events:
\begin{align}
    &\mathbb{P} \left( \mathcal{G}^{c}_{k} \right) \nonumber\\
    &= \mathbb{P} \left(\left\{ \exists\,l\in [k], \exists\, h \in [H], \mathcal{P}_{h}^{(l)} \neq \hat{\mathcal{P}}_{h}^{(l)}\right\} \cup \left\{ \exists\, l \in [k-1],\; \tau_{l} \notin \left\{ \nu_{l}, \dots, \nu_{l} + \ell_{l} - 1 \right\} \right\} \cup \left\{ \tau_{k} \leq \nu_{k} \right\}  \right) \nonumber\\ 
    &\overset{(a)}{\leq}\sum_{l = 1}^{k-1} \mathbb{P} \left( \tau_{l} \notin \left\{ \nu_{s}, \dots, \nu_{l} + \ell_{l} - 1 \right\}, \mathcal{E}_{l-1} \right) + \sum_{l = 1}^{k} \sum_{h=1}^{H} \mathbb{P} \left( \tau_{l} \notin \left\{ \nu_{s}, \dots, \nu_{l} + \ell_{l} - 1 \right\}, \mathcal{E}_{l-1} \right)  \nonumber\\
    &\quad\quad\mathbb{P} \left( \tau_{k} \leq \nu_{k}, \mathcal{E}_{k-1} \right)\nonumber\\ 
    &=\sum_{l = 1}^{k-1} \mathbb{P} \left( \mathcal{E}_{l-1} \right) \mathbb{P} \left( \tau_{l} \notin \left\{ \nu_{l}, \dots, \nu_{l} + \ell_{l} - 1 \right\} \big| \mathcal{E}_{l-1} \right) + \mathbb{P} \left( \mathcal{E}_{k-1} \right) \mathbb{P} \left( \tau_{k} \leq \nu_{k} \big| \mathcal{E}_{k-1} \right)\nonumber\\
    &\quad\quad + \sum_{l = 1}^{k} \sum_{h=1}^{H} \mathbb{P}(\mathcal{E}_{l-1}) \mathbb{P} \left( \tau_{l} \notin \left\{ \nu_{s}, \dots, \nu_{l} + \ell_{l} - 1 \right\}| \mathcal{E}_{l-1} \right)\nonumber\\
    &\overset{(b)}{\leq} \sum_{l = 1}^{k-1} \mathbb{P} \left( \tau_{l} \notin \left\{ \nu_{l}, \dots, \nu_{l} + \ell_{l} - 1 \right\} \big| \mathcal{E}_{l-1} \right) + \mathbb{P} \left( \tau_{k} \leq \nu_{k} \big| \mathcal{E}_{k-1} \right)\nonumber\\
    &\quad\quad + \sum_{l = 1}^{k} \sum_{h=1}^{H}  \mathbb{P} \left( \tau_{l} \notin \left\{ \nu_{s}, \dots, \nu_{l} + \ell_{l} - 1 \right\}| \mathcal{E}_{l-1} \right)\nonumber\\
    &= \sum_{l = 1}^{k} \underbrace{\mathbb{P} \left( \tau_{l} < \nu_{l} \big| \mathcal{E}_{l-1} \right)}_{\Phi_{1}} + \sum_{l = 1}^{k-1} \underbrace{\mathbb{P} \left( \tau_{l} \geq \nu_{l} + \ell_{l} \big| \mathcal{E}_{l-1} \right)}_{\Phi_{2}} \nonumber\\
    &\quad\quad+ \sum_{l = 1}^{k} \sum_{h=1}^{H} \underbrace{ \mathbb{P} \left( \tau_{l} \notin \left\{ \nu_{s}, \dots, \nu_{l} + \ell_{l} - 1 \right\}| \mathcal{E}_{l-1} \right)}_{\Phi_{3}}\label{eq:bad_union_thm1_linear}.
\end{align}
where step $(a)$ is owing to union bound and step $(b)$ is due to the fact that $\mathbb{P} \left\{ \mathcal{E}_{k-1} \right\} \leq 1$. We then separately bound $\Phi_{1}$, $\Phi_{2}$, and $\Phi_{3}$

\textbf{Upper-Bounding $\Phi_1$.} Recall DARLING illustrated in Algorithm in \ref{alg:DARLING_lin_mdps}. Between the restart episodes $\tau_{k-1}$ and $\tau_{k}$, DARLING executes forced probing (change detection) every $\lceil 1/\alpha_{k} \rceil$ rounds, where $\alpha_{k} \in (0,1)$ is (adaptive) forced probing frequency. For each episode $t>\tau_{k-1}$, if
\begin{align}
(t-\tau_{k-1}-1)\ \mathrm{mod}\ \lceil 1/\alpha_{k} \rceil = 0,
\end{align}
then episode $t$ is a probing episode. Since the probe set is $\mathcal{P}$ is $\mathcal{S}\times\mathcal{A}\times[H]$, the agent chooses an action from the action set $\mathcal{A}$ uniformly at random, add the received reward into the reward history $\mathcal{H}^{(r)}_{(s^{t}_{h}, a^{t}_{h}, h)}$, and add the entries of  feature vector $\phi(s^{t}_{h+1},a')$ into the transition history $\mathcal{H}^{(P)}_{(s^{t}_{h}, a^{t}_{h}, h, j, a')}$ for each $j \in [d]$ and $a' \in \mathcal{A}$. Otherwise, the agent runs the stationary RL algorithm $\mathcal{L}$ for that episode. 
For any $h \in [H]$, $(s,a) \in \mathcal{P}_{h}$, and $u \in \mathbb{N}$, we define $t_{(s,a,h), u}$ to be the $u^{\mathrm{th}}$ episode after $\tau_{l-1}$ at which $(s_{h}^{t}, a_{h}^{t})  = (s, a)$ and $(t-\tau_{l-1}-1) \mod \left\lceil 1/\alpha_{l} \right\rceil = 0$, i.e.,
\begin{align}
    t_{(s,a,h), u} \coloneqq \inf \left\{ t > t_{(s,a,h), u-1}: (s_{h}^{t}, a_{h}^{t}) = (s, a), (t-\tau_{l-1}-1) \!\!\!\mod \left\lceil 1/\alpha_{l} \right\rceil = 0 \right\}\label{eq:t'cis_linear}
\end{align}
with $t_{(s,a,h), 0} = \tau_{l-1}$. Then, we define $n_{s,a,h}\left( t\right)$ to be the number of episodes between $\tau_{l-1}+1$ and $t$ at which $(s_{h}^{t}, a_{h}^{t}) = (s, a)$ and $(t-\tau_{l-1}-1) \mod \left\lceil 1/\alpha_{l} \right\rceil = 0$, which is the number of samples obtained due to force exploration and added in the reward history $\mathcal{H}_{(s, a,h)}^{(r)}$ and the transition history $\mathcal{H}^{(P)}_{(s^{t}_{h}, a^{t}_{h}, h, j, a')}$ given that there are no restarts after $\tau_{l-1}$, i.e.,
\begin{align}
    n_{(s, a, h)}\left( t\right)\coloneqq\sum_{s=\tau_{ l-1}+1}^{t}\mathds{1}\left\{(s_{h}^{t}, a_{h}^{t}) = (s, a), \left( t - \tau_{l-1} -1 \right) \!\!\!\mod \left\lceil 1/\alpha_{l} \right\rceil = 0\right\}.\label{eq:pull_linear}
\end{align}
Recall that $\mathcal{D}$ represents the change detector, which outputs \texttt{True} if $\mathcal{D}$ detects a change. Let $\tau_{(s,a,h)}^{(r)}$ denote the stopping time at which the change detector monitoring $\mathcal{H}_{(s,a,h)}^{(r)}$ declares a change after the $(l-1)^{\mathrm{th}}$ restart episode $\tau_{l-1}$, i.e., 
\begin{align}\label{eq:tau_r}
    \tau_{(s,a,h)}^{(r)} \coloneqq \inf\left\{ u \in \mathbb{N}: \mathcal{D} \left(\mathcal{H}_{(s,a,h)}^{(r)}\right) = \texttt{True}\:\textrm{at episode}\enspace t_{(s,a,h), u} \right\}. 
\end{align}
Similarly, let $\tau_{(s,a,h,j,a')}^{(P)}$ denote the stopping time at which the change detector monitoring $\mathcal{H}_{(s,a,h,j,a')}^{(P)}$ declares a change after the $(l-1)^{\mathrm{th}}$ restart episode $\tau_{l-1}$, i.e., 
\begin{align}\label{eq:tau_P_linear}
    \tau_{(s,a,h,j,a')}^{(P)} \coloneqq \inf\left\{ u \in \mathbb{N}: \mathcal{D} \left(\mathcal{H}_{(s,a,h,j,a')}^{(P)}\right) = \texttt{True}\:\textrm{at episode}\enspace t_{(s,a,h), u} \right\}. 
\end{align}
Let $\mathbb{P}_{\infty}$ denote the probability measure at which $f_{t} = f_{\nu_{l}}$ for all $t > \nu_{l}$, i.e., the probability measure under which the MDP becomes stationary after the $k^{\mathrm{th}}$ change-point.
Then, for all $l\in [N_{T}+1]$, we have
\begin{align}
\begin{aligned}
    &\mathbb{P} \left( \tau_{l} < \nu_{l} | \mathcal{E}_{l-1} \right) \\
    &= \mathbb{P} ( \{ \exists \left( s,a,h \right): h \in [H], (s, a) \in \mathcal{P}_{h},\, \tau_{(s,a,h)}^{(r)} \leq n_{(s,a,h)} \left( \nu_{l} - 1 \right) \} \\
    &\quad\quad \cup \{ \exists (s, a, h, j,a'): h \in [H], (s, a) \in \mathcal{P}_{h}, a' \in \mathcal{A}, j \in [d]\, \tau_{(s,a,h,j,a')}^{(P)} \leq n_{(s,a,h)} \left( \nu_{l} - 1 \right) \} \big| \mathcal{E}_{l-1} ) \\
    &\overset{(a)}{\leq}\sum_{h=1}^{H} \sum_{(s,a) \in \mathcal{P}_{h}} \left( \tau_{(s,a,h)}^{(r)} \leq  n_{(s,a,h)} \left( \nu_{l} - 1 \right) \Big| \mathcal{E}_{l-1} \right) \\
    &\quad\enspace+\sum_{h=1}^{H} \sum_{(s,a) \in \mathcal{P}_{h}} \sum_{j=1}^{d} \sum_{a' \in \mathcal{A}} \mathbb{P} \left( \tau_{(s,a,h,j,a')}^{(P)} \leq n_{(s,a,h)} \left( \nu_{l} - 1 \right) \Big| \mathcal{E}_{l-1} \right) \\
    &\overset{(b)}{\leq}\sum_{h=1}^{H} \sum_{(s,a) \in \mathcal{P}_{h}} \mathbb{P}_{\infty} \left( \tau_{(s,a,h)}^{(r)} \leq T \Big| \mathcal{E}_{l-1} \right) +\sum_{h=1}^{H} \sum_{(s,a) \in \mathcal{P}_{h}} \sum_{j=1}^{d} \sum_{a' \in \mathcal{A}} \mathbb{P}_{\infty} \left( \tau_{(s,a,h,j,a')}^{(P)} \leq T \Big| \mathcal{E}_{l-1} \right)\\
    &\overset{(c)}{\leq}\sum_{h=1}^{H} \sum_{(s,a) \in \mathcal{P}_{h}} \sum_{j=1}^{d} \delta_{\mathrm{F}} + \sum_{h=1}^{H} \sum_{(s,a) \in \mathcal{P}_{h}} \sum_{j=1}^{d} \sum_{a' \in \mathcal{A}} \delta_{\mathrm{F}}\\
    &=Hd^{2}(A+1) \delta_{\mathrm{F}}. \label{eq:false_alarm_linear}
\end{aligned}
\end{align}
where step $(a)$ results from a union bound. Due to the fact that the rewards at step $h$ conditioned on the same state-action pair between $\tau_{l-1}$ and $\nu_{l}$ are i.i.d. given the past event $\mathcal{E}_{l-1}$ (as there are no changes between $\tau_{l-1}$ and $\nu_{l}$), we can change the measure to $\mathbb{P}_{\infty}$ in step $(b)$. Similarly, the next state conditioned on the same current state-action pair between $\tau_{l-1}$ and $\nu_{l}$ are i.i.d. given the past event $\mathcal{E}_{l-1}$, which allows for changing measure to $\mathbb{P}_{\infty}$. In addition, because $n_{(s,a,h)} \left( \nu_{l} - 1 \right) \leq T $, we have $\{ \tau_{(s,a,h)}^{(r)} \leq n_{(s,a,h)} \left( \nu_{l} - 1 \right) \} \subseteq \{ \tau_{(s,a,h)}^{(r)} \leq T \}$ and $\{ \tau_{(s,a,h,j,a')}^{(P)} \leq n_{(s,a,h)} \left( \nu_{l} - 1 \right) \} \subseteq \{ \tau_{(s,a,h,j,a')}^{(P)} \leq T \}$. In step $(c)$, we can apply the false alarm probability upper bound for the change detectors in Section \ref{sec:detector}, as the sequence of rewards conditioned on the same state-action pair are i.i.d. sub-Gaussian, and so are the sequence of the $j^{\mathrm{th}}$ entries of the feature vector evaluated at $(s_{h+1}^{t},a')$ conditioned on the same current state-action pair.

\textbf{Upper Bounding $\Phi_2$.} Let $(s^{*}, a^{*}, h^{*})$ be the state-action-step triple at which the mean reward or the transition kernel shifts the most at $\nu_{l}$, i.e., 
\begin{align}
(s^{*}, a^{*}, h^{*}) &\coloneqq \underset{h\in[H], s \in \mathcal{S}_{\mathrm{e},h}, a \in \mathcal{A}_{\mathrm{e},h}^{s}}{\mathrm{argmax}} \max \{ | r^{(l+1)}_{h} (s, a) - r^{(l)}_{h} (s, a)|, \nonumber\\
&\quad\max_{j \in [d], a' \in \mathcal{A}} \{ | \mathbb{E}_{s_{h+1}^{t} \sim P_{h+1}^{(l+1)}} [[\phi (s_{h+1}^{t}, a')]_{j}] - \mathbb{E}_{s_{h+1}^{t} \sim P_{h+1}^{(l)}} [[\phi (s_{h+1}^{t}, a')]_{j}] | \} \}.
\end{align}
We define the events $\mathcal{M}_{l}$ and $\mathcal{L}_{l}$ as follows:
\begin{align}
    &\mathcal{M}_{l}\coloneqq\left\{ n_{(s^{*},a^{*},h^{*})}(\nu_{l} - 1)\geq m_{\mathcal{D}}\right\},\\
    &\mathcal{L}_{l}\coloneqq\left\{n_{(s^{*},a^{*},h^{*})}(\nu_{l} + \ell_{l} - 1) - n_{(s^{*},a^{*},h^{*})}(\nu_{l} - 1)\geq \ell_{\mathcal{D}}\!\right\}.
\end{align}
When $\tau_{l} \geq \nu_{l} + \ell_{l}$, there are at least $m_{\mathcal{D}}$ reward samples with mean $r^{(l)}_{h^{*}}(s^{*}, a^{*})$ in $\mathcal{H}_{(s^{*}, a^{*}, h^{*})}^{(r)}$ under the event $\mathcal{M}_{l}$, and there are at least $\ell_{\mathcal{D}}$ reward samples with mean $r^{(l)}_{h^{*}}(s^{*}, a^{*})$ in $\mathcal{H}_{(s^{*}, a^{*}, h^{*})}^{(r)}$ under the event $\mathcal{L}_{l}$. Similarly, given that $\tau_{l} \geq \nu_{l} + \ell_{l}$, there are at least $m_{\mathcal{D}}$  samples with mean $\mathbb{E}_{s_{h+1}^{t} \sim P_{h+1}^{(l)}} [[\phi (s_{h+1}^{t}, a')]_{j}]$ in $\mathcal{H}_{(s^{*}, a^{*}, h^{*}, j, a')}^{(P)}$ under the event $\mathcal{M}_{l}$, and there are at least $\ell_{\mathcal{D}}$ samples with mean $\mathbb{E}_{s_{h+1}^{t} \sim P_{h+1}^{(l+1)}} [[\phi (s_{h+1}^{t}, a')]_{j}]$ in $\mathcal{H}_{(s^{*}, a^{*}, h^{*}, j, a')}^{(P)}$ under the event $\mathcal{L}_{l}$. Then, we have,
\begin{align}
    &\mathbb{P} \left( \tau_{l} \geq \nu_{l} + \ell_{l} \big| \mathcal{E}_{l-1} \right) \nonumber\\
    &\leq\mathbb{P} \left( \left\{ \tau_{l} \geq \nu_{l} + \ell_{l} \right\} \cup \mathcal{M}_{l}^{c} \cup \mathcal{L}_{l}^{c} \big| \mathcal{E}_{l-1} \right) \nonumber\\
    &=\mathbb{P} \left(\mathcal{M}_{l}^{c} \cup \mathcal{L}_{l}^{c} \big| \mathcal{E}_{l-1} \right) + \mathbb{P} \left( \left\{ \tau_{l} \geq \nu_{l} + \ell_{l} \right\} \cap \mathcal{M}_{l} \cap \mathcal{L}_{l} \big| \mathcal{E}_{l-1} \right) \nonumber \\
    &=\mathbb{P} \left(\mathcal{M}_{l}^{c} \cup \mathcal{L}_{l}^{c} \big| \mathcal{E}_{l-1} \right) + \mathbb{P} \left( \mathcal{M}_{l} \cap \mathcal{L}_{l} \big| \mathcal{E}_{l-1} \right) \mathbb{P} \left( \tau_{l} \geq \nu_{l} + \ell_{l} \big| \mathcal{M}_{l} \cap \mathcal{L}_{l} \cap \mathcal{E}_{l-1} \right) \nonumber \\
    &\overset{(a)}{\leq}\mathbb{P} \left(\mathcal{M}_{l}^{c} \big| \mathcal{E}_{l-1} \right) + \mathbb{P} \left(\mathcal{L}_{l}^{c} \big| \mathcal{E}_{l-1} \right) + \mathbb{P} \left( \tau_{l} \geq \nu_{l} + \ell_{l} \big| \mathcal{M}_{l} \cap \mathcal{L}_{l} \cap \mathcal{E}_{l-1} \right) \label{eq:ld_bound_linear}
\end{align}
where step $(a)$ follows from a union bound and the fact that $\mathbb{P} \left( \mathcal{M}_{l} \cap \mathcal{L}_{l} \big| \mathcal{E}_{l-1} \right) \leq 1$. Recall that $n_{(s,a,h)} \left( t \right)$ is the number of episodes between $\tau_{l-1}+1$ and $t$ at which $s_{h}^{t}=s$ and $a_{h}^{t} = a$, and that $N_{\mathrm{e}} = \max_{h \in [H], s \in \mathcal{S}_{\mathrm{e},h}} N^{s}_{\mathrm{e},h}$ in Definition \ref{def:sep_scaling}. Then, we have
\begin{align}
    &\mathbb{E}\left[n_{(s^{*}, a^{*}, h^{*})} \left( \nu_{l} - 1 \right) - n_{(s^{*}, a^{*}, h^{*})} \left( \tau_{l-1} \right) | \mathcal{E}_{l-1} \right] \nonumber\\
    &\overset{(a)}{\geq}\mathbb{E}\left[n_{(s^{*}, a^{*}, h^{*})} \left( \nu_{l} - 1 \right) - n_{(s^{*}, a^{*}, h^{*})} \left( \nu_{l} - m_{l} - 1 \right) | \mathcal{E}_{l-1} \right] \nonumber\\
    &= \mathbb{E}\left[ \sum_{t=\nu_{l} - m_{l}}^{\nu_{l} - 1}\mathds{1}\left\{ (s_{h^{*}}^{t}, a_{h^{*}}^{t}) = (s^{*}, a^{*}), \left( t - \tau_{l-1} -1 \right) \!\!\!\mod \left\lceil 1/\alpha_{l} \right\rceil = 0 \right\} \Bigg| \mathcal{E}_{l-1}\right] \nonumber\\
    &=\sum_{t=\nu_{l} - m_{l}}^{\nu_{l} - 1} \mathbb{P}\left((s_{h^{*}}^{t}, a_{h^{*}}^{t}) = (s^{*}, a^{*}) | \mathcal{E}_{l-1}\right) \mathds{1}\left\{\left( t-\tau_{l}-1\right)\!\!\!\mod\left\lceil1/\alpha_{l}\right\rceil=0\right\}\nonumber\\
    &\overset{(b)}{=} \sum_{t=\nu_{l} - m_{l}}^{\nu_{l} - 1}\mathbb{P}\left((s_{h^{*}}^{t}, a_{h^{*}}^{t}) = (s^{*}, a^{*}) \right)\mathds{1}\left\{ \left( t-\tau_{l}-1\right)\!\!\!\mod\left\lceil1/\alpha_{l}\right\rceil=i-1\right\} \nonumber\\
    &\overset{(c)}{\geq} \frac{1}{2d} \sum_{t=\nu_{l} - m_{l}}^{\nu_{l} - 1}\mathds{1}\left\{ \left( t-\tau_{l}-1\right)\!\!\!\mod\left\lceil1/\alpha_{l}\right\rceil=i-1\right\} \nonumber\\
    &\overset{(d)}{=} \frac{1}{2d}\left\lfloor \frac{m_{l}}{\left\lceil 1 / \alpha_{l} \right\rceil} \right\rfloor \nonumber\\
    &= \frac{1}{2d}\left\lceil 2dm_{\mathcal{D}} + d^{2}\log T + \sqrt{4d^{3}m_{\mathcal{D}}\log T + d^{4}(\log T)^{2}}\right\rceil, \label{eq:enough_pre_change_1_linear}
\end{align}
and
\begin{align}
    &\mathbb{E}\left[n_{(s^{*}, a^{*}, h^{*})} \left( \nu_{l} + \ell_{l} - 1 \right) - n_{(s^{*}, a^{*}, h^{*})} \left( \nu_{l} - 1 \right) \right] \nonumber\\
    &= \mathbb{E}\left[ \sum_{t=\nu_{l}}^{\nu_{l} + \ell_{l} - 1}\mathds{1}\left\{(s_{h^{*}}^{t}, a_{h^{*}}^{t}) = (s^{*}, a^{*}), \left( t-\tau_{k}-1\right)\!\!\!\mod\left\lceil1/\alpha_{l}\right\rceil=0\right\} \right] \nonumber\\
    &=\sum_{t=\nu_{l}}^{\nu_{l} + \ell_{l} - 1} \mathbb{P}\left((s_{h^{*}}^{t}, a_{h^{*}}^{t}) = (s^{*}, a^{*}) | \mathcal{E}_{l-1}\right) \mathds{1}\left\{\left( t-\tau_{l}-1\right)\!\!\!\mod\left\lceil1/\alpha_{l}\right\rceil=0\right\}\nonumber\\
    &\overset{(e)}{=} \sum_{t=\nu_{l}}^{\nu_{l} + \ell_{l} - 1} \mathbb{P}\left((s_{h^{*}}^{t}, a_{h^{*}}^{t}) = (s^{*}, a^{*}) \right) \mathds{1}\left\{\left( t-\tau_{l}-1\right)\!\!\!\mod\left\lceil1/\alpha_{l}\right\rceil=0\right\}\nonumber\\
    &\overset{(f)}{\geq} \frac{1}{2d}  \sum_{t=\nu_{l}}^{\nu_{l} + \ell_{l} - 1}\mathds{1}\left\{ \left( t-\tau_{l}-1\right)\!\!\!\mod\left\lceil1/\alpha_{l}\right\rceil=0\right\} \nonumber\\
    &\overset{(g)}{=} \frac{1}{2d} \left\lfloor \frac{\ell_{l}}{\left\lceil 1/\alpha_{l} \right\rceil} \right\rfloor \nonumber\\
    &= \frac{1}{2d}\left\lceil 2d\ell_{\mathcal{D}} + d^{2}\log T + \sqrt{4d^{3}\ell_{\mathcal{D}}\log T + d^{4}(\log T)^{2}}\right\rceil. \label{eq:enough_post_change_1_linear}
\end{align}
In step $(a)$, since $\tau_{l-1} \leq \nu_{l-1} + \ell_{l-1} - 1$ given $\mathcal{E}_{l-1}$ and $\nu_{l} - \nu_{l-1} \geq \ell_{l-1} + m_{l}$ by Assumption \ref{assum:cp_assum_lin_mdps}, $\tau_{l-1} \leq \nu_{l} - m_{l} - 1$ and thus $n_{(s^{*}, a^{*}. h^{*})} \left( \nu_{l} - 1 \right) \leq n_{(s^{*}, a^{*}. h^{*})} \left( \nu_{l} - m_{l} - 1 \right)$. 
Steps $(b)$ and $(e)$ follow from the independence between $\{ (s_{h^{*}}^{t}, a_{h^{*}}^{t}) \}_{t>\tau_{l}: (t - \tau_{l} - 1) \!\!\!\mod \lceil 1/\alpha_{l} \rceil = 0}$ and $\mathcal{E}_{l-1}$.  
Steps $(c)$ and $(f)$ stem from the definition of $p_{\mathrm{m}}$ in Assumption \ref{assum:reachability_lin_mdps} and the fact that each action in the exploration action set is chosen uniformly at random. 
Steps $(d)$ and $(g)$ result from the fact that $m_{l}$ and $\ell_{l}$ are divisible by $\lceil 1/\alpha_{l} \rceil$. 
\begin{align}
    &\mathbb{E}\left[n_{(s^{*}, a^{*}, h^{*})} \left( \nu_{l} - 1 \right) - n_{(s^{*}, a^{*}, h^{*})} \left( \tau_{l-1} \right) | \mathcal{E}_{l-1} \right] \nonumber\\
    &\overset{(a)}{\geq}\mathbb{E}\left[n_{(s^{*}, a^{*}, h^{*})} \left( \nu_{l} - 1 \right) - n_{(s^{*}, a^{*}, h^{*})} \left( \nu_{l} - m_{l} - 1 \right) | \mathcal{E}_{l-1} \right] \nonumber\\
    &= \mathbb{E}\left[ \sum_{t=\nu_{l} - m_{l}}^{\nu_{l} - 1}\mathds{1}\left\{ (s_{h^{*}}^{t}, a_{h^{*}}^{t}) = (s^{*}, a^{*}), \left( t - \tau_{l-1} -1 \right) \!\!\!\mod \left\lceil 1/\alpha_{l} \right\rceil = 0 \right\} \Bigg| \mathcal{E}_{l-1}\right] \nonumber\\
    &=\sum_{t=\nu_{l} - m_{l}}^{\nu_{l} - 1} \mathbb{P}\left((s_{h^{*}}^{t}, a_{h^{*}}^{t}) = (s^{*}, a^{*}) | \mathcal{E}_{l-1}\right) \mathds{1}\left\{\left( t-\tau_{l}-1\right)\!\!\!\mod\left\lceil1/\alpha_{l}\right\rceil=0\right\}\nonumber\\
    &\overset{(b)}{=} \sum_{t=\nu_{l} - m_{l}}^{\nu_{l} - 1}\mathbb{P}\left((s_{h^{*}}^{t}, a_{h^{*}}^{t}) = (s^{*}, a^{*}) \right)\mathds{1}\left\{ \left( t-\tau_{l}-1\right)\!\!\!\mod\left\lceil1/\alpha_{l}\right\rceil=i-1\right\} \nonumber\\
    &\overset{(c)}{\geq} \frac{p_{\mathrm{m}}}{N_{\mathrm{e}}} \sum_{t=\nu_{l} - m_{l}}^{\nu_{l} - 1}\mathds{1}\left\{ \left( t-\tau_{l}-1\right)\!\!\!\mod\left\lceil1/\alpha_{l}\right\rceil=i-1\right\} \nonumber\\
    &\overset{(d)}{=} \frac{1}{2d}\left\lfloor \frac{m_{l}}{\left\lceil 1 / \alpha_{l} \right\rceil} \right\rfloor \nonumber\\
    &= \frac{1}{2d}\left\lceil 2dm_{\mathcal{D}} + d^{2}\log T + \sqrt{4d^{3}m_{\mathcal{D}}\log T + d^{4}(\log T)^{2}}\right\rceil, \label{eq:enough_pre_change_1_linear}
\end{align}
and
\begin{align}
    &\mathbb{E}\left[n_{(s^{*}, a^{*}, h^{*})} \left( \nu_{l} + \ell_{l} - 1 \right) - n_{(s^{*}, a^{*}, h^{*})} \left( \nu_{l} - 1 \right) \right] \nonumber\\
    &= \mathbb{E}\left[ \sum_{t=\nu_{l}}^{\nu_{l} + \ell_{l} - 1}\mathds{1}\left\{(s_{h^{*}}^{t}, a_{h^{*}}^{t}) = (s^{*}, a^{*}), \left( t-\tau_{k}-1\right)\!\!\!\mod\left\lceil1/\alpha_{l}\right\rceil=0\right\} \right] \nonumber\\
    &=\sum_{t=\nu_{l}}^{\nu_{l} + \ell_{l} - 1} \mathbb{P}\left((s_{h^{*}}^{t}, a_{h^{*}}^{t}) = (s^{*}, a^{*}) | \mathcal{E}_{l-1}\right) \mathds{1}\left\{\left( t-\tau_{l}-1\right)\!\!\!\mod\left\lceil1/\alpha_{l}\right\rceil=0\right\}\nonumber\\
    &\overset{(e)}{=} \sum_{t=\nu_{l}}^{\nu_{l} + \ell_{l} - 1} \mathbb{P}\left((s_{h^{*}}^{t}, a_{h^{*}}^{t}) = (s^{*}, a^{*}) \right) \mathds{1}\left\{\left( t-\tau_{l}-1\right)\!\!\!\mod\left\lceil1/\alpha_{l}\right\rceil=0\right\}\nonumber\\
    &\overset{(f)}{\geq} \frac{1}{2d}  \sum_{t=\nu_{l}}^{\nu_{l} + \ell_{l} - 1}\mathds{1}\left\{ \left( t-\tau_{l}-1\right)\!\!\!\mod\left\lceil1/\alpha_{l}\right\rceil=0\right\} \nonumber\\
    &\overset{(g)}{=} \frac{1}{2d}\left\lfloor \frac{\ell_{l}}{\left\lceil 1/\alpha_{l} \right\rceil} \right\rfloor \nonumber\\
    &= \frac{1}{2d}\left\lceil 2d\ell_{\mathcal{D}} + d^{2}\log T + \sqrt{4d^{3}\ell_{\mathcal{D}}\log T + d^{4}(\log T)^{2}}\right\rceil. \label{eq:enough_post_change_1_linear}
\end{align}
In step $(a)$, since $\tau_{l-1} \leq \nu_{l-1} + \ell_{l-1} - 1$ given $\mathcal{E}_{l-1}$ and $\nu_{l} - \nu_{l-1} \geq \ell_{l-1} + m_{l}$ by Assumption \ref{assum:cp_assum}, $\tau_{l-1} \leq \nu_{l} - m_{l} - 1$ and thus $n_{(s^{*}, a^{*}. h^{*})} \left( \nu_{l} - 1 \right) \leq n_{(s^{*}, a^{*}. h^{*})} \left( \nu_{l} - m_{l} - 1 \right)$. 
Steps $(b)$ and $(e)$ follow from the independence between $\{ (s_{h^{*}}^{t}, a_{h^{*}}^{t}) \}_{t>\tau_{l}: (t - \tau_{l} - 1) \!\!\!\mod \lceil 1/\alpha_{l} \rceil = 0}$ and $\mathcal{E}_{l-1}$.  
Steps $(c)$ and $(f)$ stem from the definition of $p_{\mathrm{m}}$ in Assumption \ref{assum:reachability} and the fact that each action in the exploration action set is chosen uniformly at random. 
Steps $(d)$ and $(g)$ result from the fact that $m_{l}$ and $\ell_{l}$ are divisible by $\lceil 1/\alpha_{l} \rceil$. 
Therefore,
\begin{align}
&\mathbb{P} \left(\mathcal{M}_{l}^{c} \big| \mathcal{E}_{l-1} \right) \nonumber\\
&= \mathbb{P} \left( n_{(s^{*},a^{*},h^{*})}(\nu_{l} - 1)< m_{\mathcal{D}} | \mathcal{E}_{l-1} \right) \nonumber\\
&\overset{(a)}{\leq} \exp\left(  \frac{-2\left(\mathbb{E}\left[n_{(s^{*},a^{*},h^{*})}(\nu_{l} - 1) \right] - m_{\mathcal{D}}\right)^{2}}{\sum_{t=\tau_{l}+1}^{\nu_{l}-1}\mathds{1}\left\{\left( t-\tau_{l}-1\right)\!\!\!\mod\!\left\lceil 1/\alpha_l\right\rceil=0\right\}}\right) \nonumber\\
&\overset{(b)}{\leq} \exp\left(  \frac{-\left(\left\lceil 2dm_{\mathcal{D}} + d^{2}\log T + \sqrt{4d^{3}m_{\mathcal{D}}\log T + d^{4}(\log T)^{2}}\right\rceil - m_{\mathcal{D}} \right)^{2}}{d\left\lceil 2dm_{\mathcal{D}} + d^{2}\log T + \sqrt{4d^{3}m_{\mathcal{D}}\log T + d^{4}(\log T)^{2}}\right\rceil}\right) \nonumber\\
&\leq T^{-1}, \label{eq:M_l_bound}
\end{align}
and
\begin{align}
&\mathbb{P} \left(\mathcal{L}_{l}^{c} \big| \mathcal{E}_{l-1} \right) \nonumber\\
&= \mathbb{P} \left( n_{(s^{*},a^{*},h^{*})}(\nu_{l} + \ell_{l} - 1) - n_{(s^{*},a^{*},h^{*})}(\nu_{l} - 1) < \ell_{\mathcal{D}} | \mathcal{E}_{l-1} \right) \nonumber\\
&\overset{(c)}{\leq} \exp\left(  \frac{-2\left(\mathbb{E}\left[n_{(s^{*},a^{*},h^{*})}(\nu_{l} + \ell_{l} - 1) - n_{(s^{*},a^{*},h^{*})}(\nu_{l} - 1) \right] - \ell_{\mathcal{D}}\right)^{2}}{\sum_{t=\nu_{l}}^{\nu_{l}+\ell_{l}-1}\mathds{1}\left\{\left( t-\tau_{l}-1\right)\!\!\!\mod\!\left\lceil 1/\alpha_l\right\rceil=0\right\}}\right) \nonumber\\
&\overset{(d)}{\leq} \exp\left(  \frac{-\left( p_{\mathrm{m}}\left\lceil 2d\ell_{\mathcal{D}} + d^{2}\log T + \sqrt{4d^{3}\ell_{\mathcal{D}}\log T + d^{4}(\log T)^{2}}\right\rceil- \ell_{\mathcal{D}} \right)^{2}}{ d\left\lceil 2d\ell_{\mathcal{D}} + d^{2}\log T + \sqrt{4d^{3}\ell_{\mathcal{D}}\log T + d^{4}(\log T)^{2}}\right\rceil}\right) \nonumber\\
&\leq T^{-1}. \label{eq:L_l_bound}
\end{align}
In steps $(a)$ and $(c)$, we apply Hoeffding's inequality, as $\{\mathds{1}\{s_{h^{*}}^{t} = s^{*}, a_{h^{*}}^{t} = a^{*}\}\}_{t>\tau_{l}: (t - \tau_{l} - 1) \!\!\!\mod \lceil 1/\alpha_{l} \rceil = 0}$ is a sequence of i.i.d. Bernoulli random variables with parameter greater than $p_{\mathrm{m}}/N_{\mathrm{e}}$. In steps $(b)$ and $(d)$, we apply \eqref{eq:enough_post_change_1}.

Before bounding the third term in \eqref{eq:ld_bound}, recall the definitions of the stopping times of the change detectors in \eqref{eq:tau_r} and \eqref{eq:tau_P}. Without loss of generality, we assume that $\nu_{l} \leq T - \ell_{l}$; otherwise, there is no need to detect the change because the horizon will end soon after the change occurs. Let $\Pr_{\nu}$ denote the probability measure whose distribution changes at the $\nu^{\mathrm{th}}$ sample. For the case where $| r^{(l+1)}_{h} (s, a) - r^{(l)}_{h} (s, a)| \geq \max_{s' \in \mathcal{S}} \{ | P_{h}^{(l+1)}(s'|s,a) - P_{h}^{(l)}(s'|s,a)| \}$, we can derive
\begin{align}\nonumber
    &\mathbb{P} \left( \tau_{l} \geq \nu_{l} + \ell_{l} | \mathcal{E}_{l-1} \cap \mathcal{M}_{l} \cap \mathcal{L}_{l} \right)\\\nonumber
    &=\mathbb{P}(\forall\,h \in [H], \forall\,s, s' \in \mathcal{S}, \forall\,a \in \mathcal{A}, \\\nonumber
    &\quad\quad\enspace\tau_{(s,a,h)}^{(r)} > n_{(s,a,h)}\left(\nu_{l} + \ell_{l} - 1 \right), \tau_{(s,a,h,s')}^{(P)} > n_{(s,a,h)}\left(\nu_{l} + \ell_{l} - 1 \right) |\mathcal{E}_{l-1} \cap \mathcal{M}_{l} \cap \mathcal{L}_{l} )\\\nonumber
    &\overset{(a)}{\leq}\mathbb{P}\left(\tau_{(s^{*}, a^{*}, h^{*})}^{(r)} > n_{(s^{*}, a^{*}, h^{*})}\left(\nu_{l} + \ell_{l} - 1 \right)\big|\mathcal{E}_{l-1} \cap \mathcal{M}_{l} \cap \mathcal{L}_{l} \right)\\\nonumber
    &\overset{(b)}{\leq} \mathbb{P}\left(\tau_{(s^{*}, a^{*}, h^{*})}^{(r)} > n_{(s^{*}, a^{*}, h^{*})}\left(\nu_{l} - 1 \right)+\ell_{\mathcal{D}}\big|\mathcal{E}_{l-1} \cap \mathcal{M}_{l} \cap \mathcal{L}_{l}\right) \nonumber\\
    &\overset{(c)}{\leq}\sup_{\nu \in \left\{m_{\mathcal{D}}+1, \dots, T - \ell_{\mathcal{D}}\right\}}\mathbb{P}_{\nu} \left(\tau_{(s^{*}, a^{*}, h^{*})}^{(r)}\geq\nu+\ell_{\mathcal{D}} \big| \mathcal{E}_{l-1} \cap \mathcal{M}_{l} \cap \mathcal{L}_{l}\right)\nonumber\\
    &\overset{(d)}{\leq}\delta_{\mathrm{D}}. \label{eq:late_detect_r}
\end{align}
For the other case where $| r^{(l+1)}_{h} (s, a) - r^{(l)}_{h} (s, a)| < \max_{s' \in \mathcal{S}} \{ | P_{h}^{(l+1)}(s'|s,a) - P_{h}^{(l)}(s'|s,a)| \}$, let $s^{\prime *} \coloneqq \arg\max_{s' \in \mathcal{S}} \{ | P_{h}^{(l+1)}(s'|s,a) - P_{h}^{(l)}(s'|s,a)| \}$. We can similarly obtain
\begin{align}\nonumber
    &\mathbb{P} \left( \tau_{l} \geq \nu_{l} + \ell_{l} | \mathcal{E}_{l-1} \cap \mathcal{M}_{l} \cap \mathcal{L}_{l} \right)\\\nonumber
    &=\mathbb{P}(\forall\,h \in [H], \forall\,s, s' \in \mathcal{S}, \forall\,a \in \mathcal{A}, \\\nonumber
    &\quad\quad\enspace\tau_{(s,a,h)}^{(r)} > n_{(s,a,h)}\left(\nu_{l} + \ell_{l} - 1 \right), \tau_{(s,a,h,s')}^{(P)} > n_{(s,a,h)}\left(\nu_{l} + \ell_{l} - 1 \right) |\mathcal{E}_{l-1} \cap \mathcal{M}_{l} \cap \mathcal{L}_{l} )\\\nonumber
    &\overset{(e)}{\leq}\mathbb{P}\left(\tau_{(s^{*}, a^{*}, h^{*}, s^{\prime *})}^{(P)} > n_{(s^{*}, a^{*}, h^{*})}\left(\nu_{l} + \ell_{l} - 1 \right)\big|\mathcal{E}_{l-1} \cap \mathcal{M}_{l} \cap \mathcal{L}_{l} \right)\\\nonumber
    &\overset{(f)}{\leq} \mathbb{P}\left(\tau_{(s^{*}, a^{*}, h^{*}, s^{\prime *})}^{(P)} > n_{(s^{*}, a^{*}, h^{*})}\left(\nu_{l} - 1 \right)+\ell_{\mathcal{D}}\big|\mathcal{E}_{l-1} \cap \mathcal{M}_{l} \cap \mathcal{L}_{l}\right) \nonumber\\
    &\overset{(g)}{\leq}\sup_{\nu \in \left\{m_{\mathcal{D}}+1, \dots, T - \ell_{\mathcal{D}}\right\}}\mathbb{P}_{\nu} \left(\tau_{(s^{*}, a^{*}, h^{*}, s^{\prime *})}^{(P)}\geq\nu+\ell_{\mathcal{D}} \big| \mathcal{E}_{l-1} \cap \mathcal{M}_{l} \cap \mathcal{L}_{l}\right)\nonumber\\
    &\overset{(h)}{\leq}\delta_{\mathrm{D}}. \label{eq:late_detect_P}
\end{align}
In steps $(a)$ and $(e)$, DARLING restarts at the minimum of the stopping time, leading to the inequalities. 
Steps $(b)$ and $(f)$ stem from the fact that $n_{(s^{*}, a^{*}, h^{*})}\left(\nu_{l} + \ell_{l} - 1 \right) - n_{(s^{*}, a^{*}, h^{*})}\left(\nu_{l} - 1 \right) \geq \ell_{\mathcal{D}}$ given $\mathcal{L}_{l}$. 
Steps $(c)$ and $(g)$ result from the fact that $n_{(s^{*}, a^{*}, h^{*})}\left(\nu_{l} - 1 \right) \geq m_{\mathcal{D}}$ given $\mathcal{M}_{l}$ and $\nu_{l} \leq T - \ell_{l}$. Recall the definition of $t_{(s,a,h),u}$ in \eqref{eq:t'cis}. Step $(d)$ follows from the definition of latency in Section \ref{sec:detector}, as the rewards at step $h^{*}$ conditioned on the state-action pair $(s^{*}, a^{*})$ are independent sub-Gaussian whose distribution changes at $\nu$, given $\mathcal{E}_{l-1}, \mathcal{L}_{l}$, and $\mathcal{M}_{l}$. 
Step $(h)$ also follows from the definition of latency in Section \ref{sec:detector}, as the sequence of the events $\{s^{t}_{h^{*}+1} = s^{\prime *}\}$ conditioned on the current state-action pair $(s^{*}, a^{*})$ are independent sub-Gaussian whose distribution changes at $\nu$, given $\mathcal{E}_{l-1}, \mathcal{L}_{l}$, and $\mathcal{M}_{l}$. Plugging \eqref{eq:M_l_bound}, \eqref{eq:L_l_bound},  \eqref{eq:late_detect_r}, and \eqref{eq:late_detect_P} into \eqref{eq:ld_bound}, we have
\begin{equation}\label{eq:late_detect}
    \mathbb{P} \left( \tau_{l} \geq \nu_{l} + \ell_{l} \big| \mathcal{E}_{l-1} \right) \leq 2T^{-1} + \delta_{\mathrm{D}}.
\end{equation}

\textbf{Upper Bounding $\Phi_{3}$.}
Fix $h\in[H]$ and let $\delta_{\mathrm{cal}} = T^{-1}$, $p_{\mathrm{m}} = 1/2d$, and
\[
P_{h,t}^{\star}=\{(s_1^{\star},a_1^{\star}),\dots,(s_d^{\star},a_d^{\star})\}\in\mathfrak{B}_h
\]
be the slice from Assumption~\ref{assum:reachability_lin_mdps}. For each $i\in[d]$,
\[
q_i^{\star}:=q_{h,t}(s_i^{\star},a_i^{\star})=\frac{p^{\pi_{\mathrm U}}_{h,t}(s_i^{\star})}{A}\ge \frac{p_m}{A}.
\]

\paragraph{Step 1: the good slice has empirical count at least $\beta$.}
For each $i$, the count $\widehat n_h(s_i^{\star},a_i^{\star})$ is binomial with mean
\[
\mu_i = n_0 q_i^{\star} \ge \frac{n_0p_m}{A} = 2\beta.
\]
Then we have that,
\[
\mathbb{P}\!\left(\widehat n_h(s_i^{\star},a_i^{\star})<\beta\right)
\le
\mathbb{P}\!\left(\widehat n_h(s_i^{\star},a_i^{\star})<\frac{\mu_i}{2}\right)
\le e^{-\mu_i/8}
\le e^{-\beta/4}.
\]
Since $n_0\ge \frac{16A}{p_m}L$, we have $\beta\ge 8L$, hence $e^{-\beta/4}\le e^{-2L}$. Therefore
\[
\mathbb{P}\!\left(\exists i\in[d],\exists h\in[H]: \widehat n_h(s_i^{\star},a_i^{\star})<\beta\right)
\le dHe^{-2L}
\le \frac{\delta_{\mathrm{cal}}}{2}.
\]
Call this event $\mathcal E_1$.

\paragraph{Step 2: no low-occupancy pair can reach count $\beta$.}
Define the bad set
\[
\mathcal B_h:=\Bigl\{(s,a): q_{h,t}(s,a)<\frac{p_m}{8A}\Bigr\}.
\]
Partition it into dyadic bins
\[
\mathcal B_{h,m}:=\Bigl\{(s,a): 2^{-(m+1)}\frac{p_m}{8A}\le q_{h,t}(s,a)<2^{-m}\frac{p_m}{8A}\Bigr\},\qquad m\ge 0.
\]
Because $\sum_{(s,a)}q_{h,t}(s,a)=1$, every pair in $\mathcal B_{h,m}$ has mass at least $2^{-(m+1)}p_m/(8A)$, so
\[
|\mathcal B_{h,m}|\le \frac{2^{m+4}A}{p_m}.
\]
Fix $(s,a)\in\mathcal B_{h,m}$ and let $X_{s,a,h}:=\widehat n_h(s,a)\sim \mathrm{Bin}(n_0,q_{h,t}(s,a))$. Its mean satisfies
\[
\mu_{s,a,h}=n_0q_{h,t}(s,a)<2^{-m}\frac{n_0p_m}{8A}=2^{-m}\frac{\beta}{4}.
\]
Using the standard binomial upper-tail bound
\[
\mathbb{P}(X\ge \beta)\le \left(\frac{e\mu}{\beta}\right)^{\beta},
\]
we get
\[
\mathbb{P}\!\left(X_{s,a,h}\ge \beta\right)
\le
\left(\frac{e}{2^{m+2}}\right)^{\beta}.
\]
Therefore, for fixed $h$,
\begin{align*}
\mathbb{P}\!\left(\exists (s,a)\in\mathcal B_h: \widehat n_h(s,a)\ge \beta\right)
&\le \sum_{m=0}^{\infty}|\mathcal B_{h,m}|\left(\frac{e}{2^{m+2}}\right)^{\beta} \\
&\le \frac{16A}{p_m}\left(\frac{e}{4}\right)^{\beta}\sum_{m=0}^{\infty}2^{m(1-\beta)}.
\end{align*}
Since $\beta\ge 8L\ge 2$, the geometric series is at most $2$, so
\[
\mathbb{P}\!\left(\exists (s,a)\in\mathcal B_h: \widehat n_h(s,a)\ge \beta\right)
\le \frac{32A}{p_m}\left(\frac{e}{4}\right)^{\beta}.
\]
Also, $\log(4/e)>1/3$, hence $(e/4)^{\beta}\le e^{-\beta/3}\le e^{-8L/3}$. Using the definition of $L$,
\[
\frac{32A}{p_m}e^{-8L/3}
\le \frac{\delta_{\mathrm{cal}}}{2H}.
\]
Taking a union bound over $h\in[H]$ yields an event $\mathcal E_2$ with
\[
\mathbb{P}(\mathcal E_2^c)\le \frac{\delta_{\mathrm{cal}}}{2},
\]
on which every pair with empirical count at least $\beta$ satisfies $q_{h,t}(s,a)\ge p_m/(8A)$.

\paragraph{Step 3: the greedy slice is complete and well reachable.}
Assume $\mathcal E_1\cap\mathcal E_2$ holds. Fix $h$. After $j-1<d$ greedy selections, the current span has dimension $j-1$, so some pair in $P_{h,t}^{\star}$ still lies outside that span. By $\mathcal E_1$, that pair has empirical count at least $\beta$. Since the greedy rule picks the highest-count pair outside the current span, the $j$-th selected pair also has count at least $\beta$. By $\mathcal E_2$, that selected pair satisfies $q_{h,t}(s,a)\ge p_m/(8A)$.

Repeating for $j=1,\dots,d$ shows that $|\widehat P_h|=d$ and every selected pair in $\widehat P_h$ satisfies
\[
q_{h,t}(s,a)\ge \frac{p_m}{8A},
\qquad
p^{\pi_{\mathrm U}}_{h,t}(s)\ge \frac{p_m}{8}.
\]
Finally,
\[
\mathbb{P}((\mathcal E_1\cap\mathcal E_2)^c)\le \mathbb{P}(\mathcal E_1^c)+\mathbb{P}(\mathcal E_2^c)\le \delta_{\mathrm{cal}},
\]
which completes the proof.

This completes bounding $\Phi_1$ and $\Phi_2$. Plugging \eqref{eq:false_alarm} and \eqref{eq:late_detect} into \eqref{eq:bad_union_thm1}, we obtain
\begin{equation}
\mathbb{P}\left\{\mathcal{G}^{c}_{k}\right\}\leq kHS(S+1)A \delta_{\mathrm{F}} + \left( k-1 \right)\left(2T^{-1} + \delta_{\mathrm{D}} \right)
\label{eq:bad_event}.
\end{equation}
This bounds the first term in \eqref{eq:reg_decomp}.

For convenience in bounding the second term in \eqref{eq:reg_decomp}, we define $\bar{\alpha} \coloneqq \max_{k = 1, \dots, N_{T} + 1} \alpha_{k}$. For any $k\in\left[ N_{T} + 1\right]$, if $\left( t-\tau_{k-1}-1\mod\left\lceil 1/\alpha_{k}\right\rceil\right) \neq 0$, then $A_t$ follows the stationary RL algorithm $\mathcal{L}$. Thus, the second term in \eqref{eq:reg_decomp} can then be decomposed as follows:
\begin{align}
    &\mathbb{E} \left[ \mathbf{1}\left\{\mathcal{G}_{k}\right\} \sum_{t=\nu_{k-1}}^{\nu_{k}-1} \sum_{h=1}^{H} r^t_{h}\left(s^{t}_{h}, (\pi^{*})_{h}^{t}\left(s_{h}\right)\right) - r^t_{h}\left(s_{h}, \pi_{h}^t\left(s_{h}\right)\right) \right] \nonumber \\
    &\overset{(a)}{\leq} n_{0} + \ell_{k-1} + \left\lceil\frac{\nu_{k}-\nu_{k-1}}{\left\lceil 1/\alpha_{k}\right\rceil}\right\rceil \nonumber\\
    &\quad\enspace+ \mathbb{E}\left[\mathds{1}\{\mathcal{G}_{k}\}\sum_{t=\tau_{k-1} + 1: \left( t-\tau_{k-1}-1 \right)\!\!\!\mod \lceil 1/\alpha_{k} \rceil \neq 0}^{\nu_{k}-1} \sum_{h=1}^{H} r^t_{h}\left(s_{h}, (\pi^{*})_{h}^{t}\left(s_{h}\right)\right) - r^t_{h}\left(s_{h}, \pi_{h}^t\left(s_{h}\right)\right)\right] \nonumber \\
    &\overset{(b)}{\leq} n_{0} + \ell_{k-1} +  \left[ \alpha_{k} \left( \nu_{k}-\nu_{k-1}\right) + 1 \right] + R_{\mathcal{L}}\left( \nu_{k}-\nu_{k-1}\right) \nonumber \\
    &\leq n_{0} + \ell_{k-1} + \left[ \bar{\alpha} \left( \nu_{k}-\nu_{k-1}\right) + 1 \right] +R_{\mathcal{L}} \left( \nu_{k}-\nu_{k-1} \right)
    \label{eq:bound_inter}
\end{align}
where in step $(a)$, the first term bounds the regret due to the delay of the change detector, and the second term bounds the regret incurred due to probing. In step $(b)$, as the rewards and the trajectories in the history of the  $\mathcal{L}$ are independent of those in $\mathcal{H}_{s,a,h}^{(r)}$ and $\mathcal{H}_{s,a,h,j,a'}^{(P)}$, and that $\mathcal{G}_{k}$ only depends on samples in $\mathcal{H}_{s,a,h}^{(r)}$ and $\mathcal{H}_{s,a,h,j,a'}^{(P)}$, the regret bound of $\mathcal{L}$ applies. 
We also apply the fact that $R_{\mathcal{L}}\left( T \right)$ is increasing with $T$. For the tabular MDP case, we can plug \eqref{eq:bound_inter} and \eqref{eq:bad_event} into \eqref{eq:reg_decomp} and obtain:
\begin{align}
&\mathcal{R}(\pi, T) \nonumber\\
&\leq \sum_{k = 1}^{N_{T} + 1}
    H\left( \nu_{k} - \nu_{k-1} \right) \left(kHd^{2}A \delta_{\mathrm{F}} + \left( k-1 \right)\left(2T^{-1} + \delta_{\mathrm{D}} \right) \right)\nonumber\\
    &\quad\enspace+ \sum_{k = 1}^{N_{T} + 1} \left(  H\ell_{k-1} +  H\left[ \bar{\alpha} \left( \nu_{k}-\nu_{k-1}\right) + 1 \right] +R_{\mathcal{L}} \left( \nu_{k}-\nu_{k-1} \right) \right) \nonumber\\
    &\leq \sum_{k = 1}^{N_{T} + 1}
    H\left( \nu_{k} - \nu_{k-1} \right) \left((N_{T} + 1)Hd^{2}A \delta_{\mathrm{F}} + N_{T}\left(2T^{-1} + \delta_{\mathrm{D}} \right) \right) \nonumber\\
    &\quad\enspace+ \sum_{k = 1}^{N_{T} + 1} \left( H\ell_{k-1} + H\left[ \bar{\alpha} \left( \nu_{k}-\nu_{k-1}\right) + 1 \right] +R_{\mathcal{L}} \left( \nu_{k}-\nu_{k-1} \right) \right) \nonumber\\
    &=T H^{2}d^{2}A\left( N_{T} + 1\right)\delta_{\mathrm{F}} + 2HN_{T} + THN_{T} \delta_{\mathrm{D}}  + H\sum_{k=1}^{N_{T}} \ell_{k} + H\left( \bar{\alpha} T + 1 \right) \nonumber\\
    &\quad+ H\sum_{k = 1}^{N_{T} + 1} R_{\mathcal{L}} \left( \nu_{k}-\nu_{k-1} \right) \nonumber\\
    &\overset{(a)}{\leq} T H^{2}d^{2}A\left( N_{T} + 1\right)\delta_{\mathrm{F}} + 2HN_{T} + THN_{T} \delta_{\mathrm{D}}  + H\sum_{k=1}^{N_{T}} \ell_{k} + H\left( \bar{\alpha} T + 1 \right) \nonumber\\
    &\quad+ \left( N_{T} + 1 \right) R_{\mathcal{L}}\left( \frac{T}{N_{T}+1} \right) \nonumber\\
    &\overset{(b)}{=}\tilde{O}(d\sqrt{H^{3}TN_{T}}).
\end{align}
In step $(a)$, we apply Jensen's inequality to the concave function $R_{\mathcal{L}}$. In step $(b)$, we use the fact that $\mathcal{R}_{\mathcal{L}}(T) = \tilde{O}(d\sqrt{H^{3}T})$, $\bar{\alpha} = \tilde{O}(\sqrt{dN_{T}/T})$, $\sum_{k=1}^{N_{T}} 1/\alpha_{k} = \tilde{O}(\sqrt{TN_{T}}/d)$, and $\delta_{\mathrm{F}} = \delta_{\mathrm{D}} = T^{-\gamma}$ for some $\gamma>1$. This completes the proof.

\end{proof}

\section{Experimental Details}

\subsection{General Formulations of GLR and GSR}
\label{app:glr-gsr}

\begin{figure*}[h]
\centering
\begin{minipage}{0.48\textwidth}
\begin{algorithm}[H]
\caption{\textbf{G}eneralized \textbf{L}ikelihood \textbf{R}atio Test}
\label{alg:GLR}
\begin{algorithmic}[1]
\small
\STATE {\bf Input}: history $\mathcal{H}=\{X_1,\dots,X_n\}$, $\delta_\mathrm{F}$, $\delta_\mathrm{D}$, divergence $\mathrm{kl}(\cdot,\cdot)$
\FOR{$t = 1$ to $n-1$}
    \STATE Compute $\hat{\mu}_{1:t}$, $\hat{\mu}_{t+1:n}$, $\hat{\mu}_{1:n}$
    \STATE $\mathrm{GLR}_t \leftarrow t\,\mathrm{kl}(\hat{\mu}_{1:t}, \hat{\mu}_{1:n}) + (n-t)\,\mathrm{kl}(\hat{\mu}_{t+1:n}, \hat{\mu}_{1:n})$
    \STATE\textbf{if} $\mathrm{GLR}_t \ge \beta_{\mathrm{GLR}}(n,\delta_{\mathrm{F}})$ \textbf{then return} \texttt{True}
\ENDFOR
\end{algorithmic}
\end{algorithm}
\end{minipage}
\hfill
\begin{minipage}{0.48\textwidth}
\begin{algorithm}[H]
\caption{\textbf{G}eneralized \textbf{S}hiryaev--\textbf{R}oberts Test}
\label{alg:GSR}
\begin{algorithmic}[1]
\small
\STATE {\bf Input}: history $\mathcal{H}=\{X_1,\dots,X_n\}$, $\delta_\mathrm{F}$, $\delta_\mathrm{D}$, divergence $\mathrm{kl}(\cdot,\cdot)$, $\mathrm{GSR}\leftarrow 0$
\FOR{$t = 1$ to $n-1$}
    \STATE Compute $\hat{\mu}_{1:t}$, $\hat{\mu}_{t+1:n}$, $\hat{\mu}_{1:n}$
    %\STATE $\mathrm{GLR}_t \leftarrow t\,\mathrm{kl}(\hat{\mu}_{1:t}, \hat{\mu}_{1:n}) + (n-t)\,\mathrm{kl}(\hat{\mu}_{t+1:n}, \hat{\mu}_{1:n})$
    \STATE $\mathrm{GSR} \leftarrow \mathrm{GSR} + \exp(\mathrm{GLR}_t)$\hfill(Alg \ref{alg:GLR})
    \STATE\textbf{if} $\log(\mathrm{GSR}) \ge \beta_{\mathrm{GSR}}(n,\delta_{\mathrm{F}})+\log n$ \textbf{then return} \texttt{True}
\ENDFOR
\end{algorithmic}
\end{algorithm}
\end{minipage}
\end{figure*}

For completeness, we summarize the general forms of the Generalized Likelihood Ratio (GLR) and Generalized Shiryaev--Roberts (GSR) tests for sequential change detection.
Let $(X_i)_{i\ge1}$ be a sequence of real-valued observations generated from a parametric family $\{f_\theta:\theta\in\mathbb{R}\}$.
Both tests compare the no-change hypothesis (a single parameter $\theta$ for all samples) against the single change-point alternative (parameters $\theta_0$ before the change and $\theta_1$ after).

\paragraph{GLR test.}
The GLR stopping time is defined as
\begin{align*}
\tau_{\mathrm{GLR}}
\;\coloneqq\;
\inf\Big\{n\in\mathbb{N}: G_n \ge \beta(n,\delta_{\mathrm{F}})\Big\},
\end{align*}
where the GLR statistic is
\begin{align*}
G_n
\;\coloneqq\;
\sup_{t\in[n]}\;
\log\!\left(
\frac{\ \sup_{\theta_0\in\mathbb{R}}\sup_{\theta_1\in\mathbb{R}}
\prod_{i=1}^{t} f_{\theta_0}(X_i)\prod_{i=t+1}^{n} f_{\theta_1}(X_i)\ }
{\ \sup_{\theta\in\mathbb{R}}\prod_{i=1}^{n} f_{\theta}(X_i)\ }
\right).
\end{align*}

\paragraph{GSR test.}
The GSR stopping time is
\begin{align*}
\tau_{\mathrm{GSR}}
\;\coloneqq\;
\inf\Big\{n\in\mathbb{N}:\ \log W_n \ge \beta(n,\delta_{\mathrm{F}})+\log n\Big\},
\end{align*}
with statistic
\begin{align*}
W_n
\;\coloneqq\;
\frac{1}{n}\sum_{t=1}^{n}
\left(
\frac{\ \sup_{\theta_0\in\mathbb{R}}\sup_{\theta_1\in\mathbb{R}}
\prod_{i=1}^{t} f_{\theta_0}(X_i)\prod_{i=t+1}^{n} f_{\theta_1}(X_i)\ }
{\ \sup_{\theta\in\mathbb{R}}\prod_{i=1}^{n} f_{\theta}(X_i)\ }
\right).
\end{align*}

\paragraph{Anytime-valid threshold.}
In the general case, for any target false-alarm level $\delta_{\mathrm{F}}\in(0,1)$, we use the threshold
\begin{equation}
\label{eq:GLR_beta}
\beta(n,\delta_{\mathrm{F}})
=
6\log\!\big(1+\log n\big)
+\frac{5}{2}\log\!\left(\frac{4n^{3/2}}{\delta_{\mathrm{F}}}\right)
+11.
\end{equation}

\paragraph{Empirical-mean (Bernoulli / Gaussian) specialization.}
For the families used in our experiments and implementations (Algorithms~\ref{alg:GLR}--\ref{alg:GSR}), following \cite{besson2022efficient,huang2025cdbppsmabs}, the log-likelihood ratio at a candidate split $t\in\{1,\ldots,n-1\}$ admits the closed form
\begin{align}
\label{eq:lr-kl}
&\log\!\left(
\frac{\ \sup_{\theta_0\in\mathbb{R}}\prod_{i=1}^{t} f_{\theta_0}(X_i)\;\sup_{\theta_1\in\mathbb{R}}\prod_{i=t+1}^{n} f_{\theta_1}(X_i)\ }
{\ \sup_{\theta\in\mathbb{R}}\prod_{i=1}^{n} f_{\theta}(X_i)\ }
\right)
\nonumber\\
&\qquad=
t\,\mathrm{kl}\!\big(\hat{\mu}_{1:t},\hat{\mu}_{1:n}\big)
+(n-t)\,\mathrm{kl}\!\big(\hat{\mu}_{t+1:n},\hat{\mu}_{1:n}\big),
\end{align}
where $\hat{\mu}_{a:b}$ denotes the empirical mean of $\{X_a,\ldots,X_b\}$.
For sub-Bernoulli observations we use
\[
\mathrm{kl}(x,y)=x\ln\!\frac{x}{y}+(1-x)\ln\!\frac{1-x}{1-y},
\]
and for $\sigma^2$-sub-Gaussian observations (Gaussian mean-shift proxy),
\[
\mathrm{kl}(x,y)=\frac{(x-y)^2}{2\sigma^2}.
\]

In our experiments, we use the Bernoulli variants for sub-Bernoulli rewards (with $\mathrm{kl}$ as above). For $\sigma^2$-sub-Gaussian observations we use the Gaussian proxy divergence $\mathrm{kl}(x,y)=(x-y)^2/(2\sigma^2)$.

\subsection{Experimental Environments}
\label{app:exp_envs}

\subsubsection{Tabular MDP Environments}

\paragraph{Bidirectional Diabolical Combination Lock.}
We follow the Bidirectional Diabolical Combination Lock construction \cite{mao2022restartqucb} in which each episode starts from a fixed initial state. The first action routes the agent into one of two ``locks'' (paths), each a chain of length $H$; along each chain, at every step there is a \emph{unique} correct action that advances to the next state on that path, whereas any of the other $A-1$ actions sends the agent to an absorbing sinking state. The MDP is mildly stochastic: even when the agent selects the correct action, the intended transition succeeds with probability $0.98$, and with probability $0.02$ the agent still falls into the sink. Rewards are sparse on the optimal behavior: taking correct actions yields reward $0$, and the only large reward occurs at the endpoint of a path (one endpoint gives reward $1$, the other gives $0.25$). By contrast, entering the sink yields a small reward of $\frac{1}{8H}$ at the transition step and then a per-step reward of $\frac{1}{8H}$ thereafter, making the sink a tempting locally-optimal attractor. 

\emph{Non-stationarity.} For drifting experiments we use the original gradual protocol in \cite{mao2022restartqucb}: the routing dynamics at the initial state are linearly morphed over time so that an action that initially reaches path~1 with probability $0.98$ (and path~2 with probability $0.02$) is gradually transformed to reach path~1 with probability $0.02$ (and path~2 with probability $0.98$), with the symmetric change applied to the other action. For abrupt PS experiments, we use the same endpoint-swap mechanism as \cite{mao2022restartqucb}, but replace their fixed-period switching with a geometric change-point model \cite{gerogiannis2024blackboxfeas}: segment lengths are i.i.d.\ geometric with parameter $T^{-\xi}$ for $\xi\in\{0.4,0.6,0.8\}$ over a horizon of $T=50000$ episodes, and at each change-point the two endpoints swap identities (the $1$ and $0.25$ rewards exchange). Unless otherwise stated, we use the benchmark parameterization \cite{mao2022restartqucb} $H=5$, $S=10$, $A=2$, and report cumulative reward averaged over multiple random seeds.

\paragraph{DeepSea.}
We use a compact finite-horizon DeepSea-style exploration benchmark inspired by the DeepSea task in the Behaviour Suite \cite{osband2020deepsea}. Each episode starts from a fixed initial state, and the state records the agent's current depth along a sparse-reward chain. There are two actions. At every depth, one action is the ``correct'' action and advances the agent one level deeper with probability $0.98$, whereas the other action advances with probability $0.02$; when the advance fails, the agent is reset to the initial state. Rewards are sparse: in the stationary template the agent receives reward only at the terminal step after successfully reaching the deepest state. Thus, high reward requires repeatedly selecting the correct action across the entire horizon, while mistakes destroy progress and force the agent to begin again from the start of the chain. Unless otherwise stated, we use $H=5$, $S=11$, and $A=2$.

\emph{Non-stationarity.} For abrupt PS experiments, we alternate between two DeepSea templates. In one phase, action $1$ is the correct action at every depth; in the other phase, action $0$ is the correct action at every depth. Instead of switching after fixed windows, segment lengths are drawn from the same geometric change-point model used throughout the paper \cite{gerogiannis2024blackboxfeas}, with parameter $T^{-\xi}$ for $\xi\in\{0.4,0.6,0.8\}$ over $T=50000$ episodes. For drifting experiments, we keep the transition dynamics fixed and drift only the terminal reward structure: the terminal reward associated with action $1$ is smoothly decreased from $1$ to $0.25$, while the terminal reward associated with action $0$ is smoothly increased from $0.25$ to $1$. The interpolation uses a smooth monotone schedule over the full run, so the gradual experiment changes the reward landscape without changing which actions advance the chain.

\paragraph{FourRoom.}
We use a finite-horizon FourRoom gridworld based on the classical four-room navigation domain \cite{sutton1999fourroom}. The state space is a $7\times 7$ grid with a cross-shaped wall dividing the grid into four rooms and four doorways connecting adjacent rooms. Removing the wall cells leaves $S=40$ reachable states. Each episode begins from a fixed start state in the lower doorway, and the agent has four actions corresponding to left, right, up, and down. The transition dynamics are mildly stochastic: with probability $0.95$ the intended action is executed, and the remaining probability is spread uniformly over the other feasible primitive actions. Invalid moves leave the agent in place. There are two absorbing goal states in the upper-left and upper-right rooms. In each stationary phase, one goal has reward $1$ and the other has reward $0.25$, making the environment a sparse navigation problem in which the best target changes over time. Unless otherwise stated, we use $H=10$, $S=40$, and $A=4$.

\emph{Non-stationarity.} For abrupt PS experiments, the transition kernel is fixed and only the goal rewards change. At each change-point, the two goal identities are swapped: the goal with reward $1$ becomes the goal with reward $0.25$, and vice versa. Change-points are generated by the same geometric model \cite{gerogiannis2024blackboxfeas}, with segment parameter $T^{-\xi}$ for $\xi\in\{0.4,0.6,0.8\}$ over $T=50000$ episodes. For drifting experiments, the transition kernel again remains stationary, while the two goal rewards are smoothly interpolated between the two endpoint reward maps over the run. Thus the optimal goal gradually moves from one room to the other without modifying the navigation dynamics.

\paragraph{NRoom.}
We use a finite-horizon NRoom navigation benchmark following the room-chain layout used in the rlberry research environments \cite{rlberry2021nroom}. The environment consists of $5$ rooms of side length $4$ arranged in a chain, with walls separating neighboring rooms and doorways connecting the rooms. The resulting grid has $S=84$ reachable states and $A=4$ primitive actions corresponding to left, right, down, and up. The agent starts in the middle room. There is an easy reward state near the beginning of the room chain and a terminal reward state in the final room; the terminal state is absorbing. The dynamics are stochastic but mostly controlled: with probability $0.95$ the requested action is executed, while the remaining probability is assigned to the other valid neighboring moves; invalid moves keep the agent in place. The reward at the start state is a small background value $0.01$. In the first phase, the terminal state gives reward $1$ and the easy state gives reward $0.1$; in the second phase, the terminal state gives reward $0.6$ and the easy state gives reward $1$. This creates a benchmark in which the agent must trade off a nearby high reward and a longer-horizon terminal reward whose value changes over time. Unless otherwise stated, we use $H=18$, $S=84$, and $A=4$.

\emph{Non-stationarity.} For abrupt PS experiments, the room layout and transition kernel are fixed, and the rewards at the easy and terminal states switch according to the geometric change-point model \cite{gerogiannis2024blackboxfeas}. In particular, at each change-point the easy state becomes the high-reward target and the terminal state drops to its floor value, or conversely the terminal state becomes optimal again. For drifting experiments, we keep the transitions fixed and linearly interpolate the reward map from the first phase to the second phase across the full run. The drifting case therefore gradually changes the relative attractiveness of the easy and terminal targets while preserving the same navigation problem.

\paragraph{Forked RiverSwim.}
We use a finite-horizon Forked RiverSwim benchmark inspired by the hard-exploration Forked RiverSwim construction in \cite{russo2023forkedriverswim}. The MDP has a shared root state and two RiverSwim branches. With $N=6$ states per branch including the shared root, the total number of states is $S=2N-1=11$. There are three actions: a left action, a right action, and a switch action. At the root, the right action enters the first branch, the switch action enters the second branch, and the left action stays at the root. Away from the root, the left action deterministically moves one step back toward the root, the switch action moves to the corresponding depth on the other branch, and the right action is stochastic: it moves forward with probability $0.35$, stays in place with probability $0.55$, and slips backward with probability $0.10$. The root gives a small reward $0.05$ for taking the left action, while large rewards are only available by reaching a branch endpoint and taking the right action. Unless otherwise stated, we use $H=12$, $S=11$, and $A=3$.

\emph{Non-stationarity.} For abrupt PS experiments, the transition kernel is fixed and the reward identities of the two branch endpoints are swapped at geometric change-points \cite{gerogiannis2024blackboxfeas}. In one phase, the endpoint of the first branch gives reward $1$ and the endpoint of the second branch gives reward $0.95$; in the other phase, these rewards are exchanged. For drifting experiments, the same endpoint rewards are interpolated linearly over the run, while the transition dynamics remain stationary. This creates a non-stationary hard-exploration problem in which the two long branches have similar rewards, but the identity of the slightly better branch changes over time.

\subsubsection{Linear MDP Environments}

\paragraph{Synthetic Chain Combination Lock.}
We follow the synthetic linear-MDP ``chain lock'' construction \cite{zhou2022restartlsviucb} with $S=15$ states, $A=7$ actions, horizon $H=10$, feature dimension $d=10$, and $5$ special candidate chains. The MDP is linear: transitions factor as
$P^{(k)}_h(s'\mid s,a)=\langle \phi(s,a),\mu_{h,k}(s')\rangle$,
where the known feature map $\phi:\mathcal{S}\times\mathcal{A}\to\mathbb{R}^d$ is \emph{one-hot}, so each $(s,a)$ deterministically selects a latent index in $[d]$. The feature map is constructed so that for each special chain index $i\in\{1,\dots,5\}$, there is a designated ``correct'' action $a_i$ at state $s_i$ with $\phi(s_i,a_i)=e_i$; taking any other action at $s_i$ maps to a random latent coordinate in $[d]\setminus\{i\}$ (uniformly). For all remaining (``normal'') states $s_i$ with $i\ge 6$, every action maps to a uniformly random latent coordinate in $[d]$. Given this $\phi$, the vectors $\mu_{h,k}$ are chosen so that in each episode there is exactly one \emph{connected} special chain $g\in[5]$ that behaves like a combination lock: the latent index $g$ induces transitions that keep the agent on the corresponding chain with high probability (e.g., $0.99$ vs.\ $0.01$), while the other special chains are ``broken'' by reversing these probabilities; the remaining (normal) latent indices transition to randomly chosen states (e.g., a $0.8/0.2$ split between two random next states). Rewards are also linear, $r^{(k)}_h(s,a)=\langle \phi(s,a),\theta_{h,k}\rangle$: for the good-chain coordinate $g$, the reward is $0$ for steps $h\le H-1$ and $1$ at the terminal step $h=H$, whereas all other coordinates receive small dense rewards (e.g., i.i.d.\ in $[0.005,0.008]$), again creating a strong local optimum away from the rare terminal reward.

\emph{Non-stationarity.} For drifting experiments we use the original gradual protocol in \cite{zhou2022restartlsviucb}, which continuously shifts the identity of the good chain by interpolating (via convex combinations) between successive base MDPs over fixed windows (e.g., $100$ episodes). For abrupt PS experiments, we use the same abrupt switching mechanism as \cite{zhou2022restartlsviucb}, the identity of the good chain $g\in[5]$ changes at change-points, but we draw segment lengths i.i.d.\ from the same geometric change-point model used in the tabular benchmark (parameter $T^{-\xi}$ with $\xi\in\{0.4,0.6,0.8\}$ over $T=50000$ episodes). Performance is reported as cumulative reward averaged over multiple random seeds.

\paragraph{Simplex.}
We construct an exact finite-horizon linear MDP in the standard linear factorization model \cite{jin2020lsviucb,zhou2022restartlsviucb}. The feature map $\phi:\mathcal{S}\times\mathcal{A}\to\mathbb{R}^d$ is dense: for every state-action pair $(s,a)$, $\phi(s,a)$ is sampled from the probability simplex over $d$ latent coordinates. Thus every action mixes several latent transition and reward components rather than selecting a single coordinate. For each base MDP and each horizon step, the latent transition measures $\mu_h(\cdot,j)$ are independently sampled from a simplex over the $S$ states, so transitions are dense but exactly linear in $\phi(s,a)$. Rewards are also linear. At nonterminal steps all latent reward coordinates are small, while at the terminal step one latent coordinate associated with the current base MDP is assigned reward $1$ and the remaining coordinates receive smaller background rewards. Unless otherwise stated, we use $S=25$, $A=10$, $H=10$, $d=10$, and $5$ base MDPs.

\emph{Non-stationarity.} The non-stationarity follows the same base-MDP protocol as in the synthetic linear chain benchmark \cite{zhou2022restartlsviucb}. For abrupt PS experiments, the active base MDP changes at geometric change-points with parameter $N_{\mathrm{ep}}^{-\xi}$ for $\xi\in\{0.4,0.6,0.8\}$ over $N_{\mathrm{ep}}=50000$ episodes; at every change-point, the active base index advances cyclically through the $5$ base MDPs. For drifting experiments, we interpolate by convex combinations between consecutive base MDPs over windows of $100$ episodes. Both the linear rewards $\theta_{h,k}$ and linear transition measures $\mu_{h,k}$ drift under this interpolation, so the optimal latent coordinate changes gradually over time.

\paragraph{Linear GARNET.}
We use a structured sparse-mixture linear MDP inspired by GARNET-style random MDP benchmarks \cite{archibald1995garnet,bhatnagar2009garnet} and implemented within the linear MDP model. Each state-action feature vector has only a small number of active latent coordinates. One deterministic anchor coordinate is given by $(s+3a)\bmod d$, and one additional latent coordinate is sampled at random; the two active coordinates are assigned simplex weights. For each base MDP, each latent coordinate transitions only to a sparse support of $5$ next states, with transition probabilities sampled from a simplex over that support. Rewards are small at nonterminal steps, while at the terminal step the latent coordinate associated with the current base MDP receives reward $1$ and the remaining coordinates receive smaller background rewards. Unless otherwise stated, we use $S=25$, $A=10$, $H=10$, $d=10$, branching factor $5$, and $5$ base MDPs.

\emph{Non-stationarity.} For abrupt PS experiments, the sparse transition supports and terminal reward coordinate are changed by switching among the $5$ base MDPs at geometric change-points \cite{gerogiannis2024blackboxfeas}. The change-point parameter is $T^{-\xi}$ for $\xi\in\{0.4,0.6,0.8\}$ over $T=50000$ episodes. For drifting experiments, the environment follows the gradual linear-MDP protocol of \cite{zhou2022restartlsviucb}: over each $100$-episode window, both the latent reward vectors and the latent transition measures are convexly interpolated from one base MDP to the next. This makes the sparse GARNET structure non-stationary while preserving the exact linear factorization.

\paragraph{Anchor.}
We construct an exact anchor-feature linear MDP, following the anchor/separability viewpoint commonly used in linear MDP analysis \cite{yang2019anchor}. The first $d$ state-action pairs are explicit anchors: each one has feature vector equal to a standard basis vector $e_j$. All remaining state-action features are convex combinations of the anchors, biased toward a deterministic anchor coordinate $(s+2a)\bmod d$ with weight $0.85$ and completed by a small random simplex component. This creates a feature geometry in which a subset of state-action pairs directly identifies the latent coordinates, while all other pairs are mixtures of those anchors. For each base MDP, the transition measure for the good anchor is concentrated toward a moving target state, while the other latent transition measures are dense simplex distributions. Rewards are small except near the terminal step, where the good anchor receives reward $1$ and a neighboring decoy anchor receives reward $0.15$. Unless otherwise stated, we use $S=20$, $A=5$, $H=10$, $d=10$, and $5$ base MDPs.

\emph{Non-stationarity.} For abrupt PS experiments, the identity of the good anchor and the associated concentrated transition target change at geometric change-points \cite{gerogiannis2024blackboxfeas}. The active base index cycles through the $5$ anchor MDPs, with segment parameter $T^{-\xi}$ for $\xi\in\{0.4,0.6,0.8\}$ over $T=50000$ episodes. For drifting experiments, we use the gradual convex-interpolation protocol of \cite{zhou2022restartlsviucb}: both $\theta_{h,k}$ and $\mu_{h,k}$ are linearly interpolated between successive anchor MDPs over windows of $100$ episodes. Consequently, the anchor responsible for high terminal reward and directed transitions changes smoothly rather than abruptly.

\paragraph{Block Low-Rank.}
We construct a block-structured low-rank linear MDP in the same finite-horizon linear factorization model, motivated by low-rank transition models for reinforcement learning \cite{agarwal2020lowrank}. States are partitioned into latent blocks, and the feature vector of a state-action pair is mostly concentrated on a block coordinate determined by the current state block and the action. A small simplex noise component is added so that features are not exactly one-hot, but the dominant coordinate still encodes the block-level action effect. For each base MDP and each latent coordinate, the transition measure sends most mass to states in a destination block determined by the source block and the base index, with a small amount of global noise over all states. Rewards are low throughout most of the episode. Near the end of the horizon, the current good block receives reward $0.5$ at the penultimate step and reward $1$ at the terminal step, while a neighboring block receives a smaller decoy reward. Unless otherwise stated, we use $S=24$, $A=5$, $H=10$, $d=12$, $12$ latent blocks, and $5$ base MDPs.

\emph{Non-stationarity.} For abrupt PS experiments, the active block dynamics and good reward block switch among the $5$ base MDPs according to the geometric change-point model \cite{gerogiannis2024blackboxfeas}, with parameter $T^{-\xi}$ for $\xi\in\{0.4,0.6,0.8\}$ over $T=50000$ episodes. For drifting experiments, the reward vectors and transition measures are convexly interpolated between consecutive base MDPs over $100$-episode windows, as in the gradual protocol of \cite{zhou2022restartlsviucb}. This produces a smooth movement of the favorable latent block and the corresponding block-to-block transition pattern.

\subsection{Hardware Specifications}
All experiments were employed on a desktop using an Intel(R) Xeon(R) W-2245 processor with 128 GB RAM.

\subsection{Enhanced Experimental Plots}
\label{app:enhanced_plots}

\begin{figure}[!htbp]
\centering

\begin{subfigure}{0.85\linewidth}
  \centering
  \includegraphics[width=\linewidth]{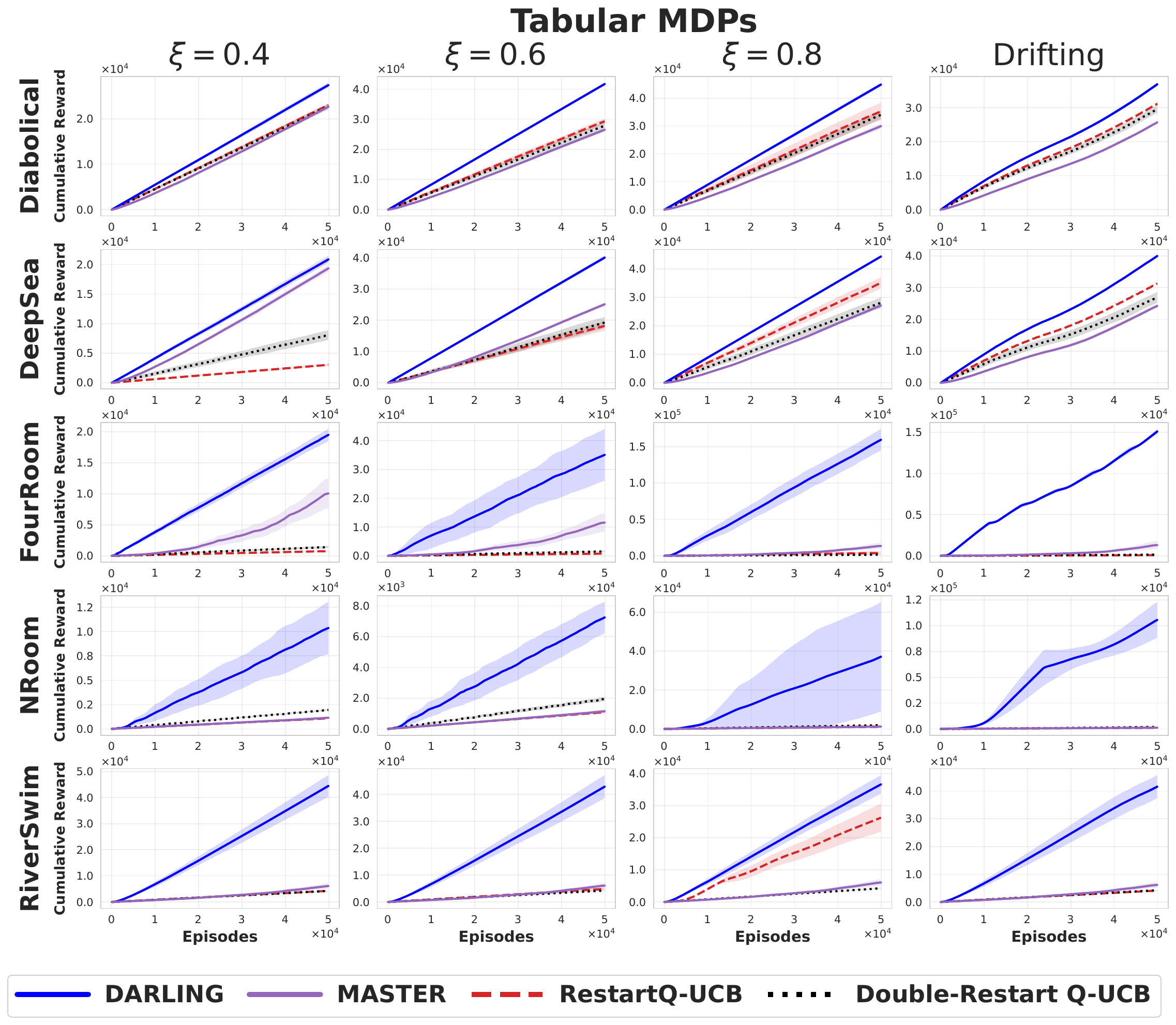}
  \caption{Tabular MDPs.}
  \label{fig:experiments_tabular}
\end{subfigure}

\vspace{0.8em}

\begin{subfigure}{0.85\linewidth}
  \centering
  \includegraphics[width=\linewidth]{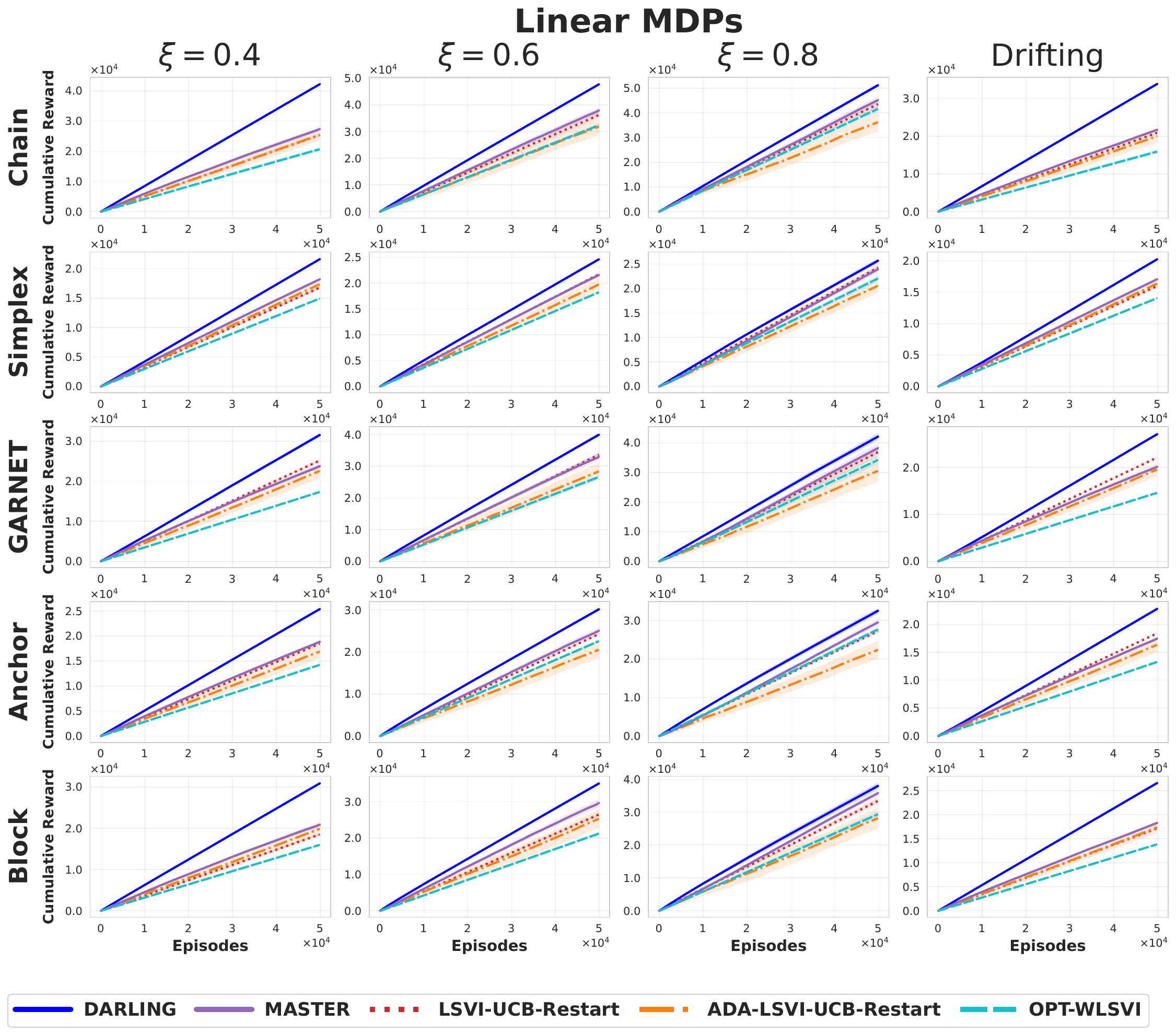}
  \caption{Linear MDPs.}
  \label{fig:experiments_linear}
\end{subfigure}

\caption{Cumulative reward results for the experiments (higher is better). DARLING outperforms all state-of-the-art baselines across the considered tabular and linear settings.}
\label{fig:enhanced_experiments}
\end{figure}

\end{document}